\documentclass[11pt]{article}

\usepackage{natbib}

\usepackage[utf8]{inputenc} 
\usepackage{wrapfig}
\usepackage{nk}
\usepackage{color}
\usepackage{epstopdf}

\begin{document}

\title{{Nonparametric Hawkes Processes: Online Estimation and Generalization Bounds}\thanks{Yingxiang Yang and Negar Kiyavash are affiliated with Department of Electrical and Computer Engineering at University of Illinois at Urbana-Champaign. Negar Kiyavash is also affiliated with Department of Industrial and Enterprise Systems Engineering at University of Illinois at Urbana-Champaign, along with Jalal Etesami and Niao He. Emails: \texttt{yyang172,etesami2,niaohe,kiyavash@illinois.edu}. This work was supported in part by MURI grant ARMY W911NF-15-1-0479 and ONR grant W911NF-15-1-0479. Part of this work was presented at Advances in Neural Information Processing Systems (NIPS 2017) \citep{yang2017online}.}}

\author{Yingxiang Yang, Jalal Etesami, Niao He, Negar Kiyavash}

\maketitle

\begin{abstract}
	In this paper, we design a nonparametric online algorithm for estimating the triggering functions of multivariate Hawkes processes. Unlike parametric estimation, where evolutionary dynamics can be exploited for fast computation of the gradient, and unlike typical function learning, where representer theorem is readily applicable upon proper regularization of the objective function, nonparametric estimation faces the challenges of (i) inefficient evaluation of the gradient, (ii) lack of representer theorem, and (iii) computationally expensive projection necessary to guarantee positivity of the triggering functions. In this paper, we offer solutions to the above challenges, and design an online estimation algorithm named NPOLE-MHP that outputs estimations with a $\calO(1/T)$ regret, and a $\calO(1/T)$ stability. Furthermore, we design an algorithm, NPOLE-MMHP, for estimation of multivariate marked Hawkes processes. We test the performance of NPOLE-MHP on various synthetic and real datasets, and demonstrate, under different evaluation metrics, that NPOLE-MHP performs as good as the optimal maximum likelihood estimation (MLE), while having a run time as little as parametric online algorithms.
\end{abstract}

\section{Introduction}\label{sec:intro}

Multivariate Hawkes Processes (MHPs) are multivariate counting process where an arrival in one dimension can affect the arrival rates of other dimensions. The origin of MHPs dates back to \citet{hawkes1971spectra}, where it was used to statistically model earthquakes, for the purpose of revealing a temporally self-excitation pattern and a spatially mutual-excitation structure. Because of their ability to capture mutual excitation between different dimensions of a multivariate counting process, MHPs have become a popular model in a plethora of scenarios. In high frequency trading \citep{bacry2015hawkes,Bacry2012,hardiman2013critical}, MHPs are commonly used to model the clustered arrival patterns of bullish and bearish orders. In computational biology, MHPs are used to model neural spike train data \citep{reynaud2010adaptive}. In social network studies, MHPs have been used to model diffusion networks as an alternative for the contagion model \citep{yang2013mixture}. In computational phenotyping, MHPs are used to extract useful information from Electronic Health Record (EHRs), such as the relationship between the symptoms experienced by patients and the intake of prescribed medicines \citep{baohawkes}. In criminology, MHPs have been applied to analyze the spatially and temporally clustered occurrences of crimes and terrorist activities, enabling more efficient dispatch of the police forces \citep{mohler2011self,porter2012self}.

The key factor that determines the ability of an MHP for capturing the self- and the mutual-excitation effects lies within the form of its intensity function. For a $p$-dimensional MHP, the intensity function of the $i$-th dimension takes the following form:
\#\label{eq::intensity}
\lambda_{i}(t)=\mu_i+\sum_{j=1}^{p}\int_0^tf_{i,j}(t-\tau)\ud N_j(\tau)=\mu_i+\sum_{j=1}^{p}\sum_{n=1}^{N_j(t)}f_{i,j}(t-\tau_{j,n}),
\#
where the constant $\mu_i$ is the base intensity of the $i$-th dimension, $N_j(t)$ counts the number of arrivals in the $j$-th dimension within $[0,t]$, and $f_{i,j}(t)$ is the triggering function that embeds the underlying causal structure of the model. Heuristically, one arrival in the $j$-th dimension at time $\tau$ will affect the intensity function of the $i$-th dimension at time $t$ by the amount $f_{i,j}(t-\tau)$ for $t>\tau$. The cumulative effect of the arrivals from different dimensions, as well as the cumulative effect of the arrivals over a period of time, are embedded within the additive structure of the intensity function.

In many cases, an MHP alone is not enough to capture the dynamics of the underlying counting process, especially in the case where the events arrived at different times are not identical. For example, in a high frequency trading scenario, each order not only has an arrival time, but also has a trading volume, which is an important parameter that influences the trend and momentum of a stock. Likewise, when studying the patterns of earthquakes, one cannot ignore the magnitude of each shock, as it is intuitive that a strong earthquake is more likely to trigger aftershocks than an earthquake with a much smaller magnitude. Such differences are typically distinguished by associating a mark to each event, and a multivariate marked Hawkes process (MMHP) \citep{fauth2012modeling} model is used to fit the data. For an MMHP, the intensity function of the $i$-th dimension takes the form
\#\label{eq::MMHP_intensity}
\lambda_i(t)=\mu_i+\sum_{j=1}^{p}\int_0^tf_{i,j}(t-\tau, v)\ud N_j(\tau\times v)=\mu_i+\sum_{j=1}^{p}\sum_{n=1}^{N_j(t)}f_{i,j}(t-\tau_{j,n},v_{j,n}),
\#
where, compared to \eqref{eq::intensity}, the triggering function now depends on both the arrival time and the corresponding mark.

\subsection{Motivations}

Driven by its wide applicability, there has been extensive studies on the estimation of MHPs and MMHPs from real-time and large volumes of event data \citep{hall2014tracking,bacry2014estimation,bacry2015generalization,bacry2012non}. However, most existing Hawkes process models, as well as the methods used for estimating these models, suffer from severe limitations from both the modeling and the computational perspectives.

Firstly, existing works often make strong assumptions and specify a restricted parametric form of the intensity functions that is not expressive
enough to capture the temporal dynamics in many applications. For example, exponential triggering functions
\#\label{eq::parametricf}
f_{i,j}(t)=\alpha_{i,j}\exp\{-\beta_{i,j} t\}\mathds{1}\{t>0\}
\#
are used in most existing works, where $\alpha_{i,j}$s are unknown while $\beta_{i,j}$s are given a priori. Under this assumption, the estimation of each triggering function is equivalent to the estimation of a real number. However, there are many scenarios where \eqref{eq::parametricf} fails to describe the correct mutual influence pattern between dimensions. This is especially true in studies related to neural spike trains, where it is well known that human body takes time to react to the information it receives. For example, \citet{krumin2010correlation} and \citet{eichler2016graphical} have reported delayed and bell-shaped triggering functions when applying the MHP model to neural spike train datasets. Moreover, when the triggering functions are not exponential, or when $\beta_{i,j}$s are inaccurate, formulation in \eqref{eq::parametricf} is prone to model mismatch \citep{hall2014tracking}.

Secondly, most existing works perform batched estimation upon observing all the samples. This can be costly when the samples are streaming in nature and are expensive to observe. For example, in criminology. On the other hand, when the amount of samples is huge, evaluating the batch gradient can be computationally expensive, and the scalability of such algorithms is poor \citep{yang2017online}.

The above concerns motivate us to investigate the estimation of MHPs in an online and nonparametric regime.

\subsection{Related Works}

Earlier works on estimating the triggering functions for MHPs can be largely categorized into three classes: (i) parametric batch estimation, (ii) nonparametric batch estimation, and (iii) parametric online estimation. The contribution on MMHPs is even less, and mostly focuses on parametric batch estimation \citep{fauth2012modeling}. 

{\noindent\bf Parametric batch estimation.} Based on the assumption that the triggering functions have exponential forms specified in \eqref{eq::parametricf} with known $\beta_{i,j}$s, parametric batch estimation uses all the available samples to estimate the coefficient $\alpha_{i,j}$s. Under the exponential form of the triggering functions, the MHPs possesses the Markov property, and the intensity function for each dimension can be evaluated by considering only ``recent" events. Therefore it is computationally much less expensive than nonparametric estimation. The most widely used estimators include the maximum likelihood estimator (MLE, \citet{ozaki1979maximum}), and the minimum mean-square error estimator (MMSE, \citet{bacry2015generalization}). These estimation methods can also be generalized to the high dimensional case when the coefficient matrix is sparse and low-rank \citep{bacry2015generalization}.

More generally, one can assume that $f_{i,j}(t)$s lie within the span of a pre-determined set of basis functions $\calS=\{{\bm{e}}_1(t),\ldots,{\bm{e}}_{|\calS|}(t)\}$:
$
f_{i,j}(t)=\sum_{i=1}^{|\calS|}c_i{\bm{e}}_i(t),
$
where ${\bm{e}}_i(t)$s have a given parametric form \citep{etesami2016learning,xu2016learning}. One example of such algorithms is presented in \citet{xu2016learning}, where the number of bases is adaptively chosen, which sometimes requires a significant portion of the data to determine the optimal set of bases. 

{\noindent\bf Nonparametric batch estimation.} A more sophisticated approach towards finding the set of basis functions is explored in \citet{zhou2013learning}, where the coefficients and the basis functions are iteratively updated and refined. Unlike \citet{xu2016learning}, where the basis functions take a predetermined form, \citet{zhou2013learning} updates the basis functions by solving a set of Euler-Lagrange equations in the nonparametric regime. However, the optimality for \citet{zhou2013learning} is not guaranteed as its formulation is nonconvex. Practically, the method also requires more than $10^5$ arrivals for each dimension in order to obtain good results, on networks of less than 5 dimensions. 

Another way to estimate $f_{i,j}(t)$s nonparametrically is proposed in \citet{bacry2014second}, which solves a set of $p$ Wiener-Hopf systems with $p^2$ dimensions. The algorithm is guaranteed to converge and achieves excellent visual goodness-of-fit in various numeric examples. However, this method requires inverting a $p^2\times p^2$ matrix, which is costly, if not at all infeasible, when $p$ is large.

{\noindent\bf Parametric online estimation.} To the best of our knowledge, online estimation of the triggering functions seems largely unexplored. Under the assumption that $f_{i,j}(t)$s are exponential, \citet{hall2014tracking} proposes an online algorithm using gradient descent, while exploiting the evolutionary dynamics of the intensity function. The time axis is discretized into small intervals, and the updates are performed at the end of each interval. Unfortunately, this method cannot be extended to the nonparametric setting where the triggering functions are not exponential, mainly because the evolutionary dynamics of the intensity functions does not hold in general. Therefore, the nonparametric estimation of the triggering functions remains largely an open problem.

\subsection{Challenges and Our Contributions}

Designing a nonparametric online estimation algorithm is not without its challenges: (i) It is not clear how to represent the triggering functions. In this work, we relate the triggering functions to a reproducing kernel Hilbert space (RKHS). Upon proper regularization of the objective function, one can apply the representer theorem \citep{scholkopf2001generalized} which reduces the estimation of the triggering function to the estimation of a growing set of coefficients. (ii) Although online kernel estimation is a well studied topic in other scenarios \citep{kivinen2004online}, a typical choice of objective function for an MHP usually involves the integral of the triggering functions, which prevents the direct application of the representer theorem. (iii) For the commonly used objective functions, such as the log-likelihood and the MSE loss, the evaluation of the stochastic gradient requires evaluating the intensity function, which is computationally expensive in nonparametric regime due to a lack of Markov property and evolutionary dynamics. (iv) The outputs of the algorithm at each iteration require a projection step to ensure positivity of the intensity function. This requires solving a quadratic programming problem, which can be computationally expensive.

In this paper, we design, to the best of our knowledge, the first nonparametric online estimation algorithm for the triggering functions. In particular, we contribute to the subject of estimating MHPs by providing solutions to the four challenges we mentioned above. (i) For representation, we base our analysis on the assumption that the triggering functions belong to an RKHS. (ii) We choose the negative log-likelihood as the objective function, and, to apply representer theorem, we approximate the objective function by discretization and characterize the approximation error bound. (iii) We achieve fast evaluation of the gradient using a truncated intensity function, and characterize the approximation error bound. (iv) For projection operation, we adopt a transformation to the triggering function, which achieves low estimation error under various simulation settings.

Theoretically, our algorithm achieves a regret bound of $\calO(\log T)$, with $T$ being the time horizon. Numerical experiments show that our approach outperforms the previous approaches despite the fact that they handle a less general setting. In particular, our algorithm attains a similar performance to the nonparametric batch maximum likelihood estimation method while reducing the run time extensively.

\subsection{Organization of This Paper}

The rest of this paper is organized as follows. In Section \ref{sec:prelim}, we provide a short introduction on RKHSs, as it is the main tool we will use throughout this paper. In Sections \ref{sec::background}-\ref{sec::result}, we develop the nonparametric online learning algorithm for MHPs, and provide theoretical guarantee to the proposed algorithm. The formulation of the problem is introduced in Section \ref{sec::background}; three online algorithms are introduced in Section \ref{sec::algorithm}, including the nonparametric online algorithm for MHPs, as well as parametric and nonparametric algorithms for MMHPs; the regret bound and statistical performances are presented in Section \ref{sec::result}; numerical simulations are provided in Section \ref{sec::exp}.

\subsection{Notations}

Prior to discussing our results, we introduce the basic notations used in the paper. Detailed notations will be introduced along the way. For a $p$-dimensional MHP, we denote the intensity function of the $i$-th dimension by $\lambda_i(t)$. We use $\blambda(t)$ to denote the vector of intensity functions, and we use $\bF=[f_{i,j}(t)]$ to denote the matrix of triggering functions. The $i$-th row of $\bF$ is denoted by $\bbf_i$. 
The number of arrivals in the $i$-th dimension up to $t$ is denoted by the counting process $N_i(t)$. We set $N(t)=\sum_{i=1}^{p}N_i(t)$. The estimates of these quantities are denoted by their ``hatted" versions. The arrival time of the $n$-th event in the $j$-th dimension is denoted by $\tau_{j,n}$. Lastly, define $\lfloor x\rfloor_{y}=y\lfloor x/y\rfloor$. The proofs appear in Appendix.

\section{Preliminaries}\label{sec:prelim}
{}
In this section, we introduce some preliminaries on RKHSs, and review previous approaches to estimating MHPs using maximum likelihood estimation.

\subsection{Reproducing Kernel Hilbert Spaces}

Consider a Hilbert space $\calH$ that contains functions supported on $\calX$. The inner product of $\calH$ is denoted by $\la\cdot,\cdot\ra_{\calH}$, and recall that $\calH$ is complete under the norm induced by the inner product. This Hilbert space $\calH$ is an RKHS if there exists a bivariate function $K(\cdot,\cdot):\calX\times\calX\to\reals$, such that for any $n\in\naturals$, $c_1,\ldots,c_n \in\reals$, and $x_1,\ldots,x_n\in\calX$, $K(\cdot,\cdot)$ is a positive definite kernel:
\$
\sum_{i=1}^{n}\sum_{j=1}^{n}c_ic_jK(x_i,x_j)\geq 0,
\$
and that for any $x\in\calX$, the evaluation functional is bounded (or equivalently continuous):
\$
f(x)=\la f,K(x,\cdot)\ra_{\calH}\leq C\|f\|_{\calH}
\$
for some constant $C$. We call $K(\cdot,\cdot)$ the reproducing kernel of $\calH$.

Several commonly used RKHSs include RKHSs associated with (i) Polynomial kernels: for $\calX\subset \reals^d$,
	$
	K(x,y)=(\alpha x^\top y+\beta)^d.
	$
	When $d=1$, $K(\cdot,\cdot)$ reduces to a linear kernel.
	(ii) Gaussian kernels with bandwidth $h$:
	$
	K(x,y)=\exp\left\{-\frac{\|x-y\|_{\calX}^2}{2h^2}\right\}.
	$
	(iii) Laplacian kernels with bandwidth $h$:
	$
	K(x,y)=\exp\left\{-\frac{\|x-y\|_{\calX}}{h}\right\}.
	$
The concrete expression of the inner product varies by the choice of the reproducing kernel. We skip the detailed discussion since it is irrelevant in our paper. The functional gradient $\ud f(x)/\ud f$, which is the fastest ascent direction of $f(x)$ within $\calH$, is defined as
$
\frac{\ud f(x)}{\ud f}:=\sup_{g\in\calH, \|g\|_{\calH}=1}\lim_{\epsilon\to 0}\frac{[f+\epsilon g](x)-f(x)}{\epsilon}.
$
Since $f(x)=\la f,K(x,\cdot)\ra_{\calH}$, we have $\frac{\ud f(x)}{\ud f}:=K(x,\cdot)$.
RKHS and kernel methods are widely used for nonparametric estimation in machine learning, especially for empirical risk minimization. This is largely due to the representer theorem, which allows reducing an infinite-dimensional optimization problem to a finite-dimensional one. 
\begin{theorem}[Representer theorem \citep{scholkopf2001generalized}] \label{thm::representer}
	Let $K(\cdot,\cdot)$ be the reproducing kernel of $\calH$, and let $R(\cdot):\reals_{+}\to \reals$ be a strictly monotonically increasing function. Then,
	\$
	f^*=\argmin_{f\in \calH} \left\{L(f(x_1),\ldots,f(x_n))+R(\|f\|_{\calH}^2)\right\}
	\$
	has a representation form of
	$
	f^*(\cdot)=\sum_{i=1}^{n}c_i K(x_i,\cdot),
	$
	where $\{c_i\}_{i=1}^{n}$ is the set of coefficients.
\end{theorem}


\subsection{Estimation of Multivariate Hawkes Processes}


A common approach for estimating the parameters of an MHP is to perform regularized MLE. The negative of the log-likelihood function of an MHP over the time interval $[0,t]$ is given by
\#\label{eq::likelihood0}
\mathcal{L}_t(\blambda):=-\sum_{i=1}^{p}\left(\int_{0}^t\log\lambda_i(\tau)\ud N_{i}(\tau)-\int_0^t\lambda_i(\tau)\ud\tau\right).
\#
Since the intensity $\lambda_i(t)$ is linear with respect to the triggering functions $f_{i,1}(t),\ldots,f_{i,p}(t)$, the negative log-likelihood function $\calL_t(\blambda)$ is convex with respect to $f_{i,j}(t)$ for all $i,j\in\{1,\ldots,p\}$. In the parametric case \eqref{eq::parametricf}, each triggering function $f_{i,j}(t)$ is a linear function of $\alpha_{i,j}$, and therefore $\calL_t(\blambda)$ is a convex with respect to $\alpha_{i,j}$ for all $i,j\in\{1,\ldots,p\}$.

Alternatively, other convex loss functions have also been used for parametric estimation of Hawkes processes in the literature, e.g., square loss~\cite{bacry2015generalization} and logistic loss~\cite{menon2018proper}.
To promote solutions with desired structures, such as sparsity of the triggering matrix or the smoothness of the triggering functions, one can also add proper penalties to these objectives (see, \eg, \cite{bacry2015generalization,zhou2013learning,xu2016learning}). The resulting optimization problems are often processed by Expectation Maximization (EM) algorithm \citep{xu2016learning}, batch gradient descent, ADMM~\citep{zhou2013learning}, and so on. 

\section{Problem Formulation}\label{sec::background}

In this section, we introduce our assumptions and definitions followed by the formulation of the objective function. We omit the basics on MHPs and instead refer the readers to \cite{liniger2009multivariate} for details.

\begin{assumption}\label{asm::stationary} 
	We assume that the constant base intensity $\mu_i$ is bounded between $\mu_{\max}<\infty$ and $\mu_{\min}>0$. We also assume bounded and stationary increments for the MHP in the sense that, for a fixed $z_{\min}>0$ and any $z>z_{\min}$, $N_i(t)-N_i(t-z)\leq \kappa_z=\calO(z)$ for all $t>0$.
\end{assumption}

\begin{definition}	\label{def:11}
	Suppose that $\{t_k\}_{k=0}^{\infty}$ is an arbitrary time sequence with $t_0=0$, and $\sup_{k\geq 1}(t_{k}-t_{k-1})\leq\delta\leq 1$. Let $\varepsilon_f:[0,\infty)\to [0,\infty)$ be a continuous and bounded function such that $\lim_{t\to\infty}\varepsilon_f(t)=0$. Then, $f(x)$ satisfies the \textit{decreasing tail} property with \textit{tail function} $\varepsilon_f(t)$ if
	\$
	\sum_{k=m}^{\infty}(t_k-t_{k-1})\sup_{x\in(t_{k-1},t_k]}|f(x)|\leq\varepsilon_{f}(t_{m-1}), \quad \forall m>0.
	\$

\end{definition}

\begin{assumption}\label{asm::f} Let $\calH$ be an RKHS associated with a kernel $K(\cdot,\cdot)$ that satisfies $K(x,x)=1$. Let $L_1[0,\infty)$ be the space of functions for which the absolute value is Lebesgue integrable. For any $i,j\in\{1,\ldots,p\}$, we assume that $f_{i,j}(t)\in\calH$ and $f_{i,j}(t)\in L_1[0,\infty)$, with both $f_{i,j}(t)$ and $\ud f_{i,j}(t)/\ud t$ satisfying the decreasing tail property of Definition \ref{def:11}.
\end{assumption}

Assumption \ref{asm::stationary} is common and has been adopted in existing literature \citep{liniger2009multivariate,hawkes1971spectra,bremaud1996stability}. In particular, it ensures that the MHP is not ``explosive" by assuming that ${N(t)}/{t}$ is bounded. A simple analysis shows that this condition likely holds: for an MHP with stationary increments, we have 
$
\mathbb{E}[\ud \Nb(t)|\mathcal{F}^t]=\blambda(t)\dt,
$
where $\mathcal{F}^t$ denotes the $\sigma$-algebra generated by $\{N_1(t),...,N_p(t)\}$, $\Nb(t)=[N_1(t),\ldots,N_p(t)]$ counts the number of arrivals for all $p$ dimensions, and $\blambda(t)=[\lambda_1(t),\ldots,\lambda_p(t)]$ collects the intensity functions of the $p$ dimensions. Taking another expectation and combining the stationary increment assumption, we have
$$
\mathbb{E}[ \Nb(t)-\Nb(t-z)]=\bar{\blambda} z,
$$
where $\bar{\blambda}:=\mathbb{E}[\blambda(t)]$ is a positive constant vector that represents the average growth rates of $\lambda_i(t)$s. On the other hand, 
using the second order statistics of Hawkes processes in \citet{bacry2015generalization}, we know that the covariance matrix of a Hawkes process is given by
$$
\mathbb{E}\left[(\Nb(t)-\Nb(t-z)- \bar{\blambda}z)(\Nb(t)-\Nb(t-z)- \bar{\blambda}z)^\top \right]=\Psi_t z,
$$
where $\Psi_t$ is a $p\times p$ matrix in Theorem 1 of \citet{bacry2012non}. This shows that $z^{-1}(\Nb(t)-\Nb(t-z))$ has a fixed mean and a covariance that converges to $\zero$ as $z\to\infty$. Therefore, asymptotically, $N_i(t)-N_i(t-z)=\Theta(z)$ for all $i\in\{1,\ldots,p\}$.

Assumption \ref{asm::f} restricts the tail behaviors of both $f_{i,j}(t)$ and $\ud f_{i,j}(t)/\dt$. Intuitively, it restricts the total impact of the historic events on the intensity function, and hence also guarantees that the MHP under consideration is not ``explosive". As it turns out, $\{t_k\}_{k=0}^{\infty}$ corresponds to the set of update time when designing the online algorithm. Complicated as it may seem, functions with exponentially decaying tails satisfy this assumption, as is illustrated by the following examples:

\begin{example}\label{eg::2}
	The function $f(t)=\exp\{-\beta t\}\mathds{1}\{t>0\}$ with $\beta>0$ satisfies Assumption \ref{asm::f} with
	$\beta^{-1}\exp\{-\beta(t-\delta)\}$ as its tail function.
\end{example}

\begin{proof}[Proof of Example \ref{eg::2}]

When $\beta>0$, $f(t)$ is a monotonically decreasing function. Therefore, for any $\delta$-update set, we have
$
\sup_{x\in(t_{k-1},t_k]}|f(x)|=f(t_{k-1})=\exp\{-\beta t_{k-1}\}.
$
Hence, 
\$
(t_k-t_{k-1})\exp\{-\beta t_{k-1}\} &\leq\int_{t_{k-1}}^{t_k}\exp\{-\beta (t-\delta)\}\ud t=\frac{1}{\beta}\left(\exp\{\beta(\delta-t_{k-1})\}-\exp\{\beta(\delta-t_k)\}\right),
\$
where the inequality is due to the fact that, for $t_{k}\leq t_{k-1}+\delta$ and $t\in[t_{k-1}, t_k]$, $\exp\{-\beta t_{k-1}\}\leq \exp\{-\beta (t-\delta)\}$. Summing up both sides of the above inequality, we have
\$
\sum_{k=m}^{\infty}(t_k-t_{k-1})\sup_{x\in(t_{k-1},t_k]}|f(x)|\leq \frac{1}{\beta}\exp\{-\beta(t_{m-1}-\delta)\}.
\$
Similarly, one can obtain the tail functions of $|\ud f(t)/\dt|$ that is $\exp\{-\beta(t_{m-1}-\delta)\}$.
	
\end{proof}

\begin{example}\label{eg::3}
	The function $f(t)=\exp\{-(t-\gamma)^2\}\mathds{1}\{t>0\}$ satisfies Assumption \ref{asm::f} with $\sqrt{2\pi}\ \text{erfc}({t/\sqrt{2}}-\gamma)\exp\{\delta^2/2\}$ as its tail function.
\end{example}

\begin{proof}[Proof or Example \ref{eg::3}]
For $t_{m-1}>\gamma+2\delta$, we have 
$
(t_k-t_{k-1})\sup_{t\in[t_{k-1},t_k]}e^{-(t-\gamma)^2} \leq\int_{t_{k-1}}^{t_k}e^{-\frac{1}{2}(t-\gamma)^2+\frac{\delta^2}{2}}\ud t,
$
in which we used the fact that $t_{k}\leq t_{k-1}+\delta$. Hence, a tail function for $f(t)$ is 
$$
\epsilon(t)=\int_{t_{m-1}}^{\infty}e^{-\frac{1}{2}(t-\gamma)^2+\frac{\delta^2}{2}}\ud t= \sqrt{\frac{\pi}{2}} \text{erfc}\left(\frac{t_{m-1}}{\sqrt{2}}-\gamma\right)e^{\frac{\delta^2}{2}}.
$$
For the derivative of $f(t)$ and for $t_{m-1}>\gamma+2\delta+1/\sqrt{2}$, we have
$
(t_k-t_{k-1})\sup_{t\in[t_{k-1},t_k]}|t-\gamma|e^{-(t-\gamma)^2}  \leq\int_{t_{k-1}}^{t_k}(t-\gamma)e^{-\frac{1}{2}(t-\gamma)^2+\frac{\delta^2}{2}}\ud t.
$
This implies the following tail function for  $f'(t)$:
\$
\varepsilon'(t)=e^{\frac{\delta^2}{2}}\int_{t_{m-1}}^{\infty}(t-\gamma)e^{-\frac{1}{2}(t-\gamma)^2}\ud t.
\$
\end{proof}

\subsection{A Discretized Objective Function for Online Estimation}

A main challenge in nonparametric estimation of MHPs arises from the lack of representer theorem, which is often necessary for reducing the complexity of function estimation.  Note that in the case of MHPs, the negative log-likelihood function, given in \eqref{eq::likelihood0}, contains an integral term of $f_{i,j}(t)$, i.e., it depends on $f_{i,j}(t)$ for all $t\in[0,t]$ rather than a discrete set. Thus it does not satisfy the condition required for representer theorem as shown in Theorem \ref{thm::representer}. 

To resolve this issue, several approaches are recently proposed: (i) the discretization approach \citep{hall2014tracking}, which approximates the integral by its Riemann sum; (ii) the adjusted Hilbert space approach \citep{flaxman2016poisson,raskutti2012minimax}, which transforms the integral term together with the RKHS norm defined on $\calH$ into a new RKHS norm defined on $\tilde{\calH}$. The latter approach, however, only applies to inhomogeneous Poisson process and cannot be easily generalized to the complicated MHPs. Moreover, it requires solving an unfavorable non-convex optimization problem.   Due to these considerations, we adopt the first approach and discretize the integral.

Let $\{\tau_1,...,\tau_{N(t)}\}$ denote the arrival times of all the events within $[0,t]$ and let $\{t_0,\ldots,t_{M(t)}\}$ be the end points of a finite partition of the time interval $[0,t]$ such that $t_0=0$ and 
$
t_{k+1}:=\min_{ \tau_{i}\geq t_k}\{\lfloor t_k\rfloor_{\delta}+\delta, \tau_{i} \}.
$
In addition, without loss of generality, we assume $t_{M(t)}=t$.
Using this partitioning, it is straightforward to see that the function in (\ref{eq::likelihood0}) can be written as 
\#\label{eq::ltdelta0}
\mathcal{L}_t(\blambda)&=\sum_{i=1}^{p}\sum_{k=1}^{M(t)}\left(\int_{t_{k-1}}^{t_k}\lambda_i(\tau)\ud\tau-x_{i,k}\log\lambda_i(t_k)\right)\!\!:=\! \sum_{i=1}^{p}L_{i,t}(\lambda_i),
\#
where $x_{i,k}:=N_i(t_k)-N_i(t_{k-1})$.
By the definition of $t_k$, we know that $t_k$ either corresponds to the arrival time of an event, or the end of an interval in which no events arrived, which implies $x_{i,k}\in\{0,1\}$. 
We now approximate the integral $\int_{t_{k-1}}^{t_k}\lambda_i(\tau)\ud\tau$ by $(t_k-t_{k-1})\lambda_i(t_k)$, and thus obtaining
\#\label{eq::ltdelta}
\mathcal{L}_t^{(\delta)}(\blambda)&:=\sum_{i=1}^{p}\sum_{k=1}^{M(t)}\left((t_k-t_{k-1})\lambda_i(t_k)-x_{i,k}\log\lambda_i(t_k)\right)\!\!:=\! \sum_{i=1}^{p}L_{i,t}^{(\delta)}(\lambda_i).
\#
Intuitively, if $\delta$ is small enough and the triggering functions are bounded, it is reasonable to expect that $L_{i,t}(\blambda)$ is close to $L_{i,t}^{(\delta)}(\blambda)$. 
Below, in Proposition \ref{prop::approx}, we characterize the accuracy of the above discretization of the intensity function.
 
Another challenge in nonparametric estimation of MHPs comes from the expensive computation cost when evaluating the functional gradient of the objective, which requires the evaluations of $\lambda_i(t_k)$ for $\{t_k\}_{k=0}^{M(t)}$. To perform fast evaluation of the intensity function, we introduce the concept of ``truncated intensity", denoted by $\lambda_i^{(z)}(t)$, by considering the impact of only those events that arrived within a recent window $[t-z,t)$.\footnote{Since the probability that an event arrive at $t-z$ is $0$, it does not matter whether the interval is close or open on the left-hand side. Since $f_{i,j}(t)=0$ for $t\leq 0$, it does not matter whether the interval is open or close on the right-hand side.} Similar to discretization, the approximation error caused by the truncation can be well-controlled given Assumptions \ref{asm::f} and \ref{asm::stationary}. We now formally define the truncated intensity function $\lambda_i^{(z)}(t)$ before characterizing the truncation error together with the discretization error in Proposition \ref{prop::approx}.
 \begin{definition}\label{def:trunc}
 We define the truncated intensity function as follows 
 \begin{equation}\label{eq::trunc}
\lambda^{(z)}_i(t):=\mu_i+\sum_{j=1}^{p}\int_{0}^{t}\mathds{1}\{t-\tau<z\}f_{i,j}(t-\tau)\ud N_j(\tau).
	\end{equation}
 \end{definition}

\begin{proposition}\label{prop::approx}
	Under Assumptions \ref{asm::stationary} and \ref{asm::f},
	for any  $i\in\{1,\ldots,p\}$, we have
	\$
	\left\vert L_{i,t}^{(\delta)}(\lambda^{(z)}_i)-L_{i,t}(\lambda_i)\right\vert&\leq (1+{\kappa_1}{\mu^{-1}_{\min}})N(t-z)\varepsilon(z)+\delta N(t)\varepsilon'(0),
	\$
	where $\mu_{\min}$ is the lower bound for $\mu_i$, $\kappa_1$ is the upper bound for $N_i(t)\!-\!N_i(t-1)$ from Definition \ref{def:11}, while $\varepsilon$ and $\varepsilon'$ are two tail functions that uniformly capture the decreasing tail property of all $f_{i,j}(t)$s and all $\ud f_{i,j}(t)/\ud t$s, respectively.
\end{proposition}

The first term in the bound characterizes the approximation error when one truncates $\lambda_i(t)$ with $\lambda_i^{(z)}(t)$. The second term describes the approximation error caused by the discretization. When $z=\infty$, $\lambda_i(t)=\lambda_i^{(z)}(t)$, and the approximation error is contributed solely by discretization. 
In many cases, a small enough truncation error can be obtained by setting a relatively small $z$, which greatly simplifies the evaluation procedure and effectively captures the effect of those important ``recent" events. For example, for $f_{i,j}(t)=\exp\{-3t\}\mathds{1}\{t>0\}$, setting $z=10$ would result in a truncation error less than $10^{-13}$. Therefore, in our algorithm, we evaluate the intensity functions with their truncated versions.

Finally, to invoke the representer theorem, we consider the regularized instantaneous risk function with Tikhonov regularization for $f_{i,j}(t)$s and $\mu_i$:
\#\label{eq::inst_loss}
l_{i,k}(\lambda_i):=(t_k-t_{k-1})\lambda_i(t_k)-x_{i,k}\log\lambda_i(t_k)+\frac{1}{2}\omega_i\mu_i^2+\sum_{j=1}^{p}\frac{\zeta_{i,j}}{2}\|f_{i,j}\|_\calH^2,
\#
and aim at producing a sequence of estimates $\{\hat{\lambda}_i(t_k)\}_{k=1}^{M(t)}$ of $\lambda_i(t)$ with minimal regret:
\#\label{eq::problem}
\sum_{k=1}^{M(t)} l_{i,k}(\hat{\lambda}_{i}(t_k))- \min_{\mu_i\geq \mu_{\min}, f_{i,j}(t)\geq 0}\sum_{k=1}^{M(t)}l_{i,k}(\lambda_i(t_k)).
\#
In \eqref{eq::inst_loss}, $\omega_i$ and $\zeta_{i,j}$ are positive regularization coefficients. Adding the regularization terms in \eqref{eq::inst_loss} not only allows us to apply the representer theorem, but also accelerates the gradient descent as the regularized instantaneous risk function is now strongly convex. Notice that a tradeoff exists between the optimality of the solution and the algorithmic stability: for large $\zeta_{i,j}$ and $\omega_i$, the algorithm is more stable but the objective function deviates from the negative log-likelihood of MHP, whereas for small $\zeta_{i,j}$ and $\omega_i$, the objective function is closer to the negative log-likelihood function but the algorithm becomes less stable due to the lack of regularization. 


\section{Online Estimation for Multivariate Hawkes Processes}\label{sec::algorithm}

\begin{algorithm}[!t]
	\caption{NonParametric OnLine Estimation for MHP (NPOLE-MHP)}  \label{alg::algorithm}
	\begin{algorithmic}[1]
		\STATE {\bf input:} a sequence of step sizes $\{\eta_k\}_{k=1}^{\infty}$ and a set of regularization coefficients $\zeta_{i,j}$s and $\omega_i$s, along with positive values of $\mu_{\min}$, $z$ and $\sigma$. \STATE {\bf output:} a sequence of estimates $\hat{\bmu}^{(k)}$ and $\hat{\bF}^{(k)}$ for $k=\{1,\ldots,M(t)\}$.
		\STATE Initialize $\hat{f}_{i,j}^{(0)}$ and $\hat{\mu}^{(0)}_{i}$ for all $i,j$. 
		\FOR{$k=0,...,M(t)-1$}
		\STATE  Observe the interval $[t_{k},t_{k+1})$, and compute $x_{i,k}$ for $i\in\{1,\ldots,p\}$.
		\FOR{$i=1,\ldots,p$}
		\STATE Set $\hat{\mu}_{i}^{(k+1)}\leftarrow \max\left\{\hat{\mu}_{i}^{(k)}-\eta_{k+1}\partial_{\mu_i}
		l_{i,k}\left(\lambda^{(z)}_i(\hat{\mu}_i^{(k)},\hat{\bbf}_{i}^{(k)})\right), \mu_{\min}\right\}$.
		\FOR{$j=1,\ldots,p$}
		\STATE Set $\hat{f}_{i,j}^{(k+\frac{1}{2})}\!\!\!\leftarrow\!\!\left[ \hat{f}_{i,j}^{(k)}-\eta_{k+1}\partial_{f_{i,j}}l_{i,k}\left(\lambda^{(z)}_i(\hat{\mu}_i^{(k)},\hat{\bbf}_{i}^{(k)})\right)\right]$, and  $\hat{f}_{i,j}^{(k+1)}\leftarrow\Pi\left[\hat{f}_{i,j}^{(k+\frac{1}{2})} \right].$

		\ENDFOR
		\ENDFOR
		\ENDFOR
	\end{algorithmic}
\end{algorithm}

We introduce our NonParametric OnLine Estimation for MHP (NPOLE-MHP) in Algorithm \ref{alg::algorithm}. We adopt an online gradient descent (OGD) framework as it achieves the known optimal asymptotic average regret among all online convex optimization algorithms. In particular, NPOLE-MHP updates the estimates of $\hat{\mu}_i$ and $\hat{f}_{i,j}(t)$ at each iteration using the corresponding gradient. The most important components of NPOLE-MHP are (i) the computation of the gradients and (ii) the projection operation, in lines 6 and 8, respectively. We now illustrate those procedures in detail.

\subsection{Computation of the Gradient} The computation of the gradient include two aspects: derivatives with respect to (i) $\mu_i$, and (ii) $f_{i,j}$.

{\bfseries\noindent Partial derivative with respect to $\mu_i$.} Recall the definition of $\l_{i,k}$ in \eqref{eq::inst_loss} and $\lambda^{(z)}_i$ in \eqref{eq::trunc}. It is easy to see that $\partial_{\mu_i}\lambda_i^{(z)}(\hat{\mu}_i^{(k)},\hat{\bbf}_i^{(k)})=1$, and therefore by chain rule, we have
\$
\partial_{\mu_i} l_{i,k}\left(\lambda^{(z)}_i(\hat{\mu}_i^{(k)},\hat{\bbf}_{i}^{(k)})\right)=(t_k-t_{k-1})-{x_{i,k}}\left[{\lambda^{(z)}_i\left(\hat{\mu}_i^{(k)},\hat{\bbf}_{i}^{(k)}\right)}\right]^{-1} +\omega_i\hat{\mu}_i^{(k)}.
\$
For simplicity, we define the following quantity
\#\label{eq::rhok}
\rho_k=(t_k-t_{k-1})-{x_{i,k}}\left[{\lambda^{(z)}_i\left(\hat{\mu}_i^{(k)},\hat{\bbf}_{i}^{(k)}\right)}\right]^{-1},
\#
which gives us
\#\label{eq::partialmu}
\partial_{\mu_i} l_{i,k}\left(\lambda^{(z)}_i(\hat{\mu}_i^{(k)},\hat{\bbf}_{i}^{(k)})\right)=\rho_k+\omega_i\hat{\mu}_i^{(k)}.
\#

Upon performing the gradient descent step on $\hat{\mu}_i^{(k)}$, the algorithm checks whether it is below $\mu_{\min}$. Since the triggering functions are constrained to be positive at each step of the algorithm, forcing $\hat{\mu}_i^{(k+1)}\geq\mu_{\min}$ guarantees that $\lambda^{(z)}_i(\hat{\mu}_i^{(k)},\hat{\bbf}_{i}^{(k)})\geq\mu_{\min}$ and is bounded away from 0 at each step of the algorithm.

{\bfseries\noindent Partial derivative with respect to $f_{i,j}$s.} We next discuss the update step for $\hat{f}_{i,j}^{(k)}(t)$ in line 8 of the algorithm. By the reproducing property of the kernel, we have
\$
\partial_{f_{i,j}}f_{i,j}(t_k-\tau_{j,n})=K(t_k-\tau_{j,n},\cdot).
\$
Therefore, by chain rule, we have
\#\label{eq::partial}
&\partial_{f_{i,j}}l_{i,k}\left(\lambda^{(z)}_i(\hat{\mu}_i^{(k)},\hat{\bbf}_{i}^{(k)})\right)=\rho_k\!\!\!\sum_{\tau_{j,n}\in[t_k-z,t_k)}\!\!\!\!K(t_k-\tau_{j,n},\cdot)+\zeta_{i,j}\hat{f}_{i,j}^{(k)}(\cdot).
\#
Afterward, a projection $\Pi[\cdot]$ is necessary to ensure that the estimated triggering functions are positive. According to representer theorem, the optimal positive function is a finite summation, and therefore, we define the projection operation for any $g(x)=\sum_{s\in\calS}a_sK(s,x)$, with $\calH_{+}:=\{f\in\calH:f(t)\geq 0\  \forall t\geq 0\}$, to be
\#\label{eq::projectionop}
\Pi[g]=\argmin_{f\in\calH_{+},f(x)=\sum_{s\in\calS}b_sK(s,x)}\|f-g\|_{\calH}.
\#

It is worth pointing out, however, that MHPs can be generalized so that the triggering functions need not be positive, as long as the intensity function is lower bounded by some constant $\lambda_{\min}>0$. For the generalized MHPs, the arrival of the events on one dimension potentially reduces the arrival rates of events in other dimensions, which happens in many applications. For example, the consumption of medication prevents the arrivals of heart attacks. While we restrain our attention to the less general case of positive triggering functions, our algorithm can be generalized to adapt to the more general case by modifying the projection step in line 8 such that $\lambda_i^{(z)}(\hat{\mu}_i^{(k)},\hat{\bbf}_i^{(k)})\geq\lambda_{\min}$ for each $k$.

\subsection{Projection of the Triggering Functions} In general, the projection step of the triggering function requires solving a constrained quadratic programming (QP) problem: $ \min \|f-\hat{f}_{i,j}^{(k+\frac{1}{2})}\|_\calH^2$ subject to $f\in\mathcal{H}$ and $f(t)\geq0,~\forall t\in\reals^+$. For NPOLE-MHP, we have $\hat{f}_{i,j}^{(k+\frac{1}{2})}(x)=\sum_{s\in\mathcal{S}}a_sK(s,x)$ and $f(x)=\sum_{s\in\mathcal{S}}b_sK(s,x)$ for some set $\calS$ due to the representer theorem, and the projection operation on $f(\cdot)$ reduces to solving a QP on the coefficient vector $\bbb\in\reals^{|\calS|}$:
\$
\|f-\hat{f}_{i,j}^{(k+\frac{1}{2})}\|_{\calH}^2=(\bbb-\ba)^\top \Kb (\bbb-\ba),
\$
where $\Kb=[K(s,s')]\in\reals^{|\calS|\times|\calS|}$ is the Gramian matrix. As $\ba$ is given, minimizing the distance between $f$ and $\hat{f}_{i,j}^{(k+\frac{1}{2})}$ is equivalent to minimizing the above quadratic form with respect to $\bbb$. Notice that the constraint has to hold for every $t>0$, resulting in infinite number of constraints in general. In order to simplify computation, one can approximate the solution by relaxing the constraints such that $f(t)\geq 0$ holds for only a finite set of $t$s within $[0,z]$.

{\bfseries\noindent Semidefinite programming (SDP) for polynomial kernels.} When the reproducing kernel is a polynomial kernel defined over $\reals\times\reals$, we can exploit its unique structure to formulate the projection as an SDP. This is mainly due to the solution of Hilbert's 17th problem \citep{bochnak2013real}, which states that a $2d$-degree polynomial is nonnegative if and only if it can be written as the sum of $d$-degree polynomials. In particular, let $K(x,y)=(1+xy)^{2d}$ and $K'(x,y)=(1+xy)^d$, then $f\in\calH_{+}$ is equivalent to $f\in\{\phi^\top(x)\Qb\phi(x):\Qb\succeq 0\}$ where $\phi(\cdot)$ is the feature map of $K'(\cdot,\cdot)$: $\phi^\top(x)\phi(y)=K'(x,y)$. With this intuition, the projection step can be formulated as an SDP problem \citep{vandenberghe1996semidefinite} as follows:
\begin{proposition}\label{prop::sosproj}
	Let $\mathcal{S}=\cup_{r\leq k}\{t_r-\tau_{j,n}:\ t_r-z\leq\tau_{j,n}<t_r\}$ be the set of $t_r-\tau_{j,n}$s. Let $K(x,y)=(1+xy)^{2d}$ and $K'(x,y)=(1+xy)^d$ be two polynomial kernels with $d\geq 1$. Furthermore, let $\Kb$ and $\Gb$ denote the Gramian matrices where the $i,j$-th element correspond to $K(s,s')$ and $K'(s,s')$, with $s$ and $s'$ being the $i$-th and $j$-th element in $\calS$. Suppose that $\ba\in\reals^{|S|}$ is the coefficient vector such that $\hat{f}_{i,j}^{(k+\frac{1}{2})}(\cdot)=\sum_{s\in\mathcal{S}}a_sK(s,\cdot)$, and that the projection step returns $\hat{f}_{i,j}^{(k+1)}(\cdot)=\sum_{s\in\mathcal{S}}b^*_sK(s,\cdot)$. Then the coefficient vector $\bbb^*$ can be obtained by 
	\#
	\bbb^*=\argmin_{\bbb\in\reals^{|\calS|}} \ -2\ba^\top\Kb\bbb+\bbb^\top\Kb\bbb, \quad \text{s.t.}  \ \ \Gb\cdot\diag(\bbb)+\diag(\bbb)\cdot\Gb\succeq0.
	\#
\end{proposition}

{\bfseries\noindent Nonconvex approaches.} Alternatively, positivity can be guaranteed by assuming $f_{i,j}(t)=\exp\{g_{i,j}(t)\}$ or $f_{i,j}(t)=g_{i,j}^2(t)$, where $g_{i,j}(t)\in\calH$. By minimizing the loss with respect to $g_{i,j}(t)$, one can naturally guarantee that $f_{i,j}(t)\geq 0$. This method was adopted in \citet{flaxman2016poisson} for estimating the intensity functions of nonhomogeneous Poisson processes. The drawback of this approach is that the objective function is no longer convex. However, as will be demonstrated in Figure \ref{fig::exp2}, the output of this method converges aligns with the ground truth when the initialization is close to the global minima.

\subsection{Computational Complexity}
Since $\hat{\bbf}_i$s can be estimated in parallel, we restrict our analysis to the case of a fixed $i\in\{1,\ldots,p\}$ in a single iteration.
For each iteration, the computational complexity comes from evaluating the intensity function and projection.  Since the number of arrivals within the interval $[t_k-z,t_k)$ is bounded by $p\kappa_z$ and $\kappa_z=\calO(1)$, evaluating the intensity costs $\calO(p^2)$ operations. For the projection in each step, one can truncate the number of kernels used to represent $f_{i,j}(t)$ to be $\calO(1)$ with controllable error (Proposition 1 of \citet{kivinen2004online}), and therefore the computation cost is $\calO(1)$. Hence, the per iteration computation cost of NPOLE-MHP is $\mathcal{O}(p^2)$. By comparison, parametric online algorithms (DMD, OGD of \citet{hall2014tracking}) also require $\calO(p^2)$ operations for each iteration, while the batch estimation algorithms (MLE-SGLP, MLE of \citet{xu2016learning}) require $\calO(p^2t^3)$ operations, which is caused by the update of the EM algorithm. 

\section{Theoretical Properties} \label{sec::result}

We now discuss the theoretical properties of NPOLE-MHP. We start with defining the regret.
\subsection{Regret Bounds}
\begin{definition}\label{def:1}
	The regret of Algorithm \ref{alg::algorithm} at time $t$ is given by
	\$
	R_t^{(\delta)}(\lambda^{(z)}_i(\mu_i,\bbf_{i})):=\sum_{k=1}^{M(t)}\left(l_{i,k}(\lambda^{(z)}_i(\hat{\mu}_i^{(k)},\hat{\bbf}_i^{(k)}))-l_{i,k}(\lambda^{(z)}_i({\mu}_i,{\bbf}_i))\right),
	\$
	where $\hat{\mu}_i^{(k)}$ and $\hat{\bbf}_i^{(k)}$ denote the estimated base intensity and the triggering functions, respectively.
\end{definition}

\begin{theorem} \label{thm::main}
	Suppose that the  $p$-dimensional MHP that satisfies Assumptions \ref{asm::stationary} and \ref{asm::f}. Let $\zeta=\min_{i,j}\{\zeta_{i,j},\omega_i\}$, and $\eta_k=1/(\zeta k+b)$ for some positive constants $b$. Then 
	\$
	R_t^{(\delta)}(\lambda_i^{(z)}(\mu_i,\bbf_i))\leq C_1(1+\log M(t)),
	\$
	where $C_1=2(1+p\kappa_z^2)\zeta^{-1}|\delta-\mu_{\min}^{-1}|^2$.
\end{theorem}

The regret bound of Theorem \ref{thm::main} resembles the regret bound for a typical online learning algorithm with strong convex objective function (see for example, Theorem 3.3 of \citet{hazan2016introduction}). When $\delta$, $\zeta$ and $\mu_{\min}^{-1}$ are fixed, $C_1=\calO(p)$, which is intuitive as one needs to update $p$ functions at each iteration. Note that the regret in Definition \ref{def:1}, encodes the performance of Algorithm \ref{alg::algorithm} by comparing its loss with the approximated loss. Below, we compare the loss of Algorithm \ref{alg::algorithm} with the original loss in (\ref{eq::ltdelta0}).

\begin{corollary}\label{cor::main} Under the same assumptions as Theorem \ref{thm::main}, we have
	\#\label{eq::bound2}
	\sum_{k=1}^{M(t)}\left(l_{i,k}(\lambda^{(z)}_i(\hat{\mu}_i,\hat{\bbf}_i^{(k)}))-l_{i,k}(\lambda_i({\mu}_i,{\bbf}_i))\right)\leq C_1[1+\log M(t)]+C_2 N(t),
	\#
	where $C_1$ is defined in Theorem \ref{thm::main} and $C_2=(1+\kappa_1\mu_{\min}^{-1})\varepsilon(z)+\delta\varepsilon'(0).$
\end{corollary}

Note that $C_2N(t)$ is due to discretization and truncation steps and it can be made arbitrary small for given $t$ by setting small $\delta$ and large enough $z$.

\subsection{Generalization Error Bounds}

Generalization error bounds are useful tools to characterize how well an algorithm perform on unseen data. Consider a general nonparametric estimation setting, in which a function $f\in\calF$ is obtained upon minimizing an objective function determined from a set of samples, denoted by the set $\calS$. The samples within $\calS$, denoted by $\{S_i\}_{i=1}^{|\calS|}$, are generated according to some joint distribution, and under a specific objective function $l(\cdot,\cdot)$, the generalization error bound refers to an upper bound for the following generalization error:
\#\label{eq::generalizationerror}
D(f)=\left|\frac{1}{|\calS|}\sum_{i=1}^{|\calS|}\ell(f,S_i)-\Expect\left[\frac{1}{|\calS|}\sum_{i=1}^{|\calS|}\ell(f,S_i)\right]\right|.
\#
The bound is typically obtained by first noticing $D(f)\leq\sup_{f\in\calF}D(f)$, and then providing an upper bound on $\sup_{f\in\calF}D(f)$. For MHPs, we define the generalization error to be
\#\label{eq::gerror}
D(\hat{\bbf}_i):=\frac{1}{t}L_{i,t}^{(\delta)}(\lambda^{(z)}_i(\hat{\bbf}_i))-\Expect\left[\frac{1}{t}L_{i,t}^{(\delta)}(\lambda^{(z)}_i(\hat{\bbf}_i))\right],
\#
where the expectation is taken over the distribution of the intensity process. Note that, although $L_{i,t}^{(\delta)}(\cdot)$ is the summation of $M(t)$ instantaneous objective functions, we average it by $1/t$ instead of $1/M(t)$. This design is based on the fact that $L_{i,t}^{(\delta)}(\cdot)$ is an approximation of the negative log-likelihood. It also simplifies the analysis since otherwise one would have to take into concern the randomness of $M(t)$.

It is immediate that \eqref{eq::gerror} is drastically different compared to \eqref{eq::generalizationerror} because the correlation between different arrivals: a slight perturbation in the arrival time of one event could lead to changes in both the arrival times and the number of events consequently. Therefore, we consider the generalization error assuming that the number of arrivals $N(t)$ is fixed, namely 
\#\label{eq::conditionalge}
D_M(\hat{\bbf}_i):=\frac{1}{t}L_{i,t}^{(\delta)}(\lambda^{(z)}_i(\hat{\bbf}_i))-\Expect\left[\frac{1}{t}L_{i,t}^{(\delta)}(\lambda^{(z)}_i(\hat{\bbf}_i))\bigg\vert M(t)\right].
\#
This notion of generalization error allows us to characterize the performance of NPOLE-MHP from a stability point of view. For MHPs with stationary increments, $t\to\infty$, $M(t)/t\to t/\delta+\|\bar{\blambda}\|_1t$. Therefore, by the law of large numbers,
\$
\bar{D}(\hat{\bbf}_i):=\Expect\left[\frac{1}{t}L_{i,t}^{(\delta)}(\lambda^{(z)}_i(\hat{\bbf}_i))\bigg\vert M(t)\right]-\Expect\left[\frac{1}{t}L_{i,t}^{(\delta)}(\lambda^{(z)}_i(\hat{\bbf}_i))\right] \to 0
\$
in probability. Hence, for large $t$, $D_M(\cdot)$ can serve the purpose of $D(\cdot)$. We now state the result of \eqref{eq::conditionalge} in Theorem \ref{thm::concentration}.


\begin{theorem}\label{thm::concentration}
	For any fixed $\bbf_i\in\calH_{+}^p$ satisfying Assumptions \ref{asm::stationary}, \ref{asm::f}, and the assumption that $\sup_{j}\|f_{i,j}\|_{\calH}\leq U$, let $D_M({\bbf}_i)$ be defined as in \eqref{eq::conditionalge}. Then, for every $M(t)$, 
	\$
	\Prob\left[\left|D_M({\bbf}_i)\right|\geq \epsilon\right]\leq 2\exp\left\{-\frac{2\epsilon^2 t}{4\kappa_1C_L^4\delta^2\left(\sum_{j=1}^{p}\zeta_{i,j}^{-1}\right)^2}\right\}.
	\$
	Moreover, when $\calH$ is the RKHS associated with the Gaussian kernel, and for $\epsilon=\Omega(t^{\nu})$ with $\nu>-0.5$, there exists constants $C_1$ and $C_2$ such that $\sup_{\bbf_i\in \calH_{+}^p,\|f_{i,j}\|_{\calH}\leq U} |D_M(\bbf_{i})|\leq \epsilon$ with probability at least $1-2\exp\{C_1-C_2t^{1+2\nu}\}$.

\end{theorem}
See Appendix \ref{pf::ge} for proof.


\section{Extensions to Marked and Spatial Point Processes}

In this section, we generalize NPOLE-MHP to other point processes. In particular, we consider multivariate marked Hawkes processes, and spatial Hawkes processes.

\subsection{Multivariate Marked Hawkes Processes}

Multivariate marked Hawkes processes (MMHPs) have received relatively less attention compared to their unmarked version. Statistical properties such as limit theorems and the convergence rate to equilibrium state are studied in \citet{karabash2015limit,bremaud2002rate}, whereas a nonparametric estimation framework was proposed in \citet{bacry2014second} based on second order statistics.

In this section, we adapt NPOLE-MHP to online estimation of MMHPs. While it is commonly assumed that the marks and the arrivals have independent effects on the intensity function, \ie,
\$
f_{i,j}(t,x)=g_{i,j}(t)h_{i,j}(x)
\$
for some functions $g_{i,j}(\cdot)$ and $h_{i,j}(\cdot)$, we do not require this assumption for NPOLE-MMHP.

The key to generalizing NPOLE-MHP to adapt to the existence of marks is to adopt a two dimensional kernel $K(\xb,\yb):\reals^2\times\reals^2\to\reals$ where $\xb$ and $\yb$ are two dimensional vectors consisting of a time variable and the value of the mark. The nonparametric online estimate of the triggering functions can then be obtained by using the following expression of functional gradient in line 9 of Algorithm \ref{alg::algorithm}.
\$
\partial_{f_{i,j}}l_{i,k}\left(\lambda_i^{(z)}(\hat{\mu}_i^{(k)},\hat{\bbf}_i^{(k)})\right)=\rho_k\sum_{\tau_{j,n\in[t_k-z,t_k)}}K([t_k-\tau_{j,n},v_{j,n}],\cdot)+\zeta_{i,j}\hat{f}_{i,j}^{(k)}(\cdot).
\$
Furthermore, when
$
f_{i,j}(t,v)=g_{i,j}(t)h_{i,j}(v),
$
we can update $g_{i,j}(t)$ and $h_{i,j}(v)$ separately by substituting the functional gradient in line 9 of Algorithm \ref{alg::algorithm} with separate updates of $g_{i,j}(t)$ and $h_{i,j}(v)$, with the functional gradients being
\$
&\partial_{g_{i,j}}l_{i,k}\left(\lambda^{(z)}_i(\hat{\mu}_i^{(k)},\hat{\bbf}_{i}^{(k)})\right)=\rho_k\!\!\!\sum_{\tau_{j,n}\in[t_k-z,t_k)}\!\!\!\!h_{i,j}(v_{j,n})K(t_k-\tau_{j,n},\cdot)+\zeta_{i,j}\hat{g}_{i,j}^{(k)}(\cdot)
\$
and
\$
&\partial_{h_{i,j}}l_{i,k}\left(\lambda^{(z)}_i(\hat{\mu}_i^{(k)},\hat{\bbf}_{i}^{(k)})\right)=\rho_k\!\!\!\sum_{\tau_{j,n}\in[t_k-z,t_k)}\!\!\!\!g_{i,j}(t_k-\tau_{j,n})K(t_k-\tau_{j,n},\cdot)+\zeta_{i,j}\hat{h}_{i,j}^{(k)}(\cdot),
\$
respectively. However, while the approach exploits the multiplicative structure of the triggering function, the objective function is no longer convex when estimating $g_{i,j}(t)$ and $h_{i,j}(v)$ separately.

\subsection{Spatial Hawkes Processes}

Lastly, we generalize NPOLE-MHP to spatial Hawkes processes. Spatial Hawkes processes are Hawkes processes where the arrivals lie within $\reals^p$. The $n$-th arrival consists of a time stamp $\tau_{n}$, and a location $\xb_n$, and the intensity at $\xb$ at time $t$ can be written as
\$
\lambda(t,\xb)=\mu+\sum_{n=1}^{N(t)}f(t-\tau_n,\xb-\xb_n).
\$
The spatial Hawkes processes can be viewed as a generalization to MHPs, for the latter of which the arrivals are restricted to $p$ distinct directions only. As a concrete example, a spatial Hawkes process can be used to model the crime happening in an entire area on the map, while MHPs can only model crimes that happen at certain locations.

We present the adpated version of NPOLE-MHP to spatial Hawkes processes in the following algorithm, where the subscripts $i,j$ in NPOLE-MHP are now ignored as there's only one base intensity and triggering function.

\begin{algorithm}[t]
	\caption{NonParametric OnLine Estimation for Spatial Hawkes Processes (NPOLE-SHP)}  \label{alg::NPOLE-SHP}
	\begin{algorithmic}[1]
		\STATE {\bf input:} a sequence of step sizes $\{\eta_k\}_{k=1}^{\infty}$, a pair regularization coefficients $\zeta$ and $\omega$, along with positive values of $\mu_{\min}$, $z$ and $\sigma$.
		\STATE {\bf output:} a sequence of estimates $\hat{\mu}^{(k)}$ and $\hat{f}^{(k)}$ for $k=\{1,\ldots,M(t)\}$.
		\STATE Initialize $\hat{f}^{(0)}$ and $\hat{\mu}^{(0)}$. 
		\FOR{$k=0,...,M(t)-1$}
		\STATE Observe the interval $[t_{k},t_{k+1})$, and compute $x_{k}$.
		\STATE Set $\hat{\mu}^{(k+1)}\leftarrow \max\left\{\hat{\mu}^{(k)}-\eta_{k+1}\partial_{\mu}
		l_{k}\left(\lambda^{(z)}(\hat{\mu}^{(k)},\hat{f}^{(k)})\right), \mu_{\min}\right\}$.
		\STATE Set $\hat{f}^{(k+\frac{1}{2})}\!\!\!\leftarrow\!\!\left[ \hat{f}^{(k)}-\eta_{k+1}\partial_{f}l_{k}\left(\lambda^{(z)}(\hat{\mu}^{(k)},\hat{f}^{(k)})\right)\right]$, and  $\hat{f}^{(k+1)}\leftarrow\Pi\left[\hat{f}^{(k+\frac{1}{2})} \right].$
		\ENDFOR
	\end{algorithmic}
\end{algorithm}

Particularly, the functional gradient in line 7 of Algorithm \ref{alg::NPOLE-SHP} takes the form
\$
&\partial_{f}l_{k}\left(\lambda^{(z)}(\hat{\mu}^{(k)},\hat{f}^{(k)})\right)=\rho_k(\xb)\!\!\!\sum_{\tau_{n}\in[t_k-z,t_k)}\!\!\!\!K([t_k-\tau_{n},\xb-\xb_n],\cdot)+\zeta\hat{f}^{(k)}(\cdot),
\$
where $\rho_k(\xb)$ has the same form as \eqref{eq::rhok}, except that it now depends on the specific location of $\xb$. As a consequence, the nonparametric estimation of the spatial Hawkes process requires significantly more time and memory due to the increased complexity in evaluating $\rho_k(\xb)$ used in the gradient.


\section{Numerical Experiments}\label{sec::exp}

We now demonstrate the performances of NPOLE-MHP and its generalizations on both synthetic and real data. On synthetic data, we compare our algorithm's performance to that of online parametric algorithms (DMD, OGD of \citet{hall2014tracking}) and nonparametric batch learning algorithms (MLE-SGLP, MLE of \citet{xu2016learning}). Our synthetic data is generated by repeatedly evaluating the intensity function upon each arrival, and generating the next arrival as the first arrival among $p$ nonhomogenous Poisson processes. Other simulation schemes exist, such as a clustered Poisson process scheme \citep{dassios2013exact}.

We use three types of evaluation metrics. (i) We assess the visual goodness-of-fit of estimating each triggering functions. (ii) We compare the numeric performances of the algorithms by measuring the log-likelihood of their estimates. When multiple trials are averaged over synthetic data, we also use a  metric named ``average $L_1$ error", which is defined as the average of $\sum_{i=1}^{p}\sum_{j=1}^{p}\|f_{i,j}-\hat{f}_{i,j}\|_{L_1[0,z]}$ over multiple trials. (iii) We compare the scalability of NPOLE-MHP over both the dimension $p$ and time horizon $T$.

\subsection{Synthetic Data for Testing NPOLE-MHP}

Consider a 5-dimensional MHP with $\mu_i=0.05$ for all dimensions. We set the triggering functions as
\begin{small}
	\begin{align*}
	\bF=
	\begin{bmatrix}
	e^{-2.5t} & 0 & 0 & e^{-10(t-1)^2} & 0\\
	2^{-5t} & (1+\cos(\pi t))e^{-t}/2 & e^{-5t} & 0 & 0\\
	0 & 2e^{-3t} & 0 & 0 & 0\\
	0 & 0 & 0 & 0.6e^{-3t^2}+0.4e^{-3(t-1)^2}& e^{-4t} \\
	0 & 0 & te^{-5(t-1)^2} & 0 & e^{-3t}
	\end{bmatrix}.
	\end{align*}
\end{small}The design of $\bF$ allows us to test NPOLE-MHP's ability of detecting (i) exponential triggering functions with various decaying rate; (ii) zero functions; (iii) functions with delayed peaks and tail behaviors different from an exponential function.

{\bfseries\noindent Goodness-of-fit.} We run NPOLE-MHP over a set of data with $T=10^5$ and around $4\times 10^4$ events for each dimension. The parameters are chosen by grid search over a small portion of data, and the parameters of the benchmark algorithms are fine-tuned. In particular, we set the discretization level $\delta=0.05$, the window size $z=3$, the step size $\eta_k=(k\delta/20+100)^{-1}$, and the regularization coefficient $\zeta_{i,j}\equiv\zeta=10^{-8}$. The performances of NPOLE-MHP and benchmarks are shown in Figure \ref{fig::exp2}. Complete set of results can be found in Appendix \ref{app::expdet}. We see that NPOLE-MHP captures the shape of the function much better than the DMD and OGD algorithms with mismatched forms of the triggering functions. It is especially visible for $f_{1,4}(t)$ and $f_{2,2}(t)$. In fact, our algorithm scores a similar performance to the batch learning MLE estimator, which is optimal for any given set of data. 

{\bfseries\noindent Run time comparison.} The simulation of the DMD and OGD algorithms took 2 minutes combined on a Macintosh with two $6$-core Intel Xeon processor at 2.4 GHz, while NPOLE-MHP took 3 minutes. The batch learning algorithms MLE-SGLP and MLE in \cite{xu2016learning} each took about 1.5 hours. Therefore, our algorithm achieves the performance similar to batch learning algorithms with a run time close to that of parametric online learning algorithms.

{\bfseries\noindent Effects of the hyperparameters: $\delta$, $\zeta_{i,j}$, and $\eta_k$.} We investigate the sensitivity of NPOLE-MHP with respect to the hyperparameters, measuring the ``averaged $L_1$ error" defined at the beginning of this section. We independently generate 100 sets of data with the same parameters, and a smaller $T=10^4$ for faster data generation. The result is shown in Table \ref{tab::hyperparameters}. For NPOLE-MHP, we fix $\eta_k=1/(k/2000+10)$. MLE and MLE-SGLP score around 1.949 with 5/5 inner/outer rounds of iterations. NPOLE-MHP's performance is robust when the regularization coefficient and  discretization level are sufficiently small. It surpasses MLE and MLE-SGLP on large datasets, in which case the iterations of MLE and MLE-SGLP are limited due to computational considerations. As $\zeta$ increases, the error decreases first before rising drastically, a phenomenon caused by the mismatch between the loss functions. 
For the step size, the error varies under different choice of $\eta_k$, which can be selected via grid-search on a small portion of the data like most other online algorithms.

\begin{figure}[!t]
\centering
  \begin{tabular}{ccc}
		\includegraphics[width=0.3\textwidth]{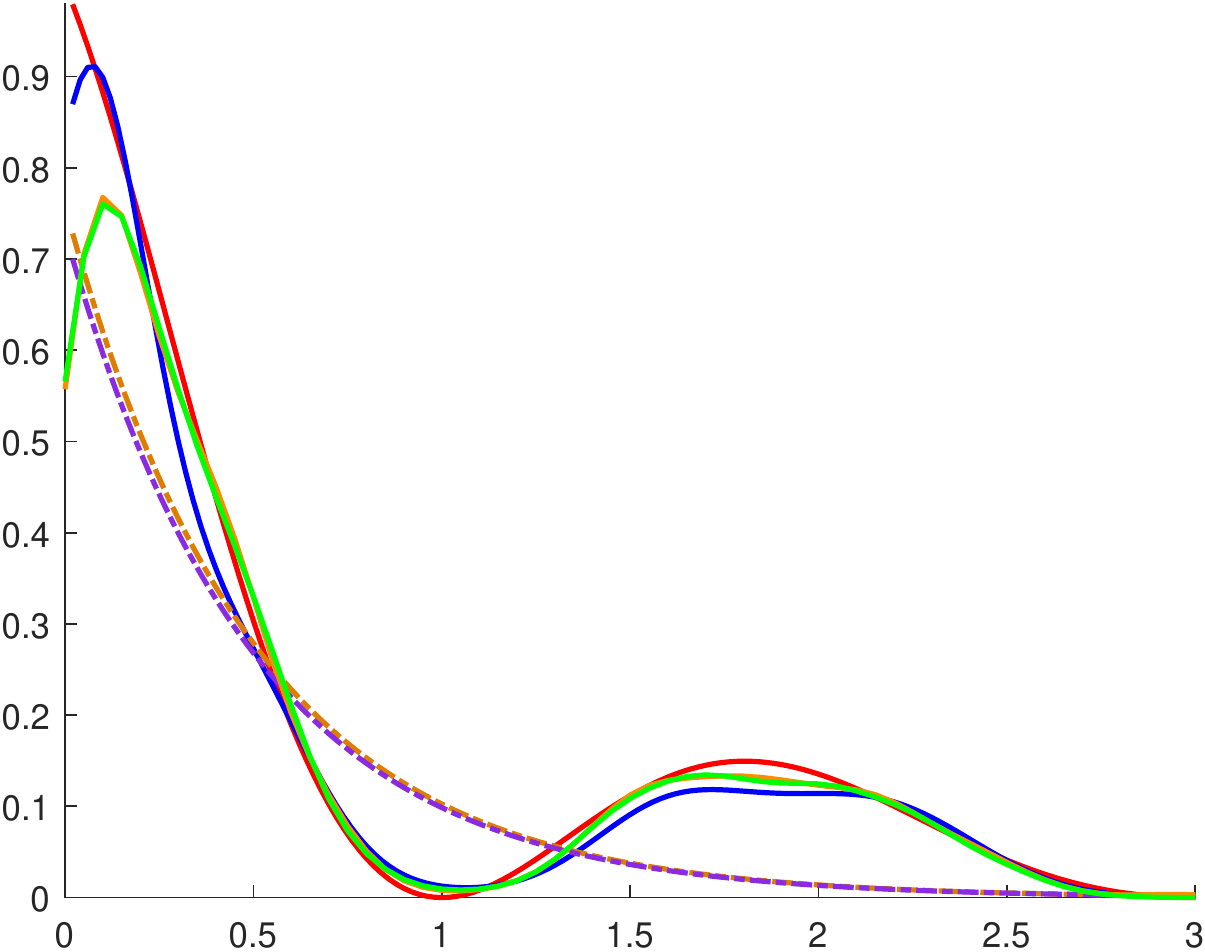}
	 & 	\includegraphics[width=0.3\textwidth]{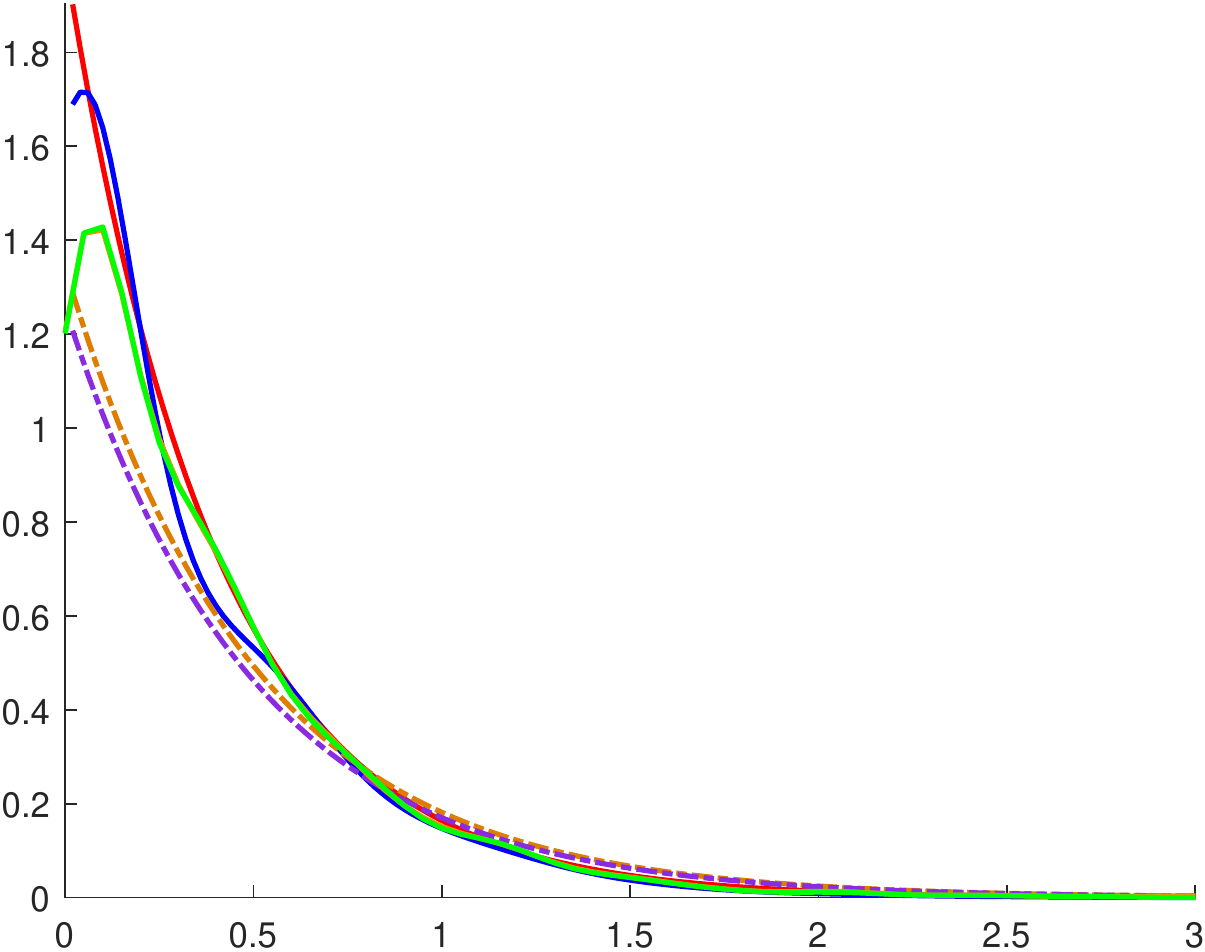}
	 &  \includegraphics[width=0.3\textwidth]{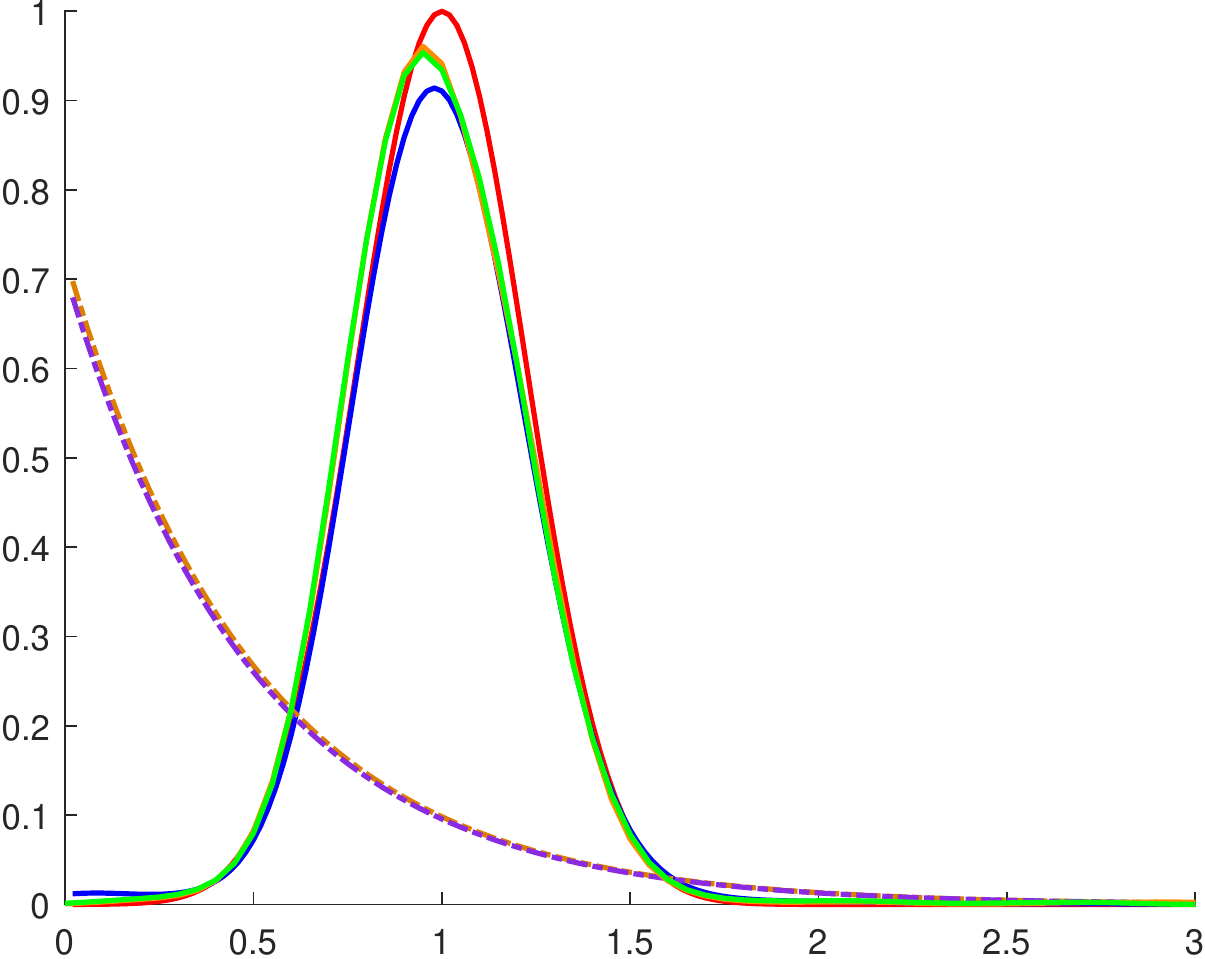}\\
	  (a) $f_{2,2}(t)$ & (b) $f_{3,2}(t)$ & (c) $f_{1,4}(t)$\\
	 \end{tabular}
	\begin{minipage}{0.5\textwidth}
		\centering
		\includegraphics[width=\textwidth]{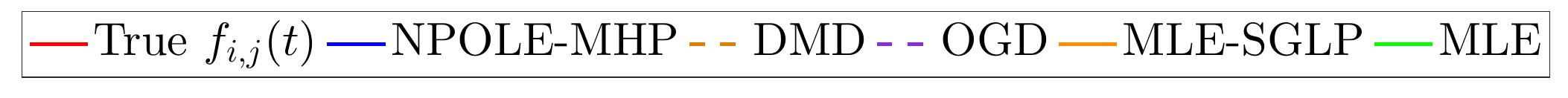}
	\end{minipage}
	\caption{Performances of different algorithms for estimating $F$.}
	\label{fig::exp2}
\end{figure}

\begin{table}[!t]
	\centering
	\begin{minipage}{0.45\textwidth}
		\centering
		\begin{tabular}[!t]{|c|c|c|c|c|c|c|}
			\hline
			\multicolumn{2}{|c|}{}  & \multicolumn{5}{c|}{Regularization $\log_{10}\zeta$} \\\cline{3-7}
			\multicolumn{2}{|c|}{} & $-8$ & $-6$ & $-4$ & $ -2 $ & $0$ \\\hline
			\multirow{5}{*}{\parbox{0.2cm}{\centering{$\delta$}}} & $0.01$  & $1.83$ & $1.83$ & $1.84$ & $4.15$ & $4.64$ \\\cline{2-7}
			& $0.05$ & $1.86$ & $1.86$ & $1.86$ & $3.10$  & $4.64$ \\\cline{2-7}
			& $0.1$ & $1.92$ & $1.92$ & $1.88$ & $2.73$ & $4.64$ \\\cline{2-7}
			& $0.5$ & $4.80$ & $4.80$ & $4.64$ & $2.19$ & $4.62$\\\cline{2-7}
			& $1$ & $5.73$ & $5.73$ & $5.58$ & $2.38$ & $4.59$ \\\hline
		\end{tabular}
		\vspace{0.1in}
		\caption{Effect of hyperparameters $\zeta$ and $\delta$, measured by the ``average $L_1$ error". }
		\label{tab::hyperparameters}
		
	\end{minipage}
	\hfill
	\begin{minipage}{0.45\textwidth}
		\centering
		\begin{tabular}{|c|c|c|c|c|}
			\hline
			\multicolumn{2}{|c|}{} & \multicolumn{3}{c|}{Horizon $T$ (days)} \\\cline{3-5}
			\multicolumn{2}{|c|}{} & $1.8$ & $3.6$ & $5.4$ \\\hline
			\multirow{5}{1.7cm}{\centering Dimension $p$} & $20$ & $3.9$ & $9.1$ & $15.3$ \\\cline{2-5} & $40$ & $4.6$ & $10.4$ & $17.0$ \\\cline{2-5} & $60$ & $4.6$ & $10.2$ & $16.7$ \\\cline{2-5} & $80$ & $4.5$ & $10.0$ & $16.4$ \\\cline{2-5} & $ 100$ & $4.5$ & $9.7$ & $15.9$ \\\hline
		\end{tabular}
		\vspace{0.1in}
		\caption{Average CPU-time for estimating one triggering function (seconds).}
		\label{tab::scalability}
		
	\end{minipage}
\end{table}


Lastly, we demonstrate the effect of discretization in Figure \ref{fig::synloss}. For $\delta\leq 0.05$, the stepwise loss evaluated with the true $f_{i,j}(t)$s varies very little. When we decrease $\delta$ from 1 to 0.05, however, the performance of NPOLE-MHP improves drastically.

\begin{figure}
	\centering
	\begin{minipage}{0.45\textwidth}
		\includegraphics[width=\textwidth]{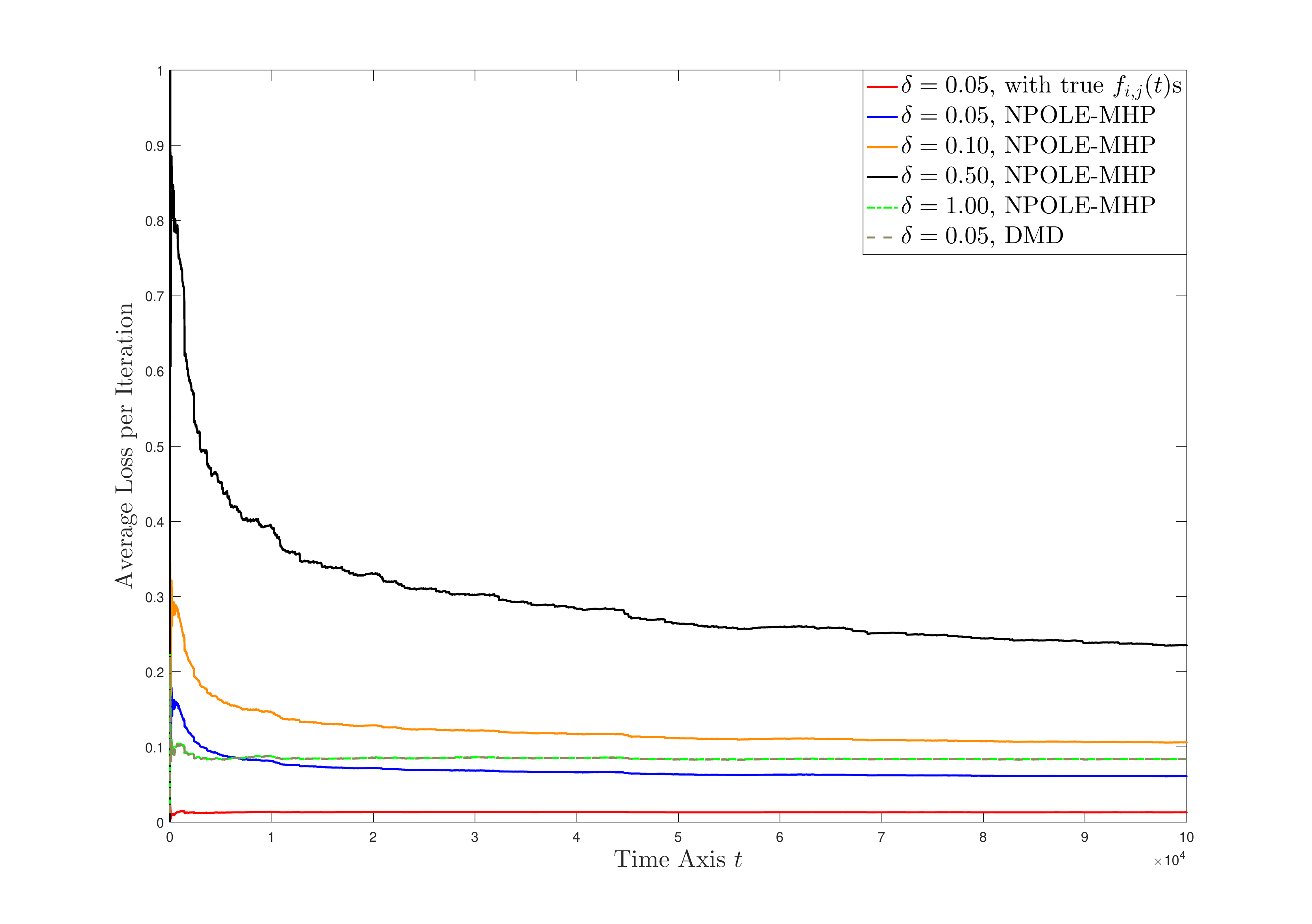}
		\caption{Effect of discretization on NPOLE-MHP.}
		\label{fig::synloss}
	\end{minipage}
	~
	\begin{minipage}{0.45\textwidth}
		\centering
		\includegraphics[width=\textwidth]{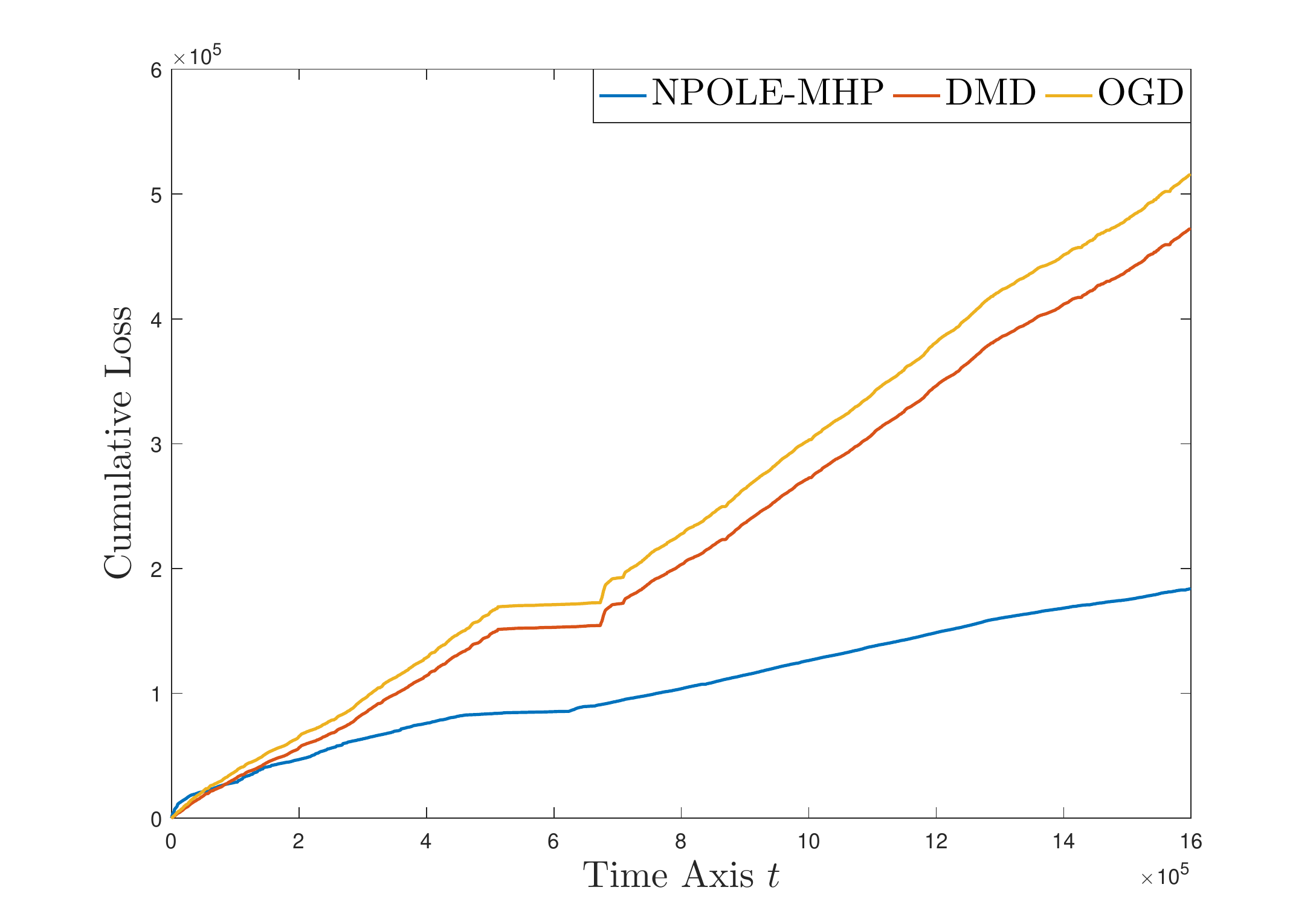}
		\caption{Cumulated loss of NPOLE-MHP, DMD and OGD on memetracker data.}
		\label{fig::cumloss}
	\end{minipage}	
\end{figure}

\subsection{Inferring Impact Between News Agencies with MHP}

We also tested the performance of our algorithm on the memetracker data \citep{leskovec2009meme}. The data collects from the web a set of popular phrases, including their content, the time at which they are published, and the url address of the articles that contributed to these occurrences. We study the relationship between different news agencies, and therefore model the data with a $p$-dimensional MHP where each dimension corresponds to the articles published by a news website. Note that a similar experiment was conducted in \citet{hall2014tracking}. Unlike \citet{hall2014tracking}, where all the data is used, we focus on only 20 websites that publish the most number of news articles using 18 days of data. The cumulative objective functions are plotted in Figure \ref{fig::cumloss}, where we set the window size to be 3 hours, discretization level $\delta=0.2$ second, and step size $\eta_k=1/(k\zeta+800)$ with $\zeta=10^{-10}$ for NPOLE-MHP. For DMD and OGD, we  set the step size $\eta_k=5/\sqrt{T/\delta}$. The result shows that NPOLE-MHP accumulates a smaller loss per step compared to OGD and DMD. 

\subsection{Inferring Crime Pattern in Chicago with MMHP}

\begin{figure}[t]
\centering
 	\begin{tabular}{cc}
		\includegraphics[width=0.45\textwidth]{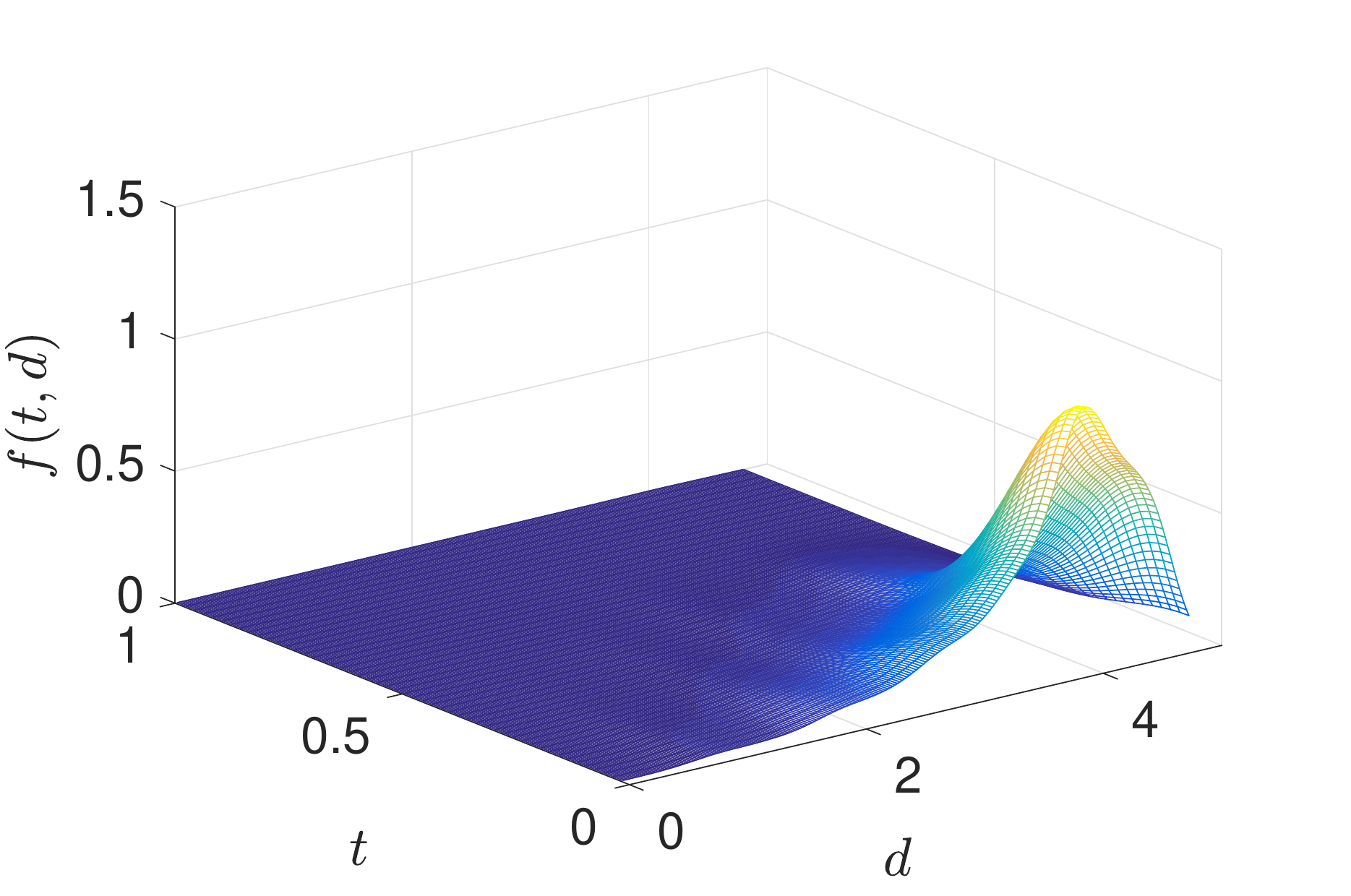}
		&\includegraphics[width=0.45\textwidth]{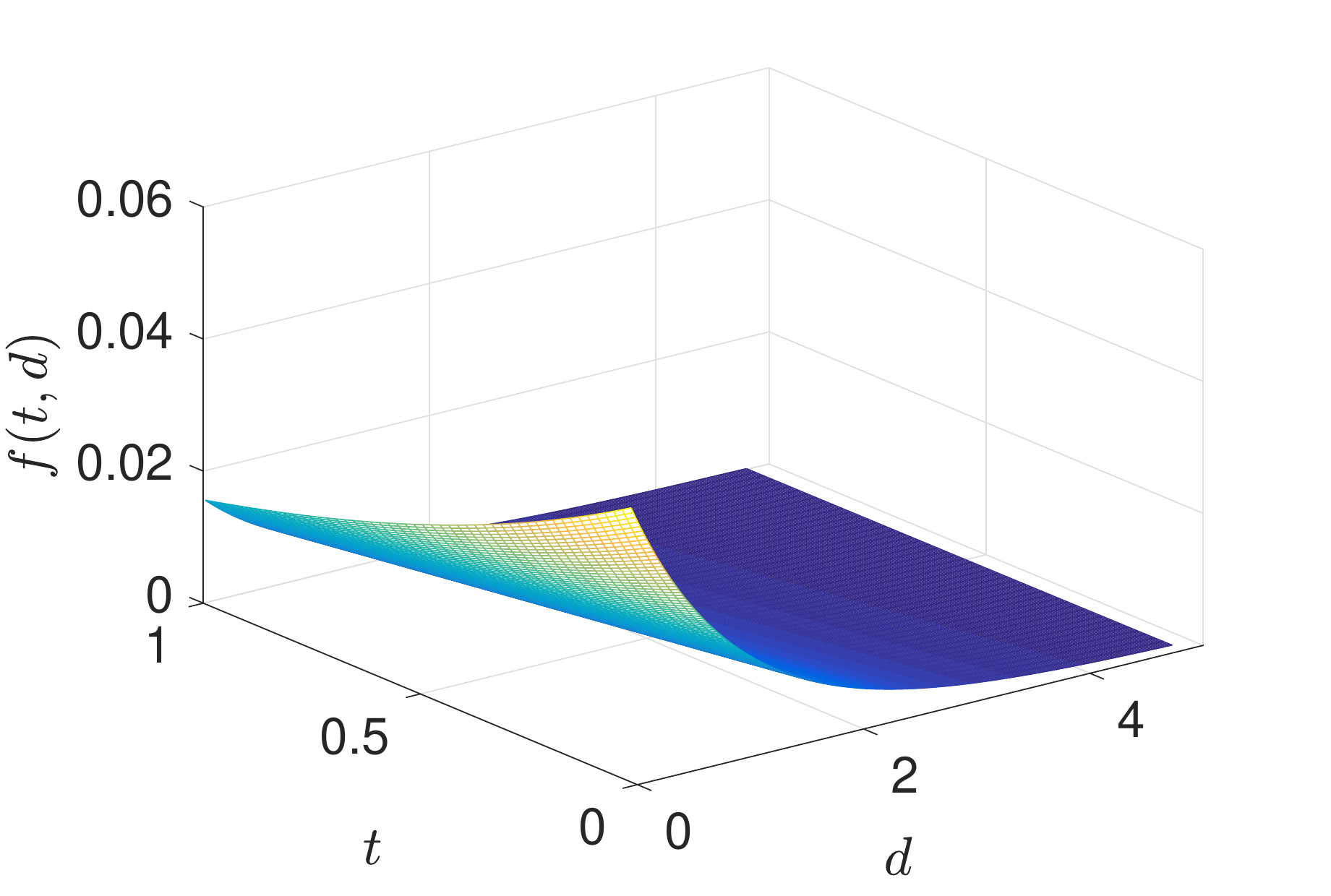}\\
		(a) NPOLE-MMHP & (b)OGD\\
	\end{tabular}
	\begin{tabular}{cc}
		\includegraphics[width=0.45\textwidth]{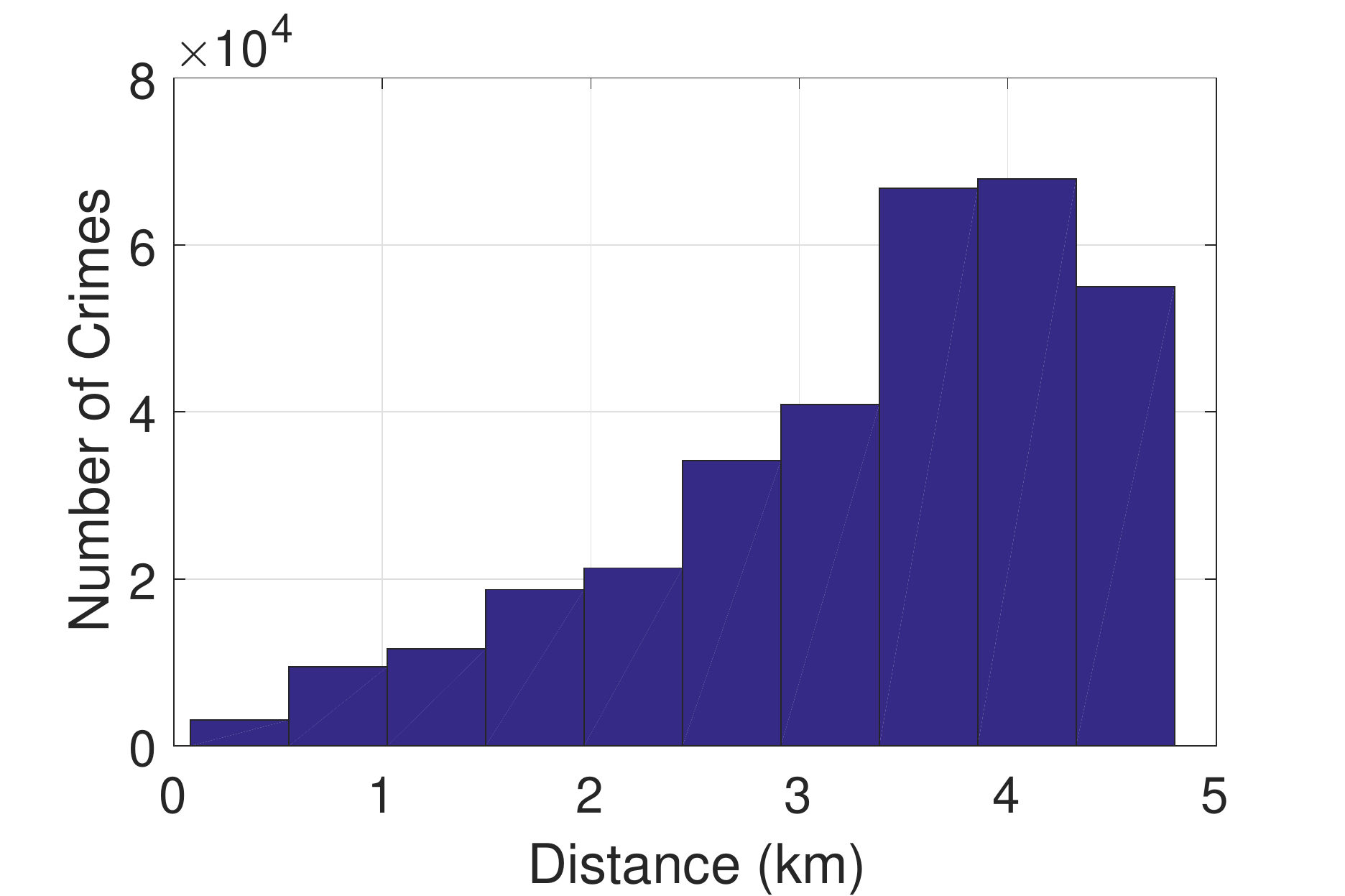}
		&\includegraphics[width=0.45\textwidth]{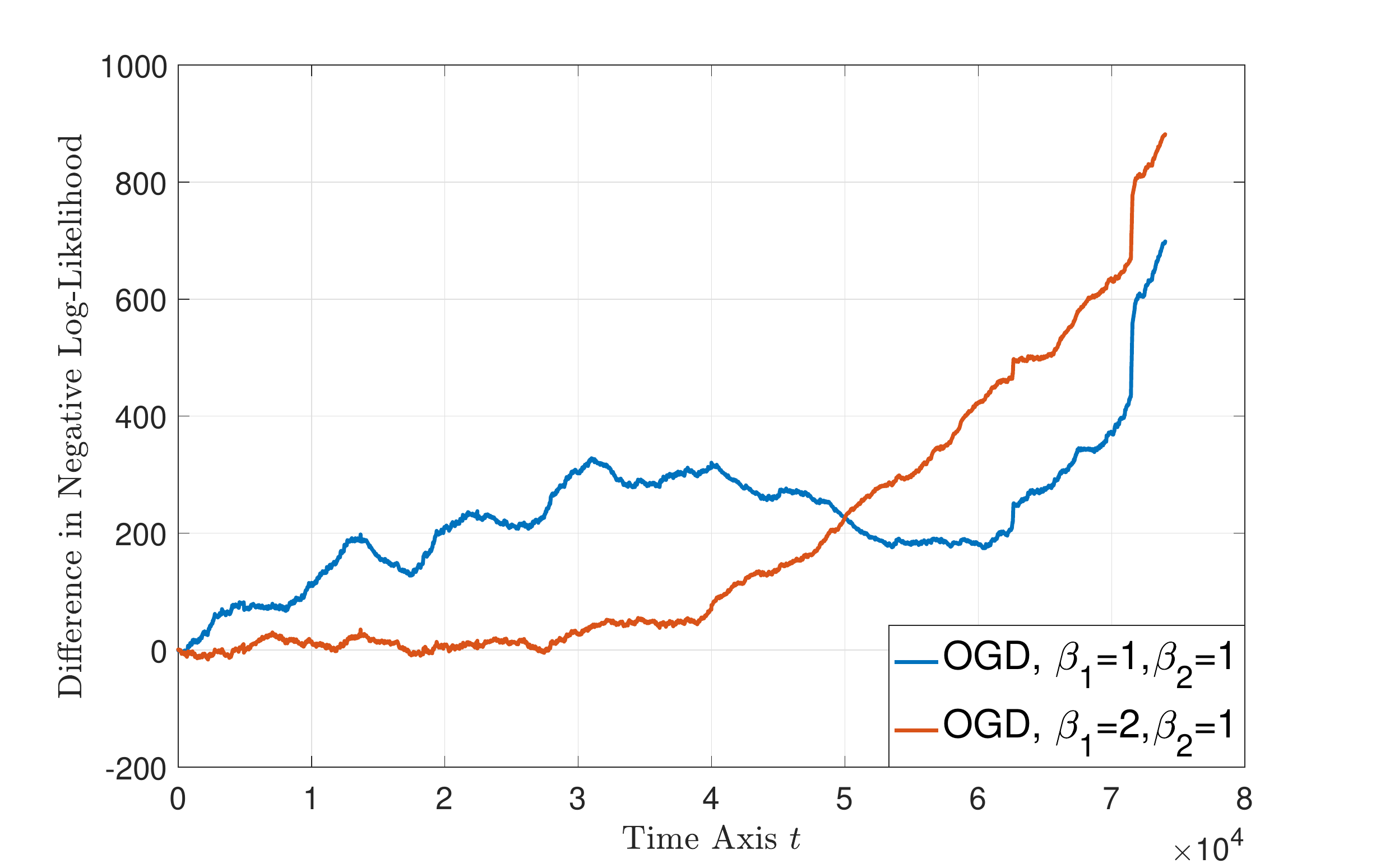}\\
		(c) Mark distribution & (d) Excess regret\\
	\end{tabular}
	
	\caption{(a) and (b) show the estimated triggering functions by NPOLE-MMHP and OGD, respectively. For OGD, the parametric form of the triggering function is $f(t,d)=\alpha\exp(-t-d)$. (c) shows the empirical distribution of the mark, which provides justification that NPOLE-MMHP generates a more reasonable estimation of the triggering function. (d) shows the excess regret, $R_{\rm{OGD}}-R_{\rm{NPOLE-MMHP}}$, under different parameters.}
	\label{fig::exp3}
\end{figure}

We next demonstrate the performance of NPOLE-MMHP on the Chicago crime data. The dataset collects the time and location information of various types of crimes within a span of 16 years from 2001 to 2017. In this study, we focus on the most common type of crime, thievery, and aim at inferring the impact of each incident on a pre-selected area. In this case, the dimension  $p=1$.

To model the data using an MMHP, we consider the thievery incidents that happen within 3 miles radius of the United Center, located at $x_c=(41.880706,-87.674230)$. Each thievery crime is considered as an event, and the location where it happens serves as its mark. For the sake of reducing computation burden, we simplify the model by assuming that a crime happening at time $\tau$ and location $x$ would affect the intensity function $\lambda(t)$ by an amount $f(t-\tau,\|x-x_c\|_2)$. That is, the location of the crime only matters through its distance towards the selected center of the region. Similar to MHP, we consider only those crimes that happen within a window of $t-\tau\leq z$, where we set $z$ to be 1 hour.

We compare our method with the parametric benchmark: $f(t,\|x-x_c\|_2)=\alpha\exp(-\beta_1t-\beta_2\|x-x_c\|_2)$, where $\alpha$ is estimated under different choices of $(\beta_1,\beta_2)$ pairs using online gradient descent (OGD). We tune the step size on the first 50\% of the updates, and use the remaining 50\% as test data.

The test results are shown in Figure \ref{fig::exp3}. In Figure \ref{fig::exp3}, we plot the triggering function $f(t,d)$ estimated by NPOLE-MMHP and OGD. As can be seen from Figure \ref{fig::exp3}(c), the majority of the crime happens at around 4km away from the selected area center. Therefore, the result generated by NPOLE-MMHP is a more natural estimation that reflects this piece of information. On the other hand, it can also be seen from Figure \ref{fig::exp3}(d) that the estimation of NPOLE-MMHP generates a slightly lower negative likelihood. Note that for various parameters pairs $(\beta_1,\beta_2)$ the performance of OGD varies slightly, indicating that the OGD model is almost merely fitting a Poisson process to the data.

\section{Conclusion and Discussions}\label{sec:conclusion}

We developed a nonparametric method for learning the triggering kernels of a multivariate Hawkes process (MHP) given time series observations.
To formulate the instantaneous objective function, we adopted the method of discretizing the time axis into small intervals of lengths at most $\delta$, and we derived the corresponding upper bound for approximation error. From this point, we proposed an online learning algorithm that is based on the framework of online kernel learning and exploits the interarrival time statistics under the MHP setup. Theoretically, we derived the regret bound for our algorithm, which is $\calO(\log T)$ when the time horizon $T$ is known a priori, and we showed that per iteration cost of the proposed algorithm is $\mathcal{O}(p^2)$. Numerically, we compared our algorithm's performance with parametric online learning algorithms and nonparametric batch learning algorithms. Results o both synthetic and real data showed that we are able to achieve similar performance to that of nonparametric batch learning algorithms with a run time comparable to parametric online learning algorithms.

\bibliographystyle{ims}
\bibliography{hawkes}

\newpage
\appendix
\section{Generalization Error Bound} \label{pf::ge}

\begin{proof}[Proof of Theorem \ref{thm::concentration}]

{\bfseries\noindent General settings for the proof.}
Suppose that there are $N(t)$ arrivals, and for each $k\in\{1,\ldots,N(t)\}$, denote the interval within which the $k$-th arrival falls using $\lfloor \tau_k\rfloor_\delta$. We wish to prove the statement by invoking the concentration inequality of bounded difference. Therefore we consider two sets of arrival times: $\calS=\{\tau_n\}_{n=1}^{N(t)}$ and $\calS'=\{\tau_n'\}_{n=1}^{N(t)}$ that differ by only one element: $\tau_{n_0}$ and $\tau_{n_0}'$, and denote the update times of NPOLE-MHP given $\calS$ and $\calS'$ as $\{t_k\}_{k=1}^{M(t)}$ and $\{t_k'\}_{k=1}^{M(t)}$, respectively.

By the design of the update rule, there exist $k,k'\in\{1,\ldots,M(t)\}$, such that $t_k=\tau_{n_0}$ and $t'_{k'}=\tau_{n_0}'$. Notice that $k$ and $k'$ need not be the same because the difference in the arrival times could span multiple intervals of length $\delta$, resulting in multiple updates in between. However, when the underlying MHP has stationary increments, $\tau_{n_0+1}-\tau_{n_0-1}=\calO(1)$, implying that $|k-k'|\delta=\calO(1)$. Therefore, we can assume $k=k'$ without loss of generality.

Suppose that the solution to the optimization problem \eqref{eq::problem} is $\bbf_i$ when the arrival times is the set $\calS$, and $\bbf_i'$ when the set of arrival times is the set $\calS'$. To avoid confusion, we write $L_{i,t}^{(\delta)}(\calS,\bbf_i)$ when referring to the objective function evaluated at $\bbf_i$, with the update times coming from the arrival set $\calS$.

{\bfseries\noindent Outline of the proof.} The proof can be obtained following the outline below.
\begin{itemize}
	\item First, we show that the discretized objective function with the truncated $\lambda_i^{(z)}(t)$ is Lipschitz continuous with respect to $\bbf_i$, with a Lipschitz constant $C_L$. That is, for any $\calS$, we have
	\#\label{eq::ge::p1}
	\frac{1}{t}\left| L_{i,t}^{(\delta)}(\calS,\lambda^{(z)}_i({\bbf}_i)) - L_{i,t}^{(\delta)}(\calS,\lambda^{(z)}_i({\bbf}'_i))\right|\leq C_L \sum_{j=1}^{p}\|f_{i,j}-f_{i,j}'\|_{\infty},
	\#
	where 
	$
	C_L=(\delta^{-1}+\kappa_1)\kappa_z|\delta-\mu_{\min}^{-1}|.
	$
	\item Second, when the regularization coefficient for $f_{i,j}(t)$ is $\zeta_{i,j}$, and when $\sup_{x\in\reals}K(x,x)\leq 1$, 
	\#\label{eq::ge::p5}
	\|f_{i,j}-f_{i,j}'\|_{\infty}\leq\|f_{i,j}-f_{i,j}'\|_{\calH},
	\#
	which holds as a property of the RKHS.
	\item Next, for $\bbf_i$ and $\bbf_i'$ learned from $\calS$ and $\calS'$, respectively, we have
	\#\label{eq::ge::p4}
	\|f_{i,j}-f_{i,j}'\|_{\calH}^2\leq\frac{2\delta C_L}{\zeta_{i,j}t}\cdot\|f_{i,j}-f_{i,j}'\|_{\infty}.
	\#
	\item Lastly, we prove the generalization error bound by combining the first three steps and invoking the McDiarmid concentration inequality.
\end{itemize}

{\noindent\bf Step 1: Proving the Lipschitz continuity, \eqref{eq::ge::p1}.} By the definition of the objective function, we have
\#\label{eq::ge::p2}
\left| L_{i,t}^{(\delta)}(\calS,\lambda^{(z)}_i({\bbf}_i))-L_{i,t}^{(\delta)}(\calS,\lambda^{(z)}_i({\bbf}'_i)) \right|&\leq\sum_{k=1}^{M(t)}\bigg|(t_k-t_{k-1})\left[\lambda_i^{(z)}(t_k,\bbf_i)-\lambda_i^{(z)}(t_k,\bbf_i')\right]-\notag\\&-x_{i,k}\left[\log\lambda_i^{(z)}(t_k,\bbf_i)-\log\lambda_i^{(z)}(t_k,\bbf_i')\right]\bigg|\notag\\&\leq \sum_{k=1}^{M(t)}|\delta-\mu_{\min}^{-1}|\cdot|\lambda_i^{(z)}(t_k,\bbf_i)-\lambda_i^{(z)}(t_k,\bbf_i')|,
\#
where in the last step, we have assumed, without generality, that $\delta$ is much smaller than $|\delta-\mu_{\min}^{-1}|$. In addition, we have also used the fact that $\log\lambda$ is $\mu_{\min}^{-1}$-Lipschitz continuous when $\lambda\geq\mu_{\min}>0$.

It is not hard to see that proving \eqref{eq::ge::p1} now reduces to proving the Lipschitz continuity of $\lambda_i^{(z)}(t_k,\bbf_i)$ for every $t_k$. Indeed,
\#\label{eq::ge::p3}
|\lambda_i^{(z)}(t_k,\bbf_i)-\lambda_i^{(z)}(t_k,\bbf_i')|&=\left|\sum_{j=1}^{p}\sum_{\tau_{j,n}\in[t_k-z,t_k)}\left[f_{i,j}(t_k-\tau_{j,n})-f_{i,j}'(t_k-\tau_{j,n})\right]\right|\notag\\&\leq \kappa_z\sum_{j=1}^{p}\sup_{\tau_{j,n}\in[t_k-z,t_k)}|f_{i,j}(t_k-\tau_{j,n})-f_{i,j}'(t_k-\tau_{j,n})|\leq \kappa_z\sum_{k=1}^{p}\|f_{i,j}-f_{i,j}'\|_{\infty}.
\#
Combining \eqref{eq::ge::p2} and \eqref{eq::ge::p3}, and noticing that $M(t)\leq t/\delta +\kappa_1t$, we reach \eqref{eq::ge::p1}.

{\bfseries\noindent Step 2: Proving \eqref{eq::ge::p5}.} Equation \eqref{eq::ge::p5} holds as a property for the RKHS. Since $f_{i,j}-f_{i,j}'\in\calH$, we have
\$
|f_{i,j}(x)-f_{i,j}'(x)|&=|\la K(x,\cdot),f_{i,j}-f_{i,j}'\ra_{\calH}|\leq \|K(x,\cdot)\|_{\calH}\|f_{i,j}-f_{i,j}\|_{\calH}=K(x,x)\|f_{i,j}-f_{i,j}\|_{\calH}\leq \|f_{i,j}-f_{i,j}\|_{\calH}
\$
for any $x$. Therefore \eqref{eq::ge::p5} holds by taking supremum over $x$ on both sides in the above inequality.

{\bfseries\noindent Step 3: Proving \eqref{eq::ge::p4}.} We now turn to prove \eqref{eq::ge::p4}, which upper bounds the RKHS norm of the difference between $f_{i,j}$ and $f_{i,j}'$. The upper bound on the right-hand side decays at the speed of $\calO(M(t)^{-1})$. This implies that the more updates used for learning $\bbf_i$ and $\bbf_i'$, the less likely that they will be different when we slightly perturb the arrival epoch of a single event.

To formally show this, fix $j=j_0$. In addition, for notational simplicity, we denote $\{1,\ldots,p\}\backslash \{j_0\}$ by $\{-j_0\}$, and define the instantaneous objective function without regularization as
\$
\iota_{i,k}(\lambda_i^{(z)}(t_k,\bbf_i))=(t_k-t_{k-1})\lambda_i^{(z)}(t_k,\bbf_i)-x_{i,k}\log\lambda_i^{(z)}(t_k,\bbf_i).
\$

Let $\phi:\calH\to\reals$ be a strictly convex function. Then the Bregman divergence with respect to $\phi$  is defined as
\$
\Delta_{\phi}(f_{i,j_0},f_{i,j_0}'):=\phi(f_{i,j_0})-\phi(f_{i,j_0}')-\la f_{i,j_0}-f_{i,j_0}',\nabla \phi(f_{i,j_0}')\ra.
\$
Set $\phi(f)=\|f\|_\calH^2$, then $\nabla\|f_{i,j_0}'\|_{\calH}^2=2f_{i,j_0}'$, and we have
$
\Delta_{\phi}(f_{i,j_0},f_{i,j_0}')=\|f_{i,j_0}-f_{i,j_0}'\|_\calH^2,
$
which further implies
\$
2\|f_{i,j_0}-f_{i,j_0}'\|_\calH^2=\Delta_{\phi}(f_{i,j_0},f_{i,j_0}')+\Delta_{\phi}(f_{i,j_0}',f_{i,j_0}).
\$
Adopt the following functional notation, where $\bg=[g_1,\ldots,g_p]\in\calH^p$:
\$
T_{\calS}(\bg):=I_{\calS}(\bg)+\sum_{j=1}^{p}\frac{\zeta_{i,j}}{2}\phi(g_{j})=\mathbb{L}_i(\calS,\lambda^{(z)}_i(\bg))/M(t),
\$
and
\$
I_{\calS}(\bg):=\frac{1}{M(t)}\sum_{k=1}^{M(t)}\left[(t_k-t_{k-1})\lambda_i(t_k,\bg)-x_{i,k}\log\lambda_i(t_k,\bg)\right].
\$

Consider two choices of $\bg$: $\bbf_i,\bbf_i'\in\calH^{p}$. Recall from the previous steps that those are the solutions to the optimization problem \eqref{eq::problem} when the observed sample sets are $\calS$ and $\calS'$, respectively. It is not hard to see that the functional $T_{\calS}$ is strictly convex with respect to $\phi$, and therefore the Bregman divergence can be defined for $T_{\calS}$ as well. Hence, by the linearity of the Bregman divergence with respect to $\phi$, we have
\$
\frac{\zeta_{i,j}}{2}\left(\Delta_{\phi}(f_{i,j_0},f_{i,j_0}')+\Delta_{\phi}(f_{i,j_0}',f_{i,j_0})\right)&\leq\Delta_{T_{\calS}}(\bbf_i',\bbf_i)+\Delta_{T_{\calS'}}(\bbf_i,\bbf_i')\\&=T_{\calS}(\bbf_i')-T_{\calS}(\bbf_i)+T_{\calS'}(\bbf_i)-T_{\calS'}(\bbf_i')\\&=I_{\calS}(\bbf_i')-I_{\calS}(\bbf_i)+I_{\calS'}(\bbf_i)-I_{\calS'}(\bbf_i'),
\$
where the second step holds true since $\bbf_i$ and $\bbf_i'$ are the minimizers and the gradient of the first order terms in the definition of $\Delta_{T_{\calS}}$ and $\Delta_{T_{\calS'}}$ become zero. Notice that $\calS$ and $\calS'$ only differ in one element, which occurs at the $k$-th update of NPOLE-MHP. As a consequence, this element in question only affects at most two instantaneous objective functions: $\iota_{i,k}$ and $\iota_{i,k+1}$. Therefore,
\$
I_{\calS}(\bbf_i')-I_{\calS'}(\bbf_i')=\frac{1}{M(t)}\sum_{n=k}^{k+1}\left[\iota_{i,n}(\lambda_i^{(z)}(t_n\bbf_i'))-\iota_{i,n}(\lambda_i^{(z)}(t_n',\bbf_i'))\right]
\$
and
\$
I_{\calS}(\bbf_i)-I_{\calS'}(\bbf_i)=\frac{1}{M(t)}\sum_{n=k}^{k+1}\left[\iota_{i,n}(\lambda_i^{(z)}(t_n,\bbf_i))-\iota_{i,n}(\lambda_i^{(z)}(t_n',\bbf_i))\right].
\$
From \eqref{eq::ge::p2}, we know that the Lipschitz continuity holds for any instantaneous objective function without regularization. Therefore,
\$
2\cdot \frac{\zeta_{i,j_0}}{2}\|f_{i,j_0}-f_{i,j_0}'\|_{\calH}^2&=\frac{\zeta_{i,j_0}}{2}\left(\Delta_{\phi}(f_{i,j_0},f_{i,j_0}')+\Delta_{\phi}(f_{i,j_0}',f_{i,j_0})\right)\\&\leq\frac{2}{M(t)}\cdot C_L\|f_{i,j}-f_{i,j}'\|_{\infty},
\$
which gives us the desired equation \eqref{eq::ge::p4}, upon noticing that $M(t)= t/\delta +N(t)\geq t/\delta$.

{\bfseries\noindent Step 4: Invoking McDiarmid concentration inequality.} Combining \eqref{eq::ge::p5} and \eqref{eq::ge::p4}, we have
\$
\|f_{i,j}-f_{i,j}'\|_{\calH}^2\leq \frac{2C_L}{\zeta_{i,j}M(t)}\cdot\|f_{i,j}-f_{i,j}'\|_{\infty}\leq\frac{2C_L}{\zeta_{i,j}M(t)}\cdot\|f_{i,j}-f_{i,j}'\|_{\calH},
\$
which implies
\$
\|f_{i,j}-f_{i,j}'\|_{\calH}\leq \frac{2C_L}{\zeta_{i,j}M(t)}.
\$
Plugging the above result into \eqref{eq::ge::p1}, we have
\#\label{eq::boundeddiff}
\frac{1}{t}\left|L_{i,t}^{(\delta)}(\calS,\lambda_i^{(z)}(\bbf_i))-L_{i,t}^{(\delta)}(\calS,\lambda_i^{(z)}(\bbf_i'))\right|&\leq C_L\sum_{j=1}^{p}\|f_{i,j}-f_{i,j}'\|_{\infty}\leq C_L\sum_{j=1}^{p}\|f_{i,j}-f_{i,j}'\|_{\calH}\notag\\&\leq \frac{2C_L^2}{M(t)}\sum_{j=1}^{p}\zeta_{i,j}^{-1}\leq\frac{2C_L^2\delta}{t}\sum_{j=1}^{p}\zeta_{i,j}^{-1}.
\#

This gives us the argument that $L_{i,t}^{(\delta)}(\calS,\lambda_i^{(z)}(\bbf_i))/t$ has bounded difference when one slightly perturbs one instance in the training data of $\bbf_i$. It can also be shown that $L_{i,t}^{(\delta)}(\calS,\lambda_i^{(z)}(\bbf_i))/t$ is bounded, by noticing that
\$
\iota_{i,k}(\lambda_i^{(z)}(t_k,\bbf_i))&\leq\delta\lambda_i(t_k)-x_{i,k}\log\lambda_i(t_k)\leq\delta\lambda_i(t_k)+|\log\mu_{\min}|\\&\leq|\log\mu_{\min}|+\delta\mu_i+\sum_{j=1}^{p}\sum_{n=1}^{N_j(t_k)}f_{i,j}(t_k-\tau_{j,n})\\&\leq|\log\mu_{\min}|+\delta(\mu_i+p\kappa_1\varepsilon(0)).
\$
Furthermore, when one fixes $\{t_i\}_{i=1}^{k-1}$, and perturbs $t_k$, the upper bound given in \eqref{eq::boundeddiff} is independent of the choice of $\{t_i\}_{i=k+1}^{N(t)}$. Therefore, we can invoke the generalized McDiarmid inequality in the form of Corollary 6.10 of \citet{mcdiarmid1989method}, and obtain
\$
\Prob\left[\left|\frac{1}{t}L_{i,t}^{(\delta)}(\bbf_i)-\Expect\left[\frac{1}{t}L_{i,t}^{(\delta)}(\bbf_i)\bigg\vert M(t)\right]\right|\geq \epsilon\right]\leq 2\exp\left\{-\frac{2\epsilon^2 t}{4\kappa_1C_L^4\delta^2\left(\sum_{j=1}^{p}\zeta_{i,j}^{-1}\right)^2}\right\}.
\$
\end{proof}

{\bf \noindent  The covering number and the $\epsilon$-net argument.}
Suppose that the true $\bbf_i$ lies within the part of the Hilbert space where $\sup_{j}\|f_{i,j}\|_{\calH}\leq U$, and denote this part of the RKHS as $\calF$. Let $\calV$ be an $\epsilon$-net for $\calF$. Then, for any $f\in\calF$, there exists $v\in\calV$ such that $\|f-v\|_{\calH}\leq \epsilon$. Recall that the risk function is $C_L$-Lipschitz, we have
\$
|D_M(\bbf_i)-D_M(\bv_i)|\leq 2C_L\sum_{j=1}^{p}\|f_{i,j}-v_{i,j}\|_{\calH}\leq 2C_Lp\epsilon.
\$
Therefore,
\$
\Prob\left[\sup_{\bbf_i\in\calF}|D(\bbf_i)|\geq x\right]&\leq \Prob\left[\sup_{\bv_i\in\calV^p}|D(\bv_i)|\geq x-2C_Lp\epsilon\right]\\&\leq2|\calV^{p}|\exp\left(-\frac{t\delta^{-1}\left(x-2C_Lp\epsilon-2C_L^2\delta t^{-1}\sum_{j=1}^{p}\zeta_{i,j}^{-1}\right)^2}{2(2C_L^2\sum_{j=1}^{p}\zeta_{i,j}^{-1}+|\log\mu_{\min}|+\delta(\mu_i+p\kappa_1\varepsilon(0)))^2}\right)\\&\leq 2p\exp\left(\log|\calV|-\frac{t\delta^{-1}\left(x-2C_Lp\epsilon-2C_L^2\delta t^{-1}\sum_{j=1}^{p}\zeta_{i,j}^{-1}\right)^2}{2(2C_L^2\sum_{j=1}^{p}\zeta_{i,j}^{-1}+|\log\mu_{\min}|+\delta(\mu_i+p\kappa_1\varepsilon(0)))^2}\right).
\$
For different RKHSs, $|\calV|$ has different forms. For the RKHS associated with a one-dimensional Gaussian kernel,
\$
\log|\calV|\leq \left(3\log\frac{U}{\epsilon}+\frac{54}{\sigma^2}+6\right)\left(7\log\frac{U}{\epsilon}+\frac{90}{\sigma^2}+14\right).
\$
Therefore, for any $x=\Omega(t^{\nu})$ with $\nu>-0.5$, we can pick $\epsilon=\Theta(x/(2C_Lp))$, and
\$
\lim_{t\to\infty}\Prob\left[\sup_{\bbf_i}|D(f_i)|\geq x\right]=0.
\$
Therefore the generalization error bound of the optimization problem (8) is $\calO(t^{-1/2})$.

\section{Proof of Proposition \ref{prop::approx}}\label{pf::prop::approx}

Fix the triggering functions $\bbf_i$ and the constant base intensity $\mu_i$. Then,
\#\label{eq::p11}
\left\vert L_{i,t}^{(\delta)}(\lambda^{(z)}_i)-L_{i,t}(\lambda_i)\right\vert&\leq \left\vert L_{i,t}^{(\delta)}(\lambda^{(z)}_i)-L_{i,t}^{(\delta)}(\lambda_i)\right\vert+\left\vert L_{i,t}^{(\delta)}(\lambda_i)-L_{i,t}(\lambda_i)\right\vert.
\#
We bound the first term on the right-hand side that is corresponding to the truncation error as follows:
\#\label{eq::p12}
\left\vert L_{i,t}^{(\delta)}(\lambda^{(z)}_i)-L_{i,t}^{(\delta)}(\lambda_i)\right\vert&=\left\vert\sum_{k=1}^{M(t)}(t_k-t_{k-1})\left(\lambda^{(z)}_i(t_k)-\lambda_i(t_k)\right)-\sum_{k=1}^{M(t)}x_{i,k}\left(\log\lambda^{(z)}_i(t_k)-\log\lambda_i(t_k)\right)\right\vert\notag\\
&\leq\left\vert\sum_{k=1}^{M(t)}(t_k-t_{k-1})\left(\lambda^{(z)}_i(t_k)-\lambda_i(t_k)\right)\right\vert+\left\vert\sum_{k=1}^{M(t)}\frac{x_{i,k}}{\mu_{\min}}\left(\lambda^{(z)}_i(t_k)-\lambda_i(t_k)\right)\right\vert\notag\\
&\leq\sum_{k=1}^{M(t)}(t_k-t_{k-1}+\frac{x_{i,k}}{\mu_{\min}})\left\vert\lambda^{(z)}_i(t_k)-\lambda_i(t_k)\right\vert.
\#
First, we define $\alpha(t):=\indc{t\leq z}$. 
Using the fact that $1-\alpha(t_k-\tau_{j,n})=\indc{t_k-\tau_{j,n}>z}$, we obtain
\#\label{eq::p13}
\sum_{k=1}^{M(t)}(t_k-t_{k-1})\left\vert\lambda^{(z)}_i(t_k)-\lambda_i(t_k)\right\vert&=\sum_{k=1}^{M(t)}\sum_{j=1}^{p}\sum_{n=1}^{N_j(t_k)}(t_k-t_{k-1}) f_{i,j}(t_k-\tau_{j,n})(1-\alpha(t_k-\tau_{j,n}))\notag\\
&=\sum_{j=1}^{p}\sum_{k:t_k\in[z,t)}\sum_{\tau_{j,n}\in[0,t_k-z)}(t_k-t_{k-1})f_{i,j}(t_k-\tau_{j,n})\notag\\
&= \sum_{j=1}^{p}\sum_{\tau_{j,n}\in[0,t-z)}\sum_{k:t_k-\tau_{j,n}\geq z}(t_k-\tau_{j,n}-t_{k-1}+\tau_{j,n})f_{i,j}(t_k-\tau_{j,n})\notag\\
&\leq\sum_{j=1}^{p}\sum_{\tau_{j,n}\in[0,t-z)}\varepsilon_{f_{i,j},\delta}(z)\leq\sum_{j=1}^{p}N_j(t-z)\varepsilon(z)\notag\\
&=N(t-z)\varepsilon(z),
\#
where $\varepsilon(t)$ is a tail function such that for any $i$ and $j$, $\varepsilon_{f_{i,j},\delta}(t)\leq \varepsilon(t)$.
The above inequality is due to Assumption \ref{asm::f}. 	
Suppose that the $m$-th arrival of the $i$-th dimension is in $[t_{k_{m}-1},t_{k_m})$. Then,
\#\label{eq::p120}
\sum_{k=1}^{M(t)}x_{i,k}\left\vert\lambda^{(z)}_i(t_k)-\lambda_i(t_k)\right\vert&=\sum_{m=1}^{N_i(t)}\left\vert\lambda^{(z)}_i(t_{k_m})-\lambda_i(t_{k_m})\right\vert\notag\\&=\sum_{m=1}^{N_i(t)}\sum_{j=1}^{p}\sum_{n=1}^{N_j(t_{k_m})}f_{i,j}(t_{k_m}-\tau_{j,n})(1-\alpha(t_{k_m}-\tau_{j,n}))\notag\\&=\sum_{j=1}^{p}\sum_{m:z<t_{k_m}<t}\sum_{\tau_{j,n}\in[0,t_{k_m}-z)}f_{i,j}(t_{k_m}-\tau_{j,n})\notag\\
&\leq \sum_{j=1}^{p}\sum_{\tau_{j,n}\in[0,t-z)}\sum_{m:t_{k_m}-\tau_{j,n}>z}f_{i,j}(t_{k_m}-\tau_{j,n})\notag\\
&\leq \sum_{j=1}^{p}\sum_{\tau_{j,n}\in[0,t-z)}\kappa_1\varepsilon(z)= N(t-z)\kappa_1\varepsilon(z).
\#
The last inequality uses Assumption \ref{asm::f} and the fact that the number of arrivals in an interval of length one is bounded by $\kappa_1$.
Therefore, by combining \eqref{eq::p12} with \eqref{eq::p13} and \eqref{eq::p120}, we get
\#\label{eq::p121}
\left\vert L_{i,t}^{(\delta)}(\lambda^{(z)}_i)-L_{i,t}^{(\delta)}(\lambda_i)\right\vert\leq \left(1+\frac{\kappa_1}{\mu_{\min}}\right)N(t-z)\varepsilon(z).
\#
We now proceed to bound the second term in \eqref{eq::p11}. By the definition, we have
\#\label{eq::p18}
\left\vert L_{i,t}^{(\delta)}(\lambda_i)-L_{i,t}(\lambda_i)\right\vert&=\left\vert\sum_{k=1}^{M(t)}(t_k-t_{k-1})\lambda_{i}(t_k)-\int_{0}^{t}\lambda_i(\tau)\ud\tau\right\vert.
\#
To bound the right-hand side, using the definition of $\lambda_i$, we have that (\ref{eq::p18}) is bounded above by
\$
&\sum_{k=1}^{M(t)}\sum_{j=1}^{p}\left[\left\vert\int_{t_{k-1}}^{t_k}\sum_{\tau_{j,n}<\tau}f_{i,j}(\tau-\tau_{j,n})-\sum_{\tau_{j,n}<t_k}f_{i,j}(t_k-\tau_{j,n})\ud\tau\right\vert\right]\notag\\
&\overset{(a)}{=} \sum_{k=1}^{M(t)}\sum_{j=1}^{p}\left[\left\vert\int_{t_{k-1}}^{t_k} \sum_{\tau_{j,n}<t_{k}}[f_{i,j}(\tau-\tau_{j,n})-f_{i,j}(t_k-\tau_{j,n}) ]\ud\tau\right\vert\right]\notag\\
&\leq\sum_{j=1}^{p}\sum_{k=1}^{M(t)}\left[\sum_{\tau_{j,n}<t_{k}} (t_k-t_{k-1})^2\sup_{x\in(t_{k-1}-\tau_{j,n},t_k-\tau_{j,n}]}\left\vert\frac{\ud f_{i,j}(x)}{\dx}\right\vert\right]\notag\\
&= \sum_{j=1}^{p}\sum_{\tau_{j,n}<t}\sum_{t_k\geq \tau_{j,n}} (t_k-t_{k-1})^2\sup_{x\in(t_{k-1}-\tau_{j,n},t_k-\tau_{j,n}]}\left\vert\frac{\ud f_{i,j}(x)}{\dx}\right\vert \\
&\leq\sum_{j=1}^{p}\sum_{\tau_{j,n}<t} \delta \varepsilon_{f'_{i,j},\delta}(\tau_{j,n}) \overset{(b)}{\leq} \delta N(t)\varepsilon'(0),
\$
where $\varepsilon'(t)$ is a tail function such that for any $i$ and $j$, $\varepsilon_{f'_{i,j},\delta}(t)\leq \varepsilon'(t)$.
In the above equations,  (a) uses the fact that in an interval $[t_{k-1}, t_k]$, arrivals can only happen at the endpoints. Moreover, (b) uses Assumption \ref{asm::f}. 
Using the upper bounds of (\ref{eq::p18}) and (\ref{eq::p121}) in \eqref{eq::p11} will imply the result.

\section{Proof of Theorem \ref{thm::main}} \label{pf::main}

We prove this regret bound following the proof technique for Theorem 4 of \citet{kivinen2004online} and the proof technique for Theorem 3.3 of \citet{hazan2016introduction}. The outline of this proof is as follows:
\begin{itemize}
	\item Firstly, we derive the following upper bound:
	\#\label{eq::main1}
\sum_{k=1}^{M(t)}\left(l_{i,k}[\lambda_i^{(z)}(\hat{\mu}_i^{(k)},\hat{\bbf}_i^{(k)})]-l_{i,k}[\lambda_i^{(z)}(\mu_i,\hat{\bbf}_i^{(k)})]\right)\leq 2\zeta^{-1}|\delta-\mu_{\min}^{-1}|^2(1+\log M(t)).
	\#
	\item Next, we derive the following upper bound:
\#\label{eq::main2}
\sum_{k=1}^{M(t)}\left(l_{i,k}[\lambda_i^{(z)}(\mu_i,\hat{\bbf}^{(k)}_i)]-l_{i,k}[\lambda_i^{(z)}(\mu_i,\bbf_i)]\right)&\leq 2p\kappa_z^2\zeta^{-1}|\delta-\mu_{\min}^{-1}|^2(1+\log M(t)).
\#
	To do this, we need three separate steps:
	\begin{itemize}
		\item Prove the Lemma \ref{lem::1}, which we state below.
		\item Prove that the instantaneous objective function is strongly convex with respect to $\hat{\bbf}_i^{(k)}$ and $\|\cdot\|_{\calH}^2$.
		\item Use the result of Lemma \ref{lem::1} and apply the proof procedure of Theorem 4 of \citet{kivinen2004online} and Theorem 3.3 of \citet{hazan2016introduction}.
	\end{itemize}
	\item Lastly, we combine the results of \eqref{eq::main1} and \eqref{eq::main2} to obtain the regret bound:
	\#\label{eq::pf::main}
	R_t^{(\delta)}[\lambda_i^{(z)}(\mu_i,\bbf_i))\leq C_1(1+\log M(t)],
	\#
	where $C_1=2(1+p\kappa_z^2)|\delta-\mu_{\min}^{-1}|^2$.
\end{itemize}

{\bfseries\noindent Step 0: Technical assumptions and lemma.}
Before the main body of the proof, we need to introduce the following technical assumption, as well as a lemma that bounds the $\calH$-norm of $\partial_{f_{i,j}} l_{i,k}$. These result will be frequently referred to throughout the main body of the proof.

\begin{assumption}\label{asm::tech} We assume that $\delta$ is set small enough such that $|\delta-\mu_{\min}^{-1}|>\delta$.	
\end{assumption}
This assumption does not affect the implementation of the algorithm since $\delta$ and $\mu_{\min}$ are both manually set.

\begin{assumption}\label{asm::initial} We assume that the initialization of the algorithm is nice enough:
\$
	\hat{\mu}_i^{(0)}\leq\omega_i^{-1}|\delta-\mu_{\min}^{-1}|,
\$
and
\$
\left\Vert\hat{f}_{i,j}^{(0)}\right\Vert_{\calH}\leq \kappa_z\zeta_{i,j}^{-1}\left|\delta-\mu^{-1}_{\min}\right|.
\$
\end{assumption}
Similar to \ref{asm::tech}, this assumption does not affect the implementation of the algorithm as we can set $\mu_{\min}$ to be small.

The following lemma is needed in Step 2, and will be proved in Step 2.
\begin{lemma}\label{lem::1} Suppose that Assumptions \ref{asm::stationary} and \ref{asm::f} hold. Then, for any $i,j,k$, the intermediate output of Algorithm \ref{alg::algorithm} at step $k$ satisfies
	\$	
	\left\|\partial_{f_{i,j}} l_{i,k}\left(\lambda^{(z)}_i(\hat{\mu}_i^{(k)},\hat{\bbf}_i^{(k)})\right)\right\|_\calH\leq\left\{\begin{array}{ll} 2\left|\delta-\mu^{-1}_{\min}\right|\kappa_z \quad & \quad \text{if}~ x_{i,k}=1\\
		2\delta \kappa_z \quad & \quad \text{if}~ x_{i,k}=0\end{array}\right.,
	\$
	and
	\$
	\left\|\hat{f}_{i,j}^{(k)}\right\|_{\calH}\leq \zeta_{i,j}^{-1}\kappa_z\left|\delta-\mu^{-1}_{\min}\right|.
	\$	
\end{lemma}

We are now ready to prove the main part of the theorem.

{\bfseries\noindent Step 1: Proving equation \eqref{eq::main1}.} We start the proof of \eqref{eq::main1} by observing the following fact: given $\hat{\bbf}_i^{(k)}$, the objective function $l_{i,k}(\lambda_i^{(z)}(\hat{\mu}_i^{(k)},\hat{\bbf}_i^{(k)}))$ is $\omega_i$-strongly convex with respect to $\hat{\mu}_i$ and the square operator. This implies that
\$
l_{i,k}[\lambda_i^{(z)}(\mu_i,\hat{\bbf}_i^{(k)})]&\geq l_{i,k}[\lambda_i^{(z)}(\hat{\mu}_i^{(k)},\hat{\bbf}_i^{(k)})]+\left\la\partial_{\mu_i}l_{i,k}[\lambda_i^{(z)}(\hat{\mu}_i^{(k)},\hat{\bbf}_i^{(k)})],\mu_i-\hat{\mu}_i^{(k)}\right\ra+\frac{\omega_i}{2}(\mu_i-\hat{\mu}_i^{(k)})^2,
\$
which further indicates that
\#\label{eq::main6}
2l_{i,k}[\lambda_i^{(z)}(\hat{\mu}_i^{(k)},\hat{\bbf}_i^{(k)})]-2l_{i,k}[\lambda_i^{(z)}(\mu_i,\hat{\bbf}_i^{(k)})]&\leq 2\left\la\partial_{\mu_i}l_{i,k}[\lambda_i^{(z)}(\hat{\mu}_i^{(k)},\hat{\bbf}_i^{(k)})],\hat{\mu}_i^{(k)}-\mu_i\right\ra-\omega_i(\mu_i-\hat{\mu}_i^{(k)})^2.
\#
By the update rule, we have
\$
\hat{\mu}_i^{(k+1)}=\Pi\left[\hat{\mu}_i^{(k)}-\eta_k\partial_{\mu_i}l_{i,k}[\lambda_i^{(z)}(\hat{\mu}_i^{(k)},\hat{\bbf}_i^{(k)})]\right].
\$
Since the projection is contractive, we have
\$
\left(\hat{\mu}_i^{(k+1)}-\mu_i\right)^2\leq\left(\hat{\mu}_i^{(k+\frac{1}{2})}-\mu_i\right)^2&=\left(\hat{\mu}_i^{(k)}-\eta_k\partial_{\mu_i}l_{i,k}[\lambda_i^{(z)}(\hat{\mu}_i^{(k)},\hat{\bbf}_i^{(k)})]-\mu_i\right)^2\notag\\&=\left(\hat{\mu}_i^{(k)}-\mu_i\right)^2-2\eta_k\left\la\hat{\mu}_i^{(k)}-\mu_i,l_{i,k}[\lambda_i^{(z)}(\hat{\mu}_i^{(k)},\hat{\bbf}_i^{(k)})]\right\ra+\\& +\eta_k^2\left[\partial_{\mu_i}l_{i,k}[\lambda_i^{(z)}(\hat{\mu}_i^{(k)},\hat{\bbf}_i^{(k)})]\right]^2.
\$
Hence,
\#\label{eq::main3}
2\left\la\hat{\mu}_i^{(k)}-\mu_i,l_{i,k}[\lambda_i^{(z)}(\hat{\mu}_i^{(k)},\hat{\bbf}_i^{(k)})]\right\ra&\leq \frac{1}{\eta_k}\left[\left(\hat{\mu}_i^{(k)}-\mu_i\right)^2-\left(\hat{\mu}_i^{(k+1)}-\mu_i\right)^2\right]+\notag\\&+\eta_k\left[\partial_{\mu_i}l_{i,k}[\lambda_i^{(z)}(\hat{\mu}_i^{(k)},\hat{\bbf}_i^{(k)})]\right]^2.
\#

It is not hard to verify that $|\partial_{\mu_i}l_{i,k}[\lambda_i^{(z)}(\hat{\mu}_i^{(k)},\hat{\bbf}_i^{(k)})]|$ is bounded: first, notice that
\#\label{eq::main4}
\left|\partial_{\mu_i}l_{i,k}[\lambda_i^{(z)}(\hat{\mu}_i^{(k)},\hat{\bbf}_i^{(k)})]\right|&=\left|\partial_{\mu_i}(t_k-t_{k-1})\lambda_i(t_k)-x_{i,k}\log\lambda_i(t_k)+\frac{\omega_i}{2}[\hat{\mu}_i^{(k)}]^2\right|\notag\\&=\rho_k+\omega_i\hat{\mu}_i^{(k)}\leq |\delta-\mu_{\min}^{-1}|+\omega_i\hat{\mu}_i^{(k)},
\#
where the last step uses the result $\rho_k\leq|\delta-\mu_{\min}^{-1}|$, which is a direct consequence from Assumption \ref{asm::tech}. By the update rule of $\hat{\mu}_i^{(k)}$, we can see that if $\hat{\mu}_i^{(k)}\leq \omega_{i}^{-1}|\delta-\mu_{\min}^{-1}|$, then
\$
\hat{\mu}_i^{(k+1)}&\leq \hat{\mu}_i^{(k)}(1-\omega_i\eta_k)+\eta_k|\delta-\mu_{\min}^{-1}|\leq(1-\omega_i\eta_k)\omega_i^{-1}|\delta-\mu_{\min}^{-1}|+\eta_k|\delta-\mu_{\min}^{-1}|\\&=\omega_i^{-1}|\delta-\mu_{\min}^{-1}|.
\$
Therefore by Assumption \ref{asm::initial} and mathematical induction, $\hat{\mu}_i^{(k)}\leq\omega_i^{-1}|\delta-\mu_{\min}^{-1}|$ for every $k\geq 0$. Combining this result with \eqref{eq::main4}, we have
\#\label{eq::main5}
\left|\partial_{\mu_i}l_{i,k}[\lambda_i^{(z)}(\hat{\mu}_i^{(k)},\hat{\bbf}_i^{(k)})]\right|\leq 2|\delta-\mu_{\min}^{-1}|.
\#

With \eqref{eq::main3}, \eqref{eq::main5} and \eqref{eq::main6}, we have
\$
&2\sum_{k=1}^{M(t)}\left(l_{i,k}[\lambda_i^{(z)}(\hat{\mu}_i^{(k)},\hat{\bbf}_i^{(k)})]-l_{i,k}[\lambda_i^{(z)}(\mu_i,\hat{\bbf}_i^{(k)})]\right)\\&\leq -\omega_i\sum_{k=1}^{M(t)}\left(\mu_i-\hat{\mu}_i^{(k)}\right)^2+\sum_{k=1}^{M(t)}\frac{1}{\eta_k}\left[\left(\hat{\mu}_i^{(k)}-\mu_i\right)^2-\left(\hat{\mu}_i^{(k+1)}-\mu_i\right)^2\right]+\sum_{k=1}^{M(t)}\eta_k\left[\partial_{\mu_i}l_{i,k}[\lambda_i^{(z)}(\hat{\mu}_i^{(k)},\hat{\bbf}_i^{(k)})]\right]^2\\&=\sum_{k=1}^{M(t)}\left[\frac{1}{\eta_k}-\frac{1}{\eta_{k-1}}-\omega_i\right]\left(\mu_i-\hat{\mu}_i^{(k)}\right)^2+4|\delta-\mu_{\min}^{-1}|^2\sum_{k=1}^{M(t)}\eta_k\\&\leq 4|\delta-\mu_{\min}^{-1}|^2\sum_{k=1}^{M(t)}\eta_k,
\$
where in the last step, we have invoked the assumption that when $\omega_i\geq\zeta$, and $\eta_k=1/(k\zeta+b)$ for $k>0$ \footnote{we assume $1/\eta_0=0$ since $\hat{\mu}_i^{(0)}$ was not involved in the summation.}, $\eta_k^{-1}-\eta_{k-1}^{-1}\leq\omega_i$. Furthermore,
$
\sum_{k=1}^{M(t)}\eta_k\leq \zeta^{-1}(1+\log M(t)).
$
Hence, plugging this result into the previous equation, we have
\$
\sum_{k=1}^{M(t)}\left(l_{i,k}[\lambda_i^{(z)}(\hat{\mu}_i^{(k)},\hat{\bbf}_i^{(k)})]-l_{i,k}[\lambda_i^{(z)}(\mu_i,\hat{\bbf}_i^{(k)})]\right)\leq 2\zeta^{-1}|\delta-\mu_{\min}^{-1}|^2(1+\log M(t)),
\$
completing the proof of Step 1.

{\bfseries\noindent Step 2: Proving equation \eqref{eq::main2}.} The proof of \eqref{eq::main2} follows the same procedure as the proof of \eqref{eq::main1}. However, proving the counterpart of \eqref{eq::main5} is more complicated. We stated it in Lemma \ref{lem::1}, and we now formally prove it.

{\bfseries\noindent Step 2.1: Proof of Lemma \ref{lem::1}.}

Recall from equation \eqref{eq::partial} that, at the $k$-th update epoch, the update rule for $\hat{f}_{i,j}^{(k)}$ can be written as
\$
\hat{f}_{i,j}^{(k+1)}&= -\eta_k\left[(t_{k}-t_{k-1})-\frac{x_{i,k}}{\lambda^{(z)}_i\left(\hat{\mu}_i^{(k)},\hat{\bbf}_i^{(k)}\right)}\right]\sum_{\tau_{j,n}\in[t_k-z,t_k)}K(t_k-\tau_{j,n},\cdot)\notag\\
&+(1-\eta_k\zeta_{i,j})\hat{f}_{i,j}^{(k)},
\$
where, by Assumption \ref{asm::f}, $K(x,x)\leq 1$ for all $x\in\reals$. Since we have used the truncated intensity function $\lambda_i^{(z)}$, we have, by triangle inequality,
\$
\big\Vert\sum_{\tau_{j,n}\in[t_k-z,t_k)}K(t_k-\tau_{j,n},\cdot)\big\Vert_{\calH}^2&\leq  \big[\sum_{\tau_{j,n}\in[t_k-z,t_k)}\|K(t_k-\tau_{j,n},\cdot)\|_{\calH}\big]^2=\big[\sum_{\tau_{j,n}\in[t_k-z,<t_k)}K(t_k-\tau_{j,n},t_k-\tau_{j,n})\big]^2\leq \kappa_z^2,
\$	
where $z$ is the window size that is selected at the beginning of the algorithm.
Here, we have used the assumption that the number of arrivals within $[t_k-z,t_k)$ is upper bounded by $\kappa_z$, by Assumption \ref{asm::stationary}.	
In addition, by the design of the algorithm, we always have
$
\lambda^{(z)}_i\left(\hat{\mu}_i^{(k)},\hat{\bbf}_i^{(k)}\right)\geq\mu_{\min}.
$
Hence, when $x_{i,k}=1$,
\#\label{eq::l11}
\left\|\hat{f}_{i,j}^{(k+1)}\right\|_{\calH}&\leq (1-\eta_k\zeta_{i,j})\left\|\hat{f}_{i,j}^{(k)}\right\|_{\calH}+\left|\eta_k\left(t_k-t_{k-1}-\frac{x_{i,k}}{\lambda^{(z)}_i\left(\hat{\mu}_i^{(k)},\hat{\bbf}_i^{(k)}\right)}\right)\right|\kappa_z\notag\\
&\leq (1-\eta_k\zeta_{i,j})\left\|\hat{f}_{i,j}^{(k)}\right\|_{\calH}+\eta_k\left|\delta-\mu^{-1}_{\min}\right|\kappa_z ,
\#
where in the last step of \eqref{eq::l11}, we have used the technical assumption \ref{asm::tech}.
When the algorithm initializes with $\hat{f}_{i,j}^{(0)}$ satisfies Assumption \ref{asm::initial}, \ie,
\$
\left\Vert\hat{f}_{i,j}^{(0)}\right\Vert_{\calH}\leq \kappa_z\zeta_{i,j}^{-1}\left|\delta-\mu^{-1}_{\min}\right|,
\$
we can use induction and \eqref{eq::l11} to show that every $\hat{f}_{i,j}^{(k)}$ satisfies the above bound.
In addition, by \eqref{eq::partial}, we have
\$
\left\|\partial_{f_{i,j}}l_{i,k}\left(\hat{f}_{i,j}^{(k)}\right)\right\|_\calH&\leq 2\kappa_z\left|\delta-\mu^{-1}_{\min}\right|.
\$

Similarly, when $x_{i,k}=0$, the term $\mu_{\min}^{-1}$ vanishes because $x_{i,k}=0$, and hence we reach the desired statement.

{\bfseries\noindent Step 2.2: Strong convexity of the objective function.} The instantaneous objective function is strongly convex in the following sense:
\$
l_{i,k}[\lambda_i^{(z)}(\mu_i,\bbf_i)]&\geq l_{i,k}[\lambda_i^{(z)}(\mu_i,\hat{\bbf}^{(k)}_i)]+\sum_{j=1}^{p}\left\la\partial_{f_{i,j}}l_{i,k}(\lambda_i^{(z)}(\mu_i,\hat{\bbf}_i^{(k)})),f_{i,j}-\hat{f}_{i,j}^{(k)}\right\ra+\sum_{j=1}^{p}\frac{\zeta_{i,j}}{2}\left\|f_{i,j}-\hat{f}_{i,j}^{(k)}\right\|_{\calH}^2.
\$
In particular, the instantaneous objective function is strongly convex with respect to any one of the $f_{i,j}(t)$s and $\|\cdot\|_{\calH}^2$ when the remaining $p-1$ are fixed.
The proof follows directly from the strong convexity of $\|\cdot\|_{\calH}^2$.

{\bfseries\noindent Step 2.3: Proof of \eqref{eq::main2}.} We now prove \eqref{eq::main2}.  By the strong convexity of the instantaneous objective function proved in Step 2.2, we have
\#
2l_{i,k}[\lambda_i^{(z)}(\mu_i,\bbf_i)]&\geq 2l_{i,k}[\lambda_i^{(z)}(\mu_i,\hat{\bbf}^{(k)}_i)]+ 2\sum_{j=1}^{p}\left\la\partial_{f_{i,j}}l_{i,k}(\hat{f}^{(k)}_{i,j}),f_{i,j}-\hat{f}_{i,j}^{(k)}\right\ra_{\calH}+\sum_{j=1}^{p}\zeta_{i,j}\left\Vert f_{i,j}-\hat{f}_{i,j}^{(k)}\right\Vert_\calH^2.
\#
This can be written as follows
\#\label{eq::t13}
2l_{i,k}[\lambda_i^{(z)}(\mu_i,\hat{\bbf}^{(k)}_i)]- 2l_{i,k}[\lambda_i^{(z)}(\mu_i,\bbf_i)]&\leq 2\sum_{j=1}^{p}\left\la\partial_{f_{i,j}}l_{i,k}(\hat{f}^{(k)}_{i,j}),\hat{f}_{i,j}^{(k)}-f_{i,j}\right\ra-\sum_{j=1}^{p}\zeta_{i,j}\left\Vert f_{i,j}-\hat{f}_{i,j}^{(k)}\right\Vert_\calH^2.
\#
For any $j\in\{1,\ldots,p\}$, since $\hat{f}_{i,j}^{(k+1)}=\Pi[\hat{f}_{i,j}^{(k)}-\eta_k\partial_{f_{i,j}}l_{i,k}[\lambda_i^{(z)}(\mu_i,\hat{\bbf}^{(k)}_i)]]$ and $\Pi[\cdot]$ is contractive, we have
\$
\left\Vert\hat{f}_{i,j}^{(k+1)}-f_{i,j}\right\Vert_\calH^2&\leq\left\Vert\hat{f}_{i,j}^{(k)}-f_{i,j}-\eta_k\partial_{f_{i,j}}l_{i,k}[\lambda_i^{(z)}(\mu_i,\hat{\bbf}^{(k)}_i)]\right\Vert_\calH^2\\
&=\left\Vert\hat{f}_{i,j}^{(k)}-f_{i,j}\right\Vert_\calH^2+\eta_k^2\left\Vert\partial_{f_{i,j}}l_{i,k}[\lambda_i^{(z)}(\mu_i,\hat{\bbf}^{(k)}_i)]\right\Vert_\calH^2-2\eta_k\left\la\partial_{f_{i,j}}l_{i,k}[\lambda_i^{(z)}(\mu_i,\hat{\bbf}^{(k)}_i)],\hat{f}_{i,j}^{(k)}-f_{i,j}\right\ra_\calH.
\$
Therefore,
\#\label{eq::t11}
2\left\la\partial_{f_{i,j}}l_{i,k}[\lambda_i^{(z)}(\mu_i,\hat{\bbf}^{(k)}_i)],\hat{f}_{i,j}^{(k)}-f_{i,j}\right\ra_\calH&\leq\frac{1}{\eta_k}\left[\left\Vert\hat{f}_{i,j}^{(k)}-f_{i,j}\right\Vert_\calH^2-\left\Vert\hat{f}_{i,j}^{(k+1)}-f_{i,j}\right\Vert_\calH^2\right]+\eta_k\left\Vert\partial_{f_{i,j}}l_{i,k}[\lambda_i^{(z)}(\mu_i,\hat{\bbf}^{(k)}_i)]\right\Vert_\calH^2.
\#
Using Lemma \ref{lem::1}, we have
\$
\eta_k\left\Vert\partial_{f_{i,j}}l_{i,k}[\lambda_i^{(z)}(\mu_i,\hat{\bbf}^{(k)}_i)]\right\Vert_\calH^2\leq 4\eta_k \kappa_z^2|\delta-\mu_{\min}^{-1}|^2
\$
when $x_{i,k}=1$, and
\$
\eta_k\left\Vert\partial_{f_{i,j}}l_{i,k}[\lambda_i^{(z)}(\mu_i,\hat{\bbf}^{(k)}_i)]\right\Vert_\calH^2\leq 4\eta_k \kappa_z^2\delta^2\leq 4\eta_k\kappa_z^2|\delta-\mu_{\min}^{-1}|^2
\$
when $x_{i,k}=0$. 

We now proceed to final step, which sums \eqref{eq::t13} over $k\in\{1,\ldots,M(t)\}$ and then combines the result with \eqref{eq::t11} summed over $j\in\{1,\ldots,p\}$. To obtain stronger intuition, we choose to use the uniform upper bound $4\eta_k\kappa_z^2|\delta-\mu_{\min}^{-1}|^2$ for $\eta_k\|\partial_{f_{i,j}}\|l_{i,k}[\lambda_i^{(z)}(\mu_i,\hat{\bbf}_i^{(k)})]$, which holds for all values of $x_{i,k}$. We thus obtain
\#\label{ew:1}\notag
2\sum_{k=1}^{M(t)}\left(l_{i,k}[\lambda_i^{(z)}(\mu_i,\hat{\bbf}^{(k)}_i)]-l_{i,k}[\lambda_i^{(z)}(\mu_i,\bbf_i)]\right)\leq
&\sum_{k=1}^{M(t)}\sum_{j=1}^{p}\left\Vert\hat{f}_{i,j}^{(k)}-f_{i,j}\right\Vert_\calH^2\left(\frac{1}{\eta_k}-\frac{1}{\eta_{k-1}}-\zeta_{i,j}\right)+\\&+4\kappa_z^2|\delta-\mu_{\min}|^2\sum_{k=1}^{M(t)}\sum_{j=1}^{p}\eta_k.
\#

Since $\eta_k=1/(k\zeta+b)$, we obtain
\$
\sum_{k=1}^{M(t)}\eta_k\leq \zeta^{-1}(1+\log M(t)).
\$

Furthermore, $1/\eta_k-1/\eta_{k-1}-\zeta_{i,j}\leq 0$. Therefore, substituting the above inequalities into \eqref{ew:1}, we get
\$
\sum_{k=1}^{M(t)}\left(l_{i,k}[\lambda_i^{(z)}(\mu_i,\hat{\bbf}^{(k)}_i)]-l_{i,k}[\lambda_i^{(z)}(\mu_i,\bbf_i)]\right)&\leq 2p\zeta^{-1}\kappa_z^2|\delta-\mu_{\min}^{-1}|^2(1+\log M(t)).
\$
{\bfseries\noindent Step 3.} The overall regret bound can be obtained by adding \eqref{eq::main1} and \eqref{eq::main2}.


\subsection{Proof of Corollary \ref{cor::main}}\label{cor::mainp}

From the result of Proposition \ref{prop::approx}, we have 
\$
\left|L_{i,t}^{(\delta)}(\lambda^{(z)}_i)-L_{i,t}(\lambda_i)\right|\leq (1+\frac{\kappa_1}{\mu_{\min}})N(t-z)\varepsilon(z)+\delta N(t)\varepsilon'(0).
\$	
Using the above inequality, the results of Theorem \ref{thm::main}, and the triangle inequality, we obtain
\$
\sum_{k=1}^{M(t)}\left(l_{i,k}(\lambda^{(z)}_i(\hat{\mu}_i,\hat{\bbf}_i^{(k)}))-l_{i,k}(\lambda_i({\mu}_i,{\bbf}_i))\right) \leq (C_1+C_2)(1+\log M(t))+C_3 N(t),
\$
where $C_1=(1+\zeta)^{-2}|\delta-\mu_{\min}^{-1}|^2+2\kappa_z^2\delta^2p$ and $C_2=2\kappa_z^2\mu_{\min}^{-2}-4\kappa_z^2\delta\mu_{\min}^{-1}$.

\section{Proof of Proposition \ref{prop::sosproj}}\label{pf::sosproj}

Generally speaking, the projection operation is a QP problem:
\#\label{eq:weq}
\hat{f}_{i,j}^{(k+1)}=\argmin_{f\in\mathcal{H},f(t)\geq0}|| \hat{f}_{i,j}^{(k+\frac{1}{2})}-f||^2_\mathcal{H}=\argmin_{f\in\mathcal{H},f(t)\geq0}-2\la f,\hat{f}_{i,j}^{(k+\frac{1}{2})}\ra_{\mathcal{H}}+||f||^2_\mathcal{H}.
\#
Recall $\hat{f}_{i,j}^{(k+\frac{1}{2})}(\cdot)=\sum_{s\in\mathcal{S}} a_s K(s,\cdot)$, where 
$
\mathcal{S}=\cup_{r\leq k}\{t_r-\tau_{j,n}:\ t_r-z\leq\tau_{j,n}<t_r\}.
$
Hence, \eqref{eq:weq} can be written as
\#\label{eq:weq2}
\argmin_{f\in\mathcal{H},f(t)\geq0}-2\sum_{s\in\mathcal{S}} a_s f(s)+||f||^2_\mathcal{H}.
\#
Let $\mathcal{H}$ be the RKHS with kernel $K(x,y)=(1+xy)^{2d}$, for some integer $d$. By the solution of Hilbert's 17th problem \citep{bochnak2013real}, we know that a 1-dimensional and $2d$-degree polynomial is nonnegative iff it can be written as the sum of squares of $d$-degree polynomials, \ie, a quadratic form of $d$-degree polynomials.  This allows us to substitute the constraint in (\ref{eq:weq2}) with $f\in\{\phi^\top(x)\Qb\phi(x):\ \Qb\succeq0\}\subset\mathcal{H}$, where $\phi(x)$ is the feature map of the kernel function $K'(x,y)=(1+xy)^d$, \ie, $\phi^\top(x)\phi(y)=K'(x,y)$.

Finally, by the representer Theorem in \citet{bagnell2015learning} for positive functions, we obtain that the minimizer to (\ref{eq:weq2}) is of the form $\sum_{s\in\mathcal{S}} b_s K(s,\cdot)$. Hence,  (\ref{eq:weq2}) can be written as follows,
\#
\arg\min_{\bbb}& \ -2\ab^\top\Kb\bbb+\bbb^\top\Kb\bbb\\ \notag
\text{s.t.}  \ \ f(x)=\sum_{s\in\mathcal{S}} b_s & K(s,x)=\phi^\top(x)\Qb\phi(x),\ \text{for some}\ \Qb\succeq0.
\#
Upon simple manipulations, the above problem can be rewritten as
\#
\arg\min_{\bbb}& \ -2\ab^{\top}\Kb\bbb+\ab^\top\Kb\bbb\ \ \ \ \ \quad \ \\ \notag
\text{s.t.} \ \ \Gb\cdot\diag&(\bbb)+\diag(\bbb)\cdot\Gb\succeq0,\ \ \ 
\#
where $\Kb=[K(s,s')]$, and $\Gb=[\phi^\top(s)\phi(s')]=[K'(s,s')]$.

\section{Experiment Details} \label{app::expdet}


\begin{figure}
\centering
	\begin{tabular}{ccccc}
		\includegraphics[width=0.16\linewidth]{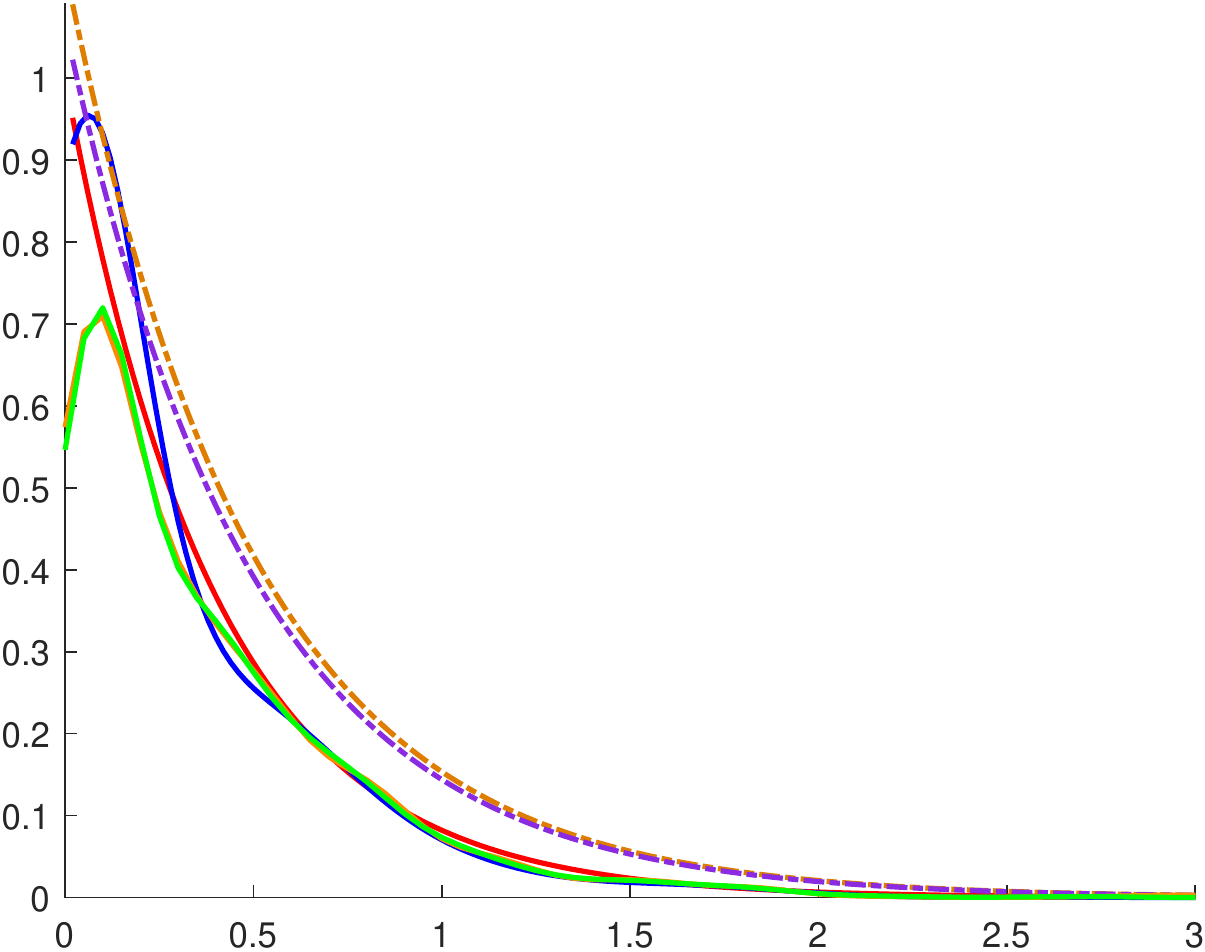} & \includegraphics[width=0.16\linewidth]{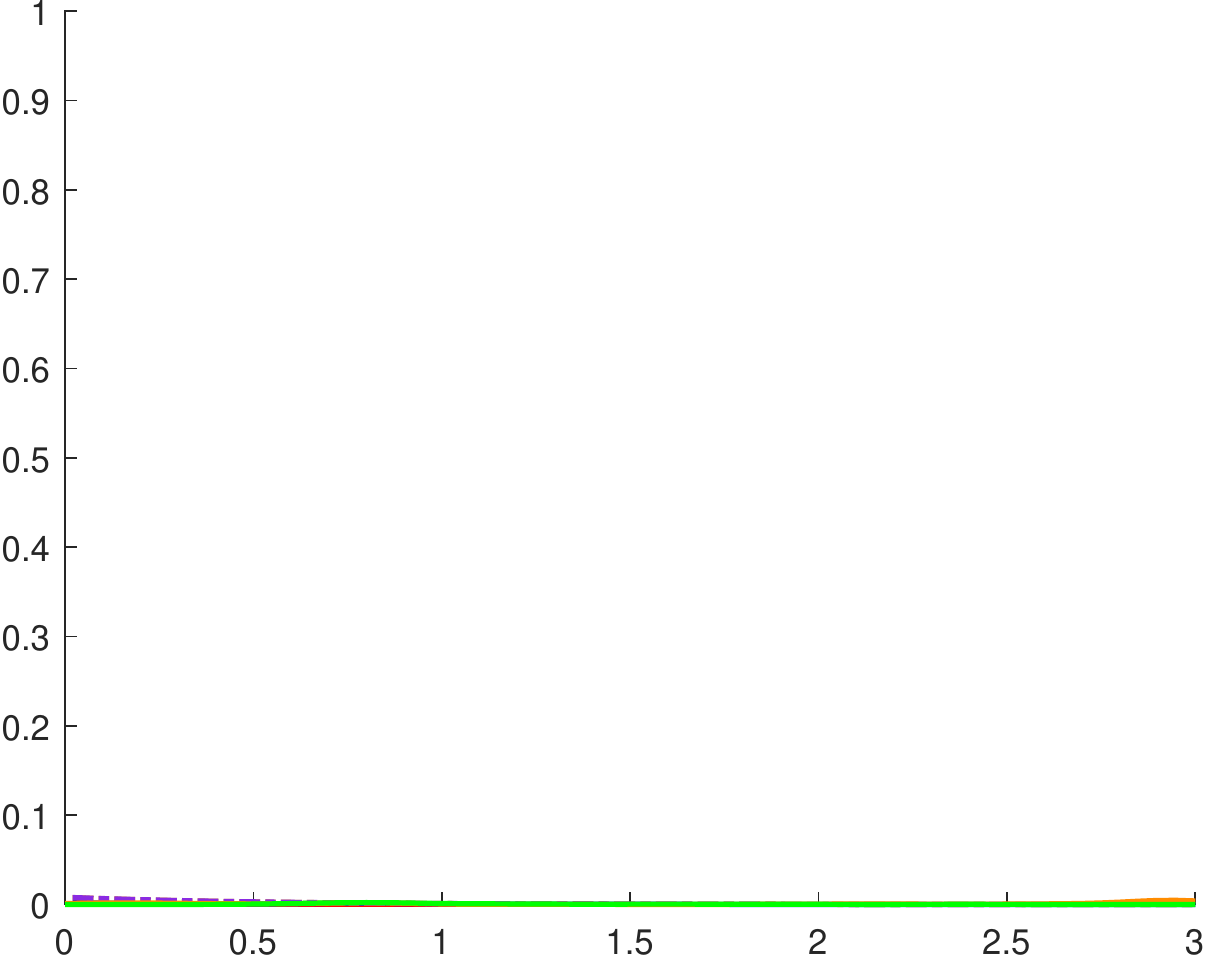} & \includegraphics[width=0.16\linewidth]{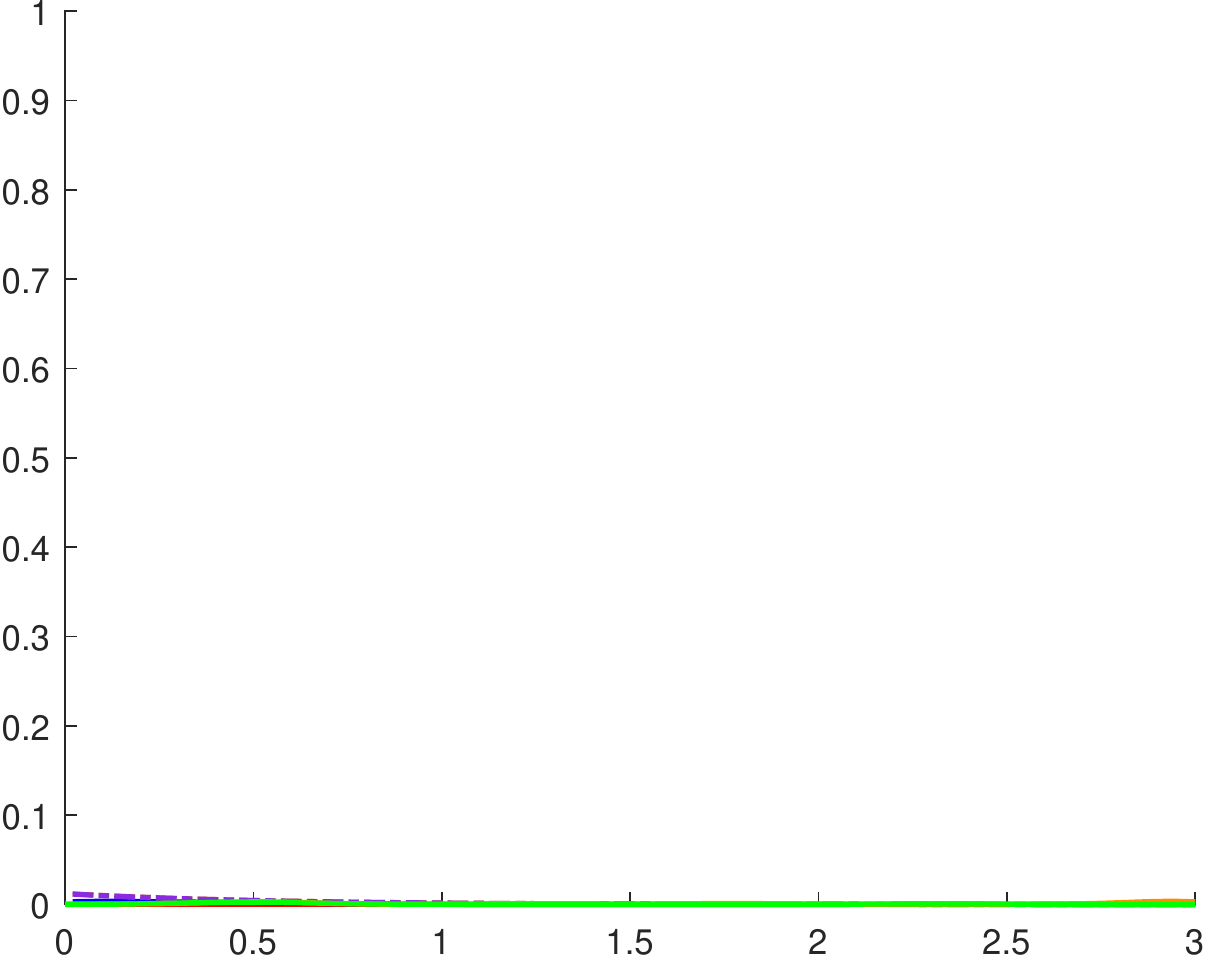} & \includegraphics[width=0.16\linewidth]{14-eps-converted-to.pdf} & \includegraphics[width=0.16\linewidth]{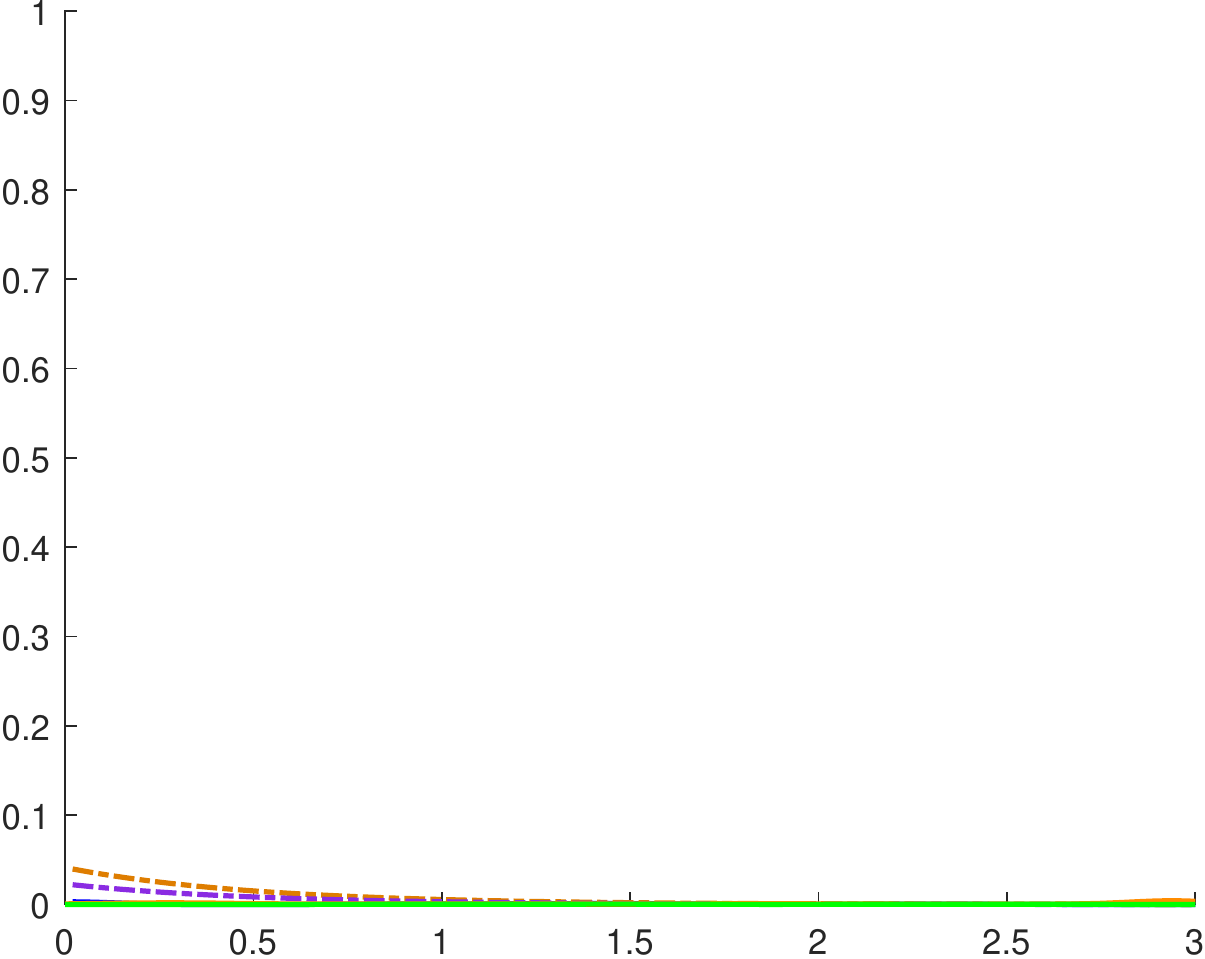} \\
		\includegraphics[width=0.16\linewidth]{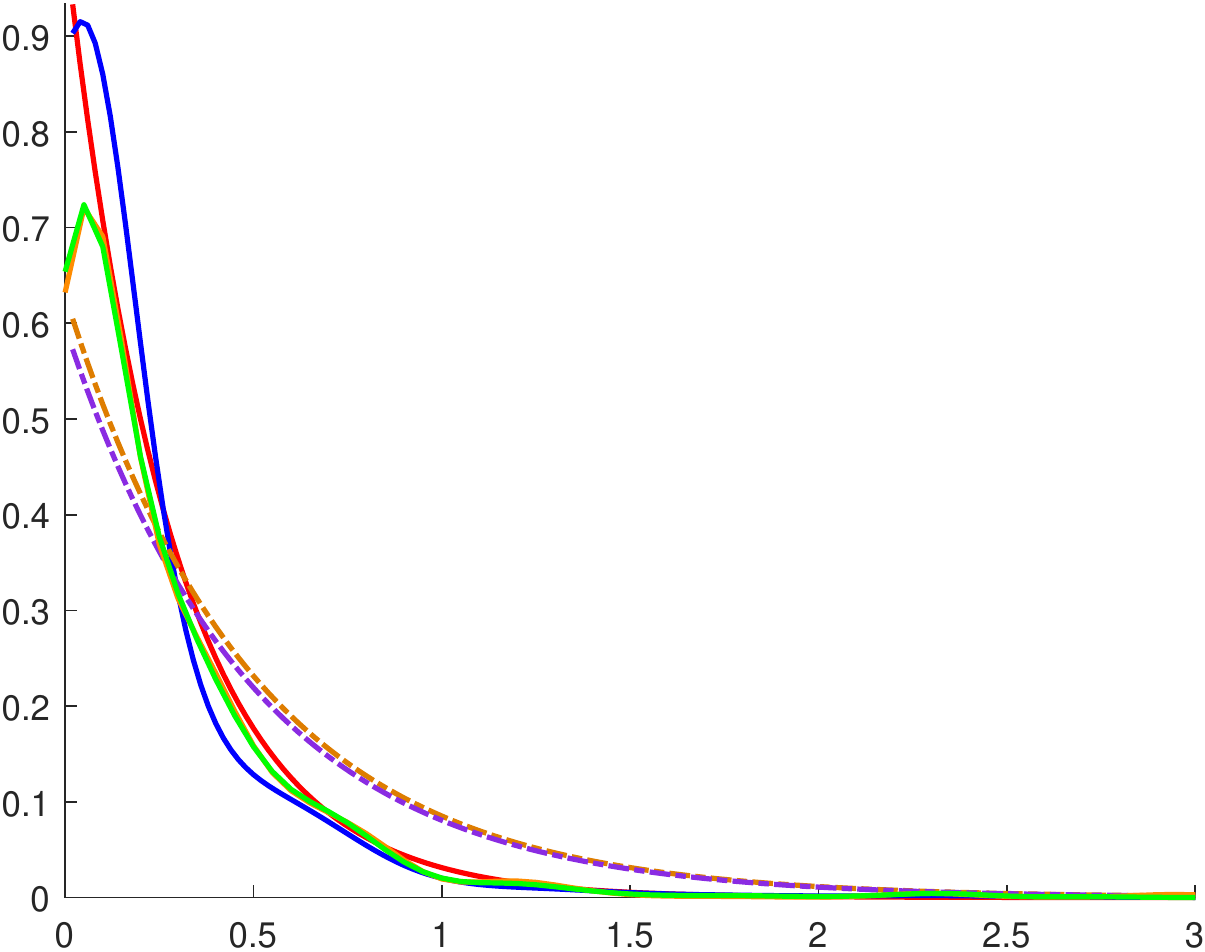} & \includegraphics[width=0.16\linewidth]{22-eps-converted-to.pdf} & \includegraphics[width=0.16\linewidth]{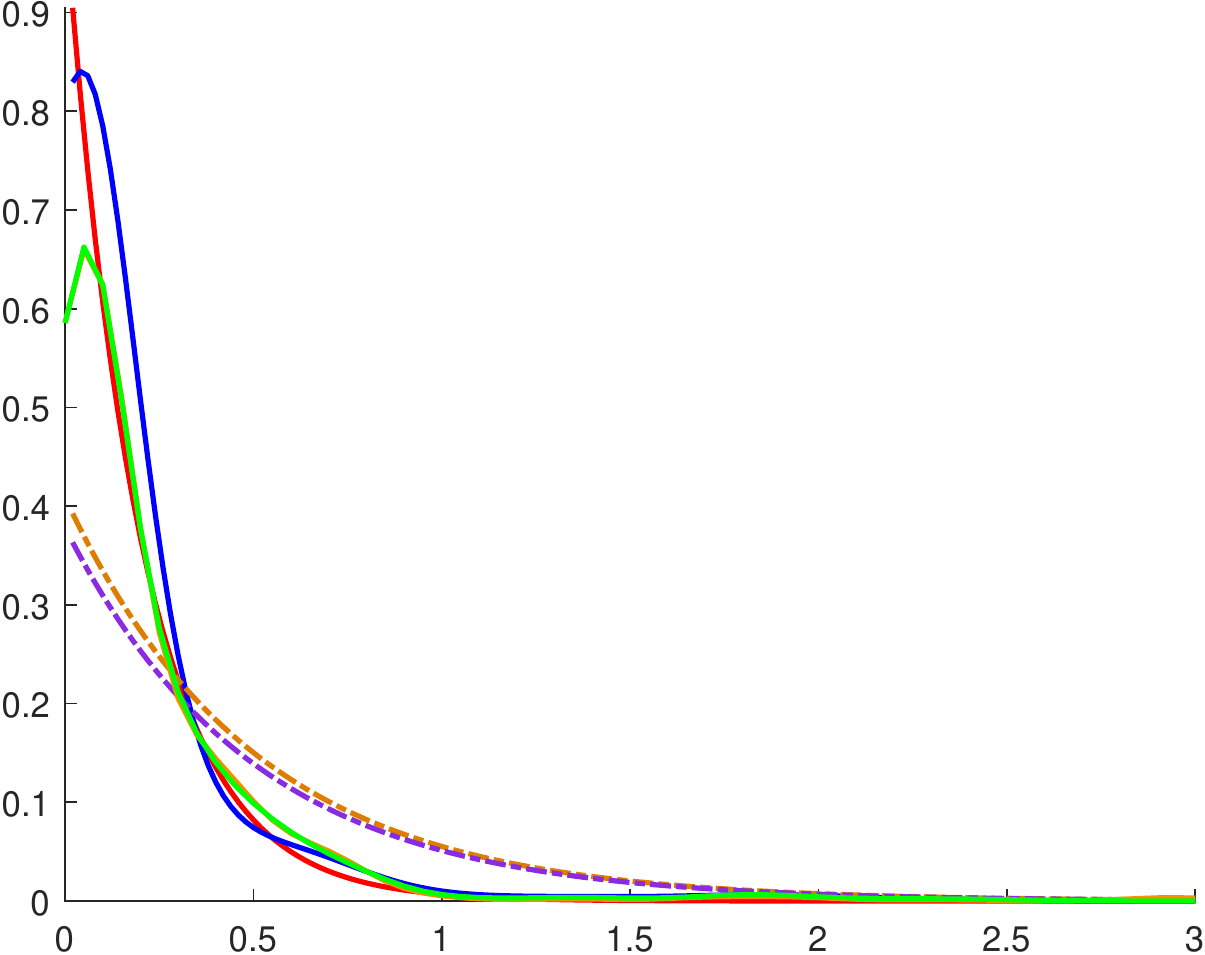} & \includegraphics[width=0.16\linewidth]{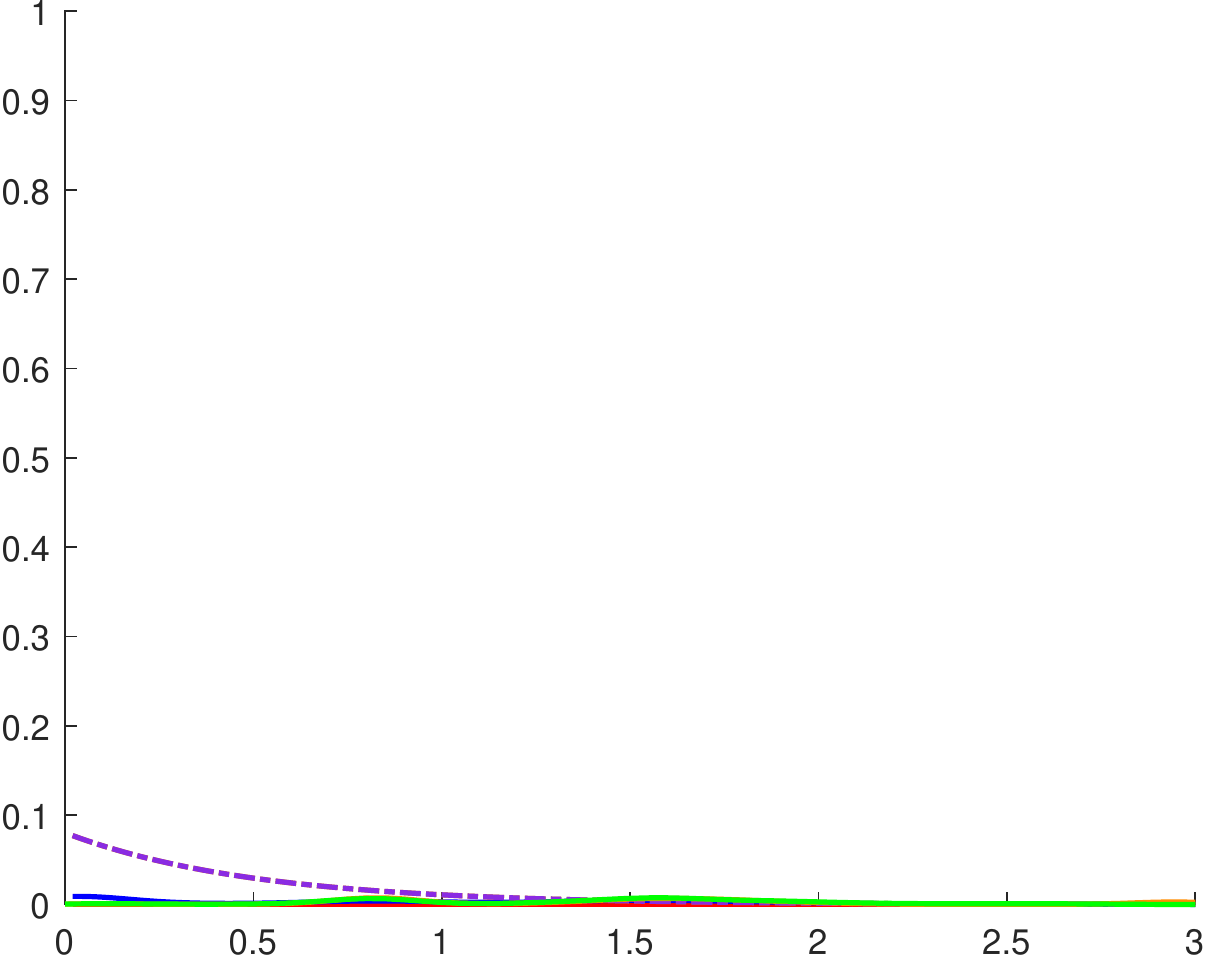} & \includegraphics[width=0.16\linewidth]{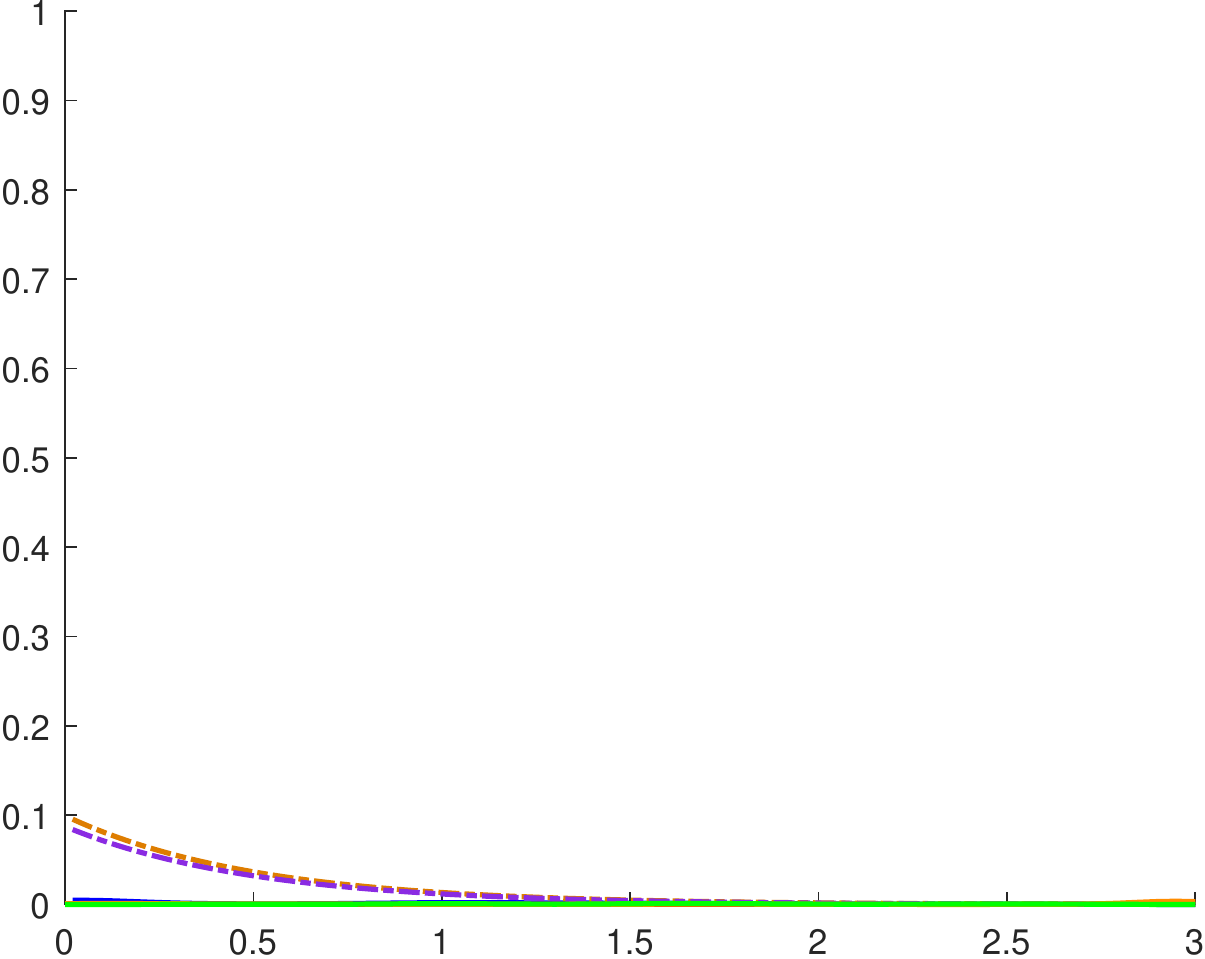} \\
		\includegraphics[width=0.16\linewidth]{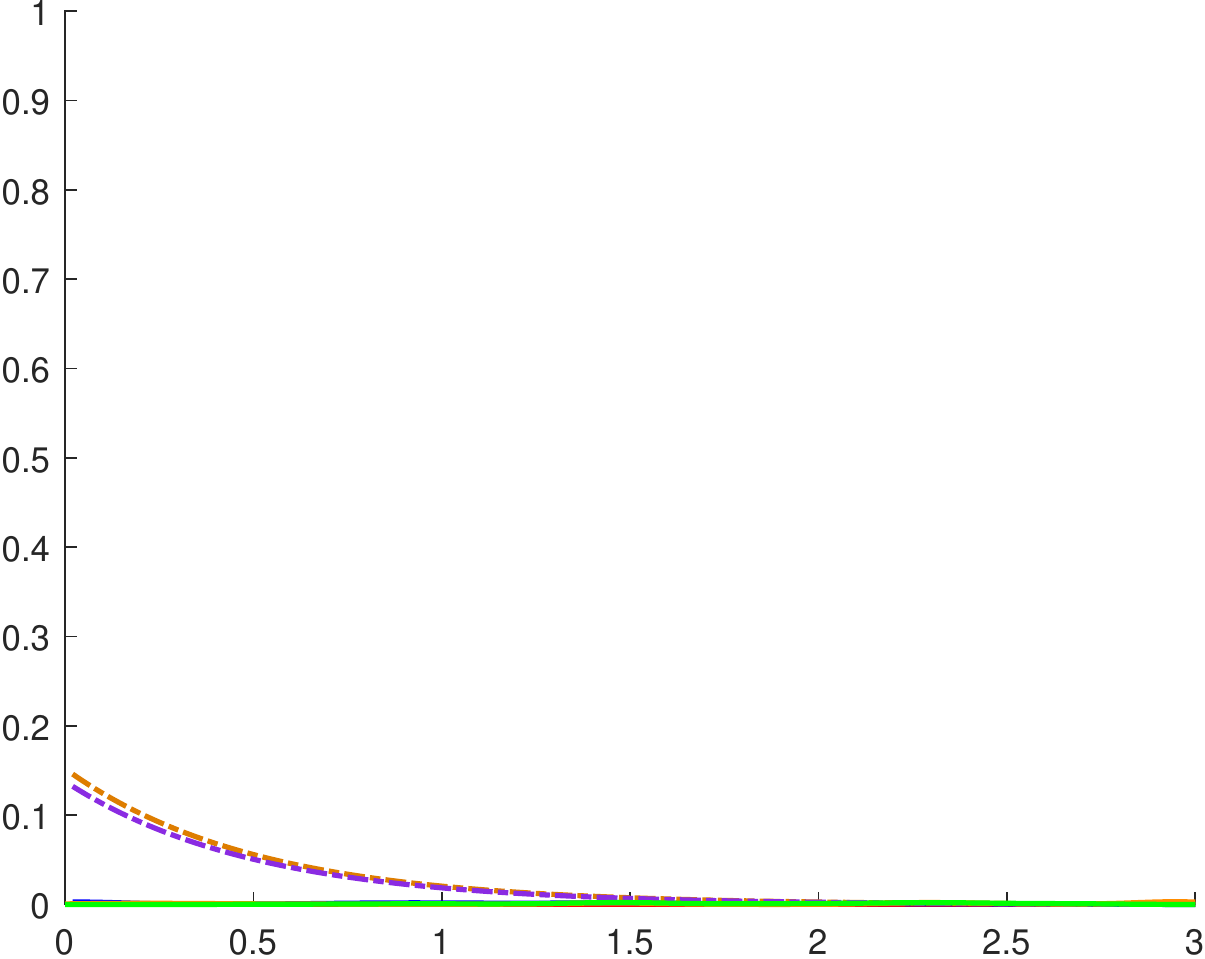} & \includegraphics[width=0.16\linewidth]{32-eps-converted-to.pdf} & \includegraphics[width=0.16\linewidth]{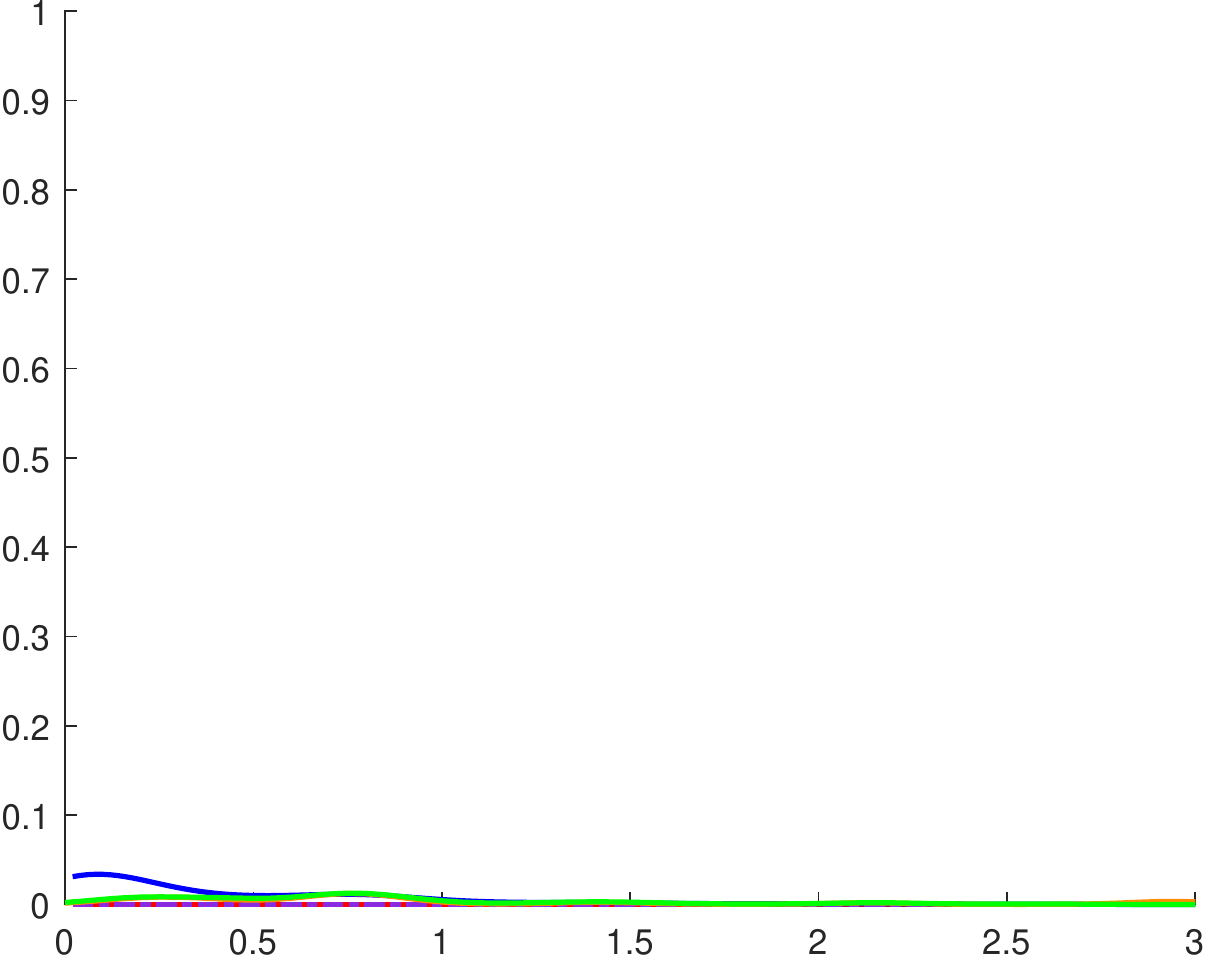} & \includegraphics[width=0.16\linewidth]{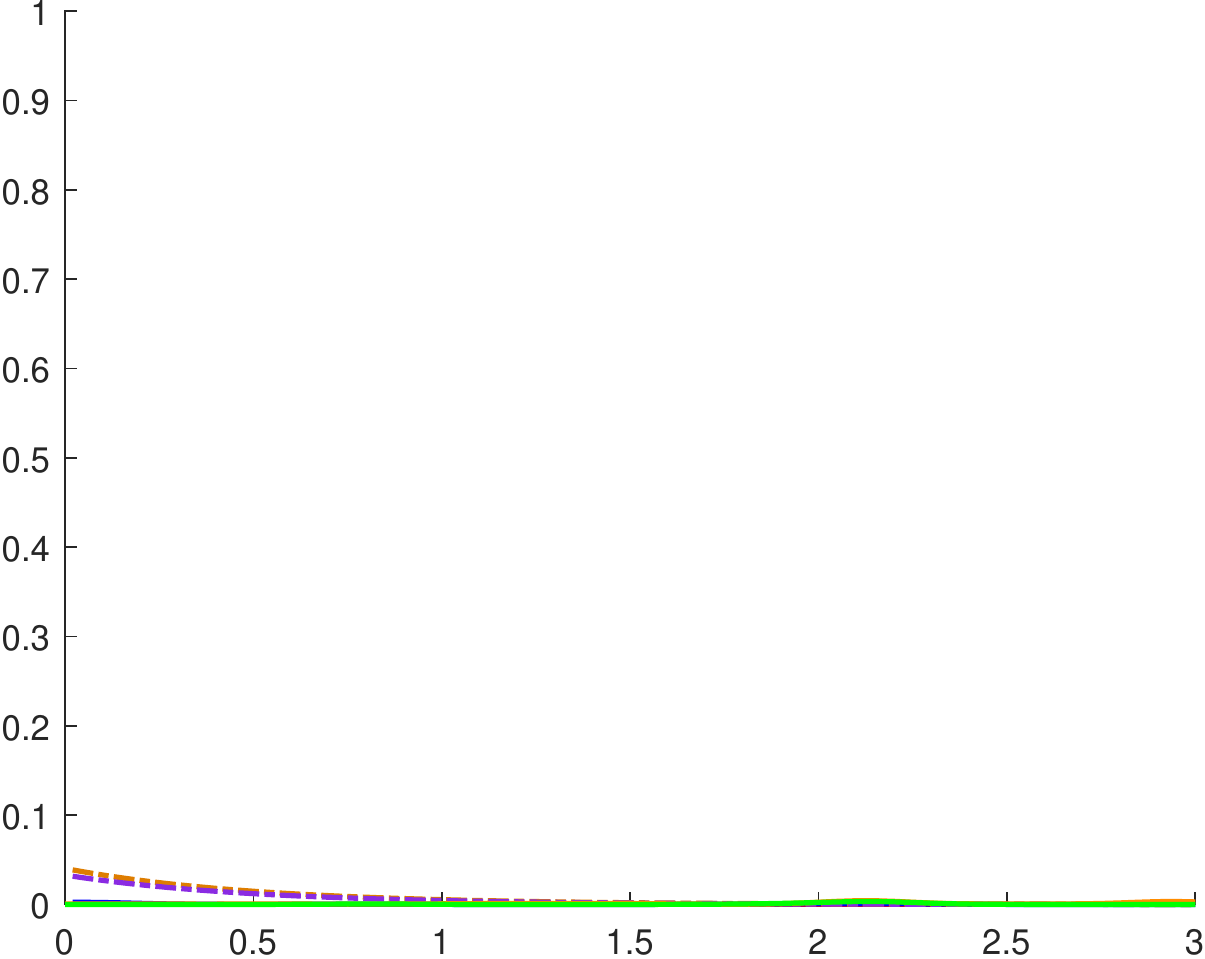} & \includegraphics[width=0.16\linewidth]{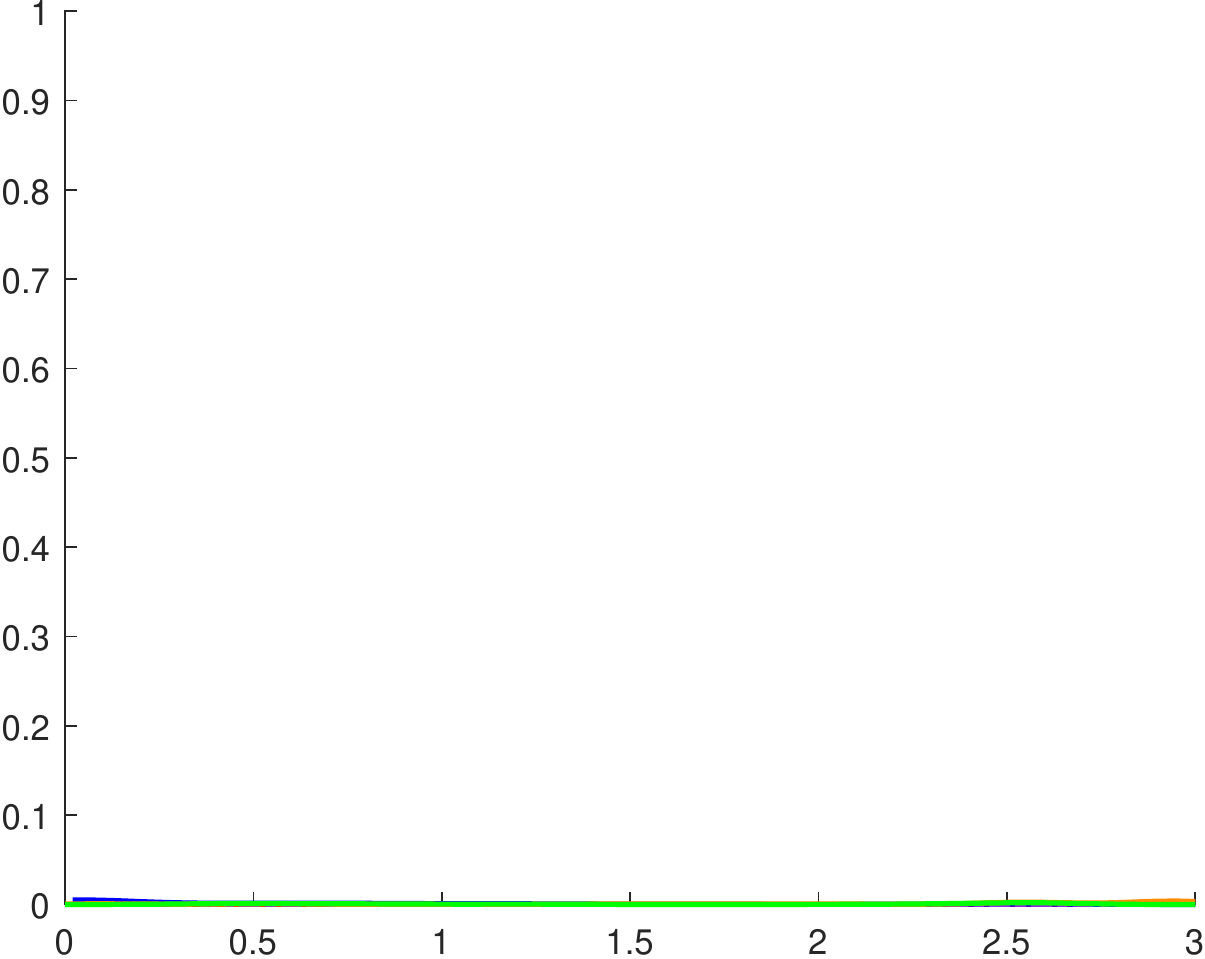} \\
		\includegraphics[width=0.16\linewidth]{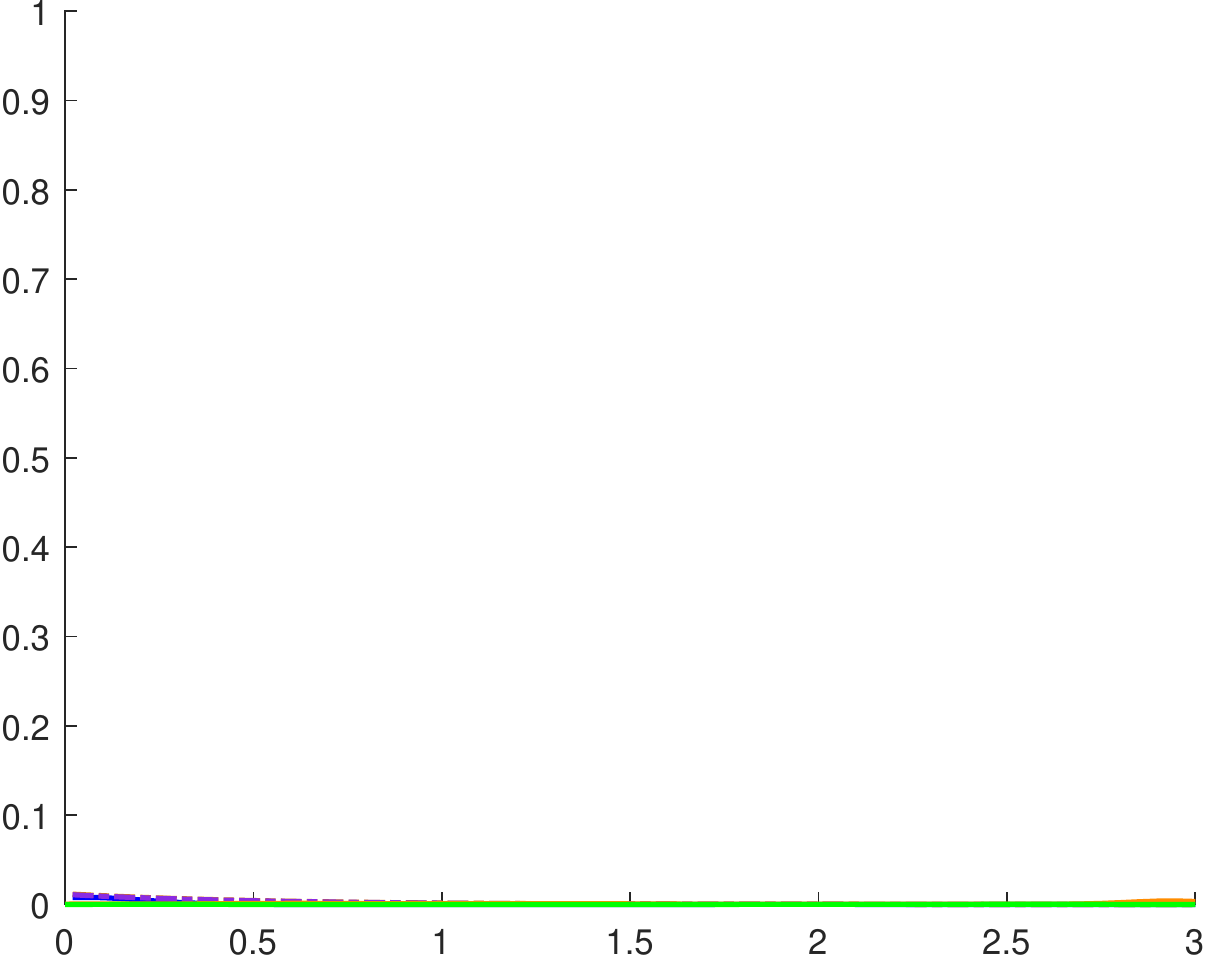} & \includegraphics[width=0.16\linewidth]{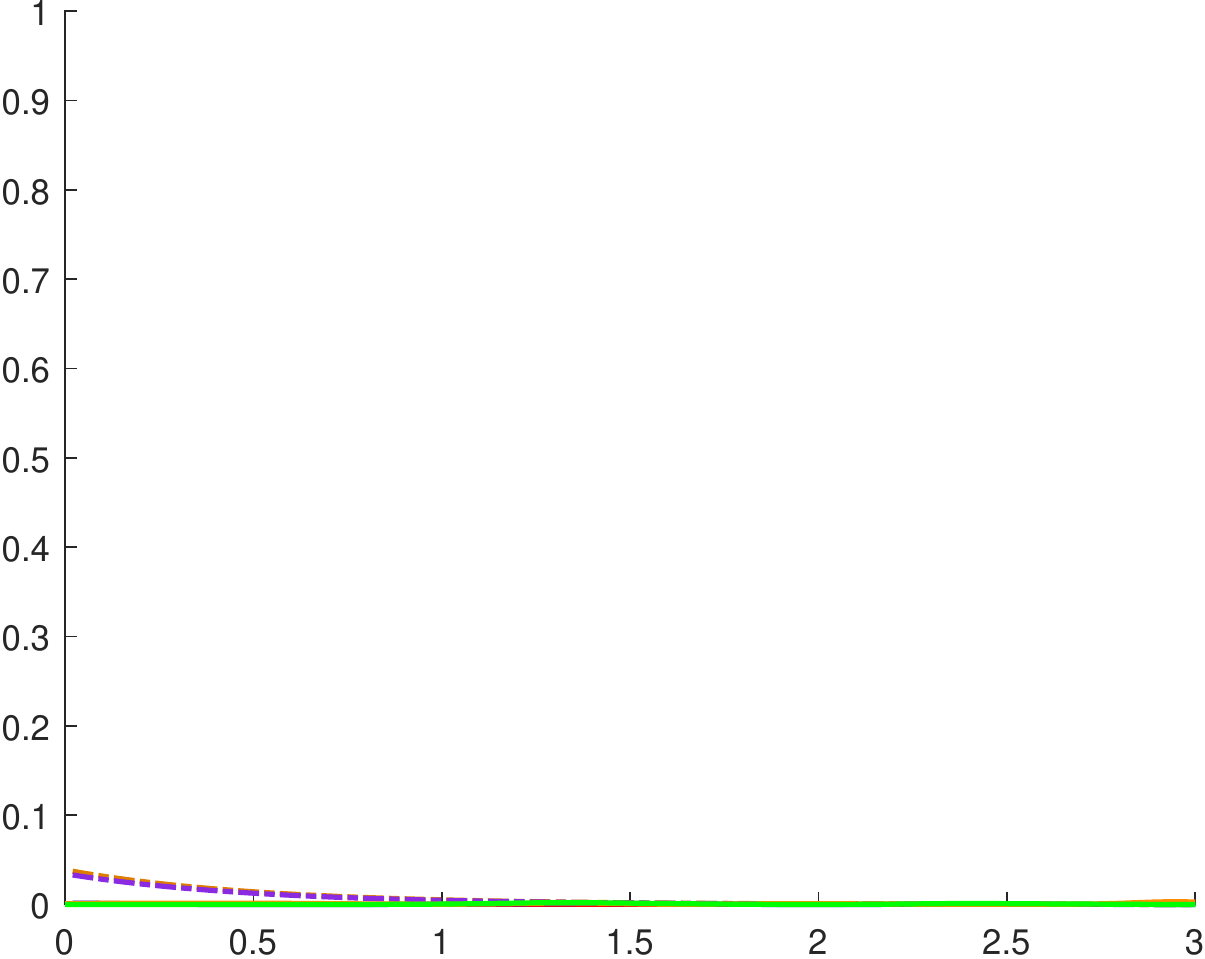} & \includegraphics[width=0.16\linewidth]{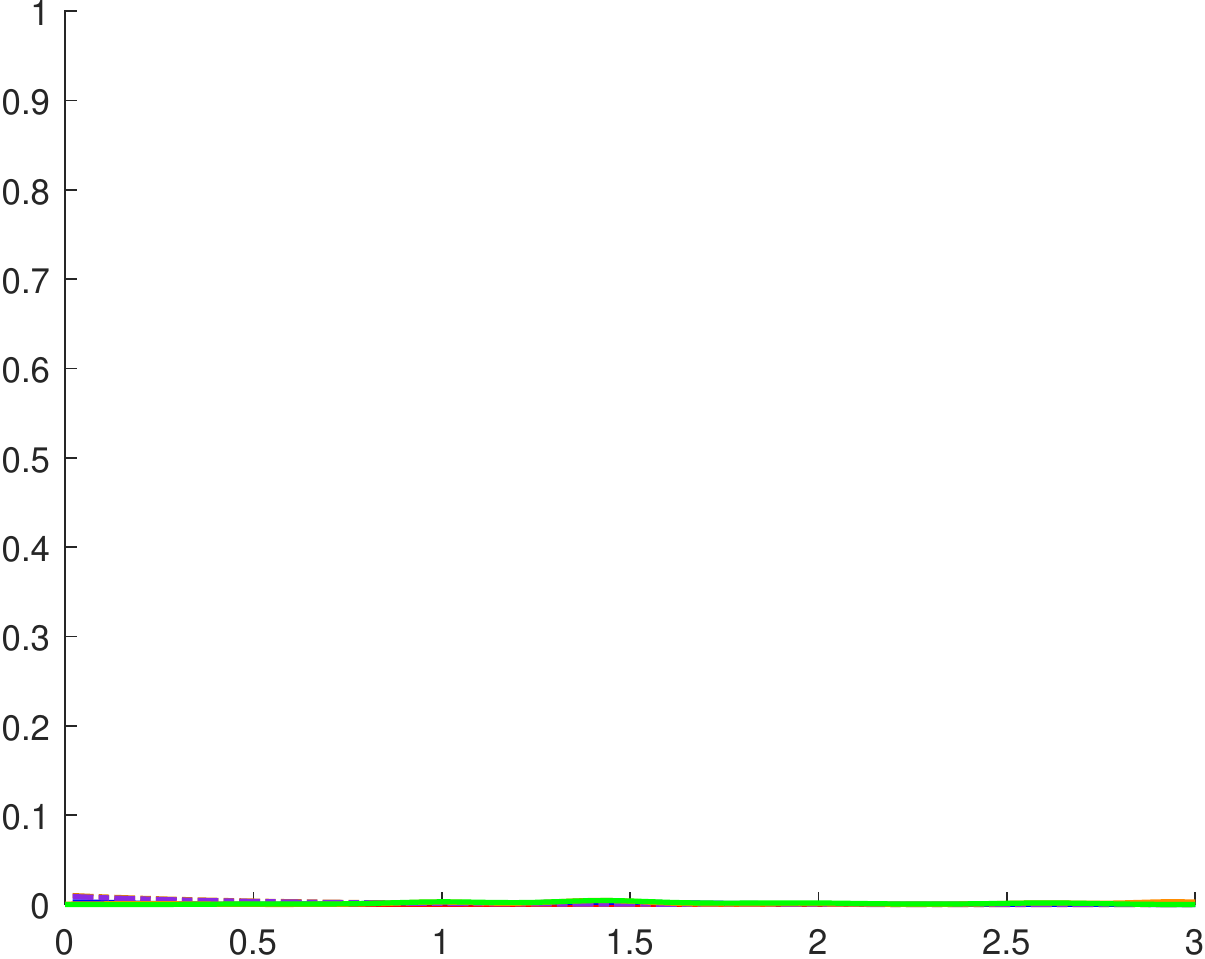} & \includegraphics[width=0.16\linewidth]{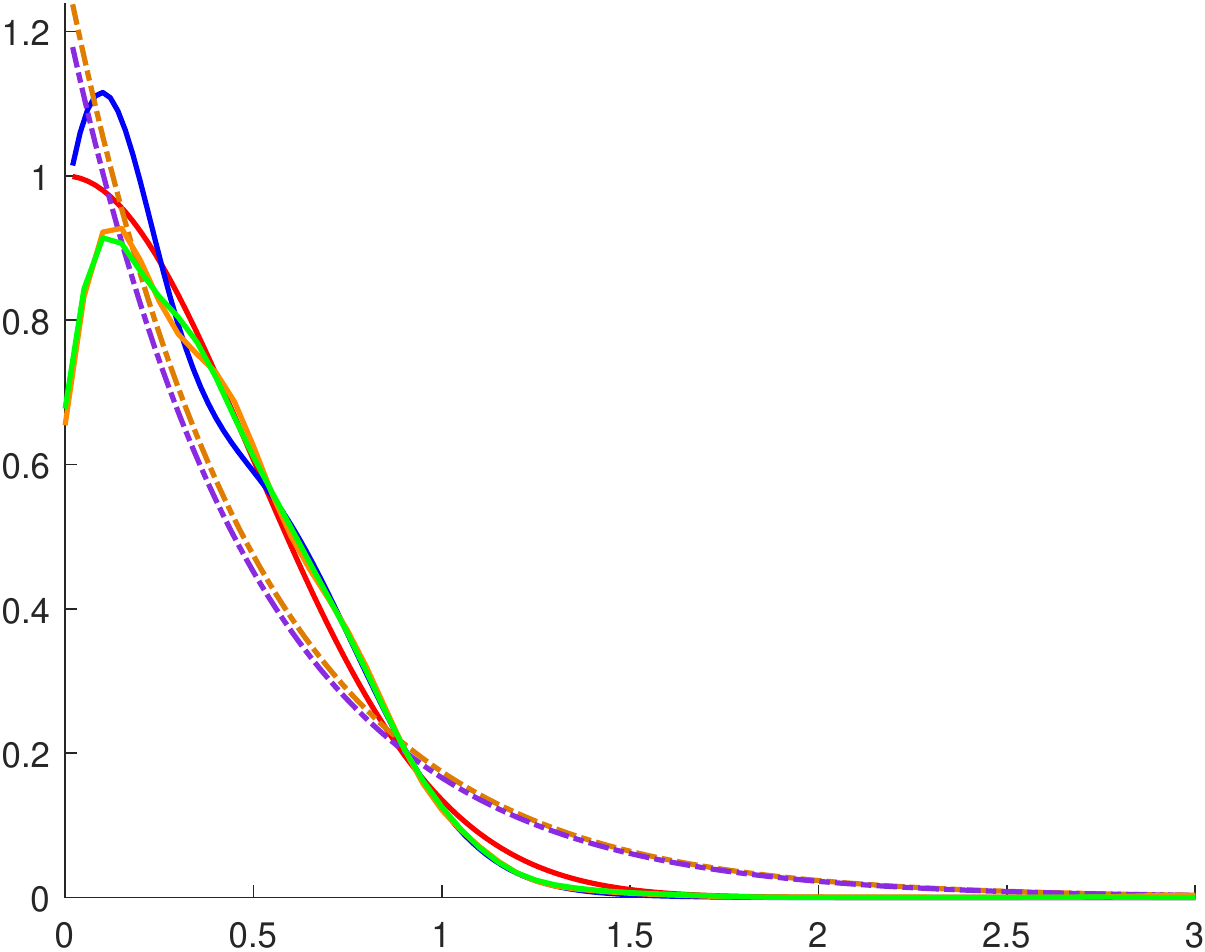} & \includegraphics[width=0.16\linewidth]{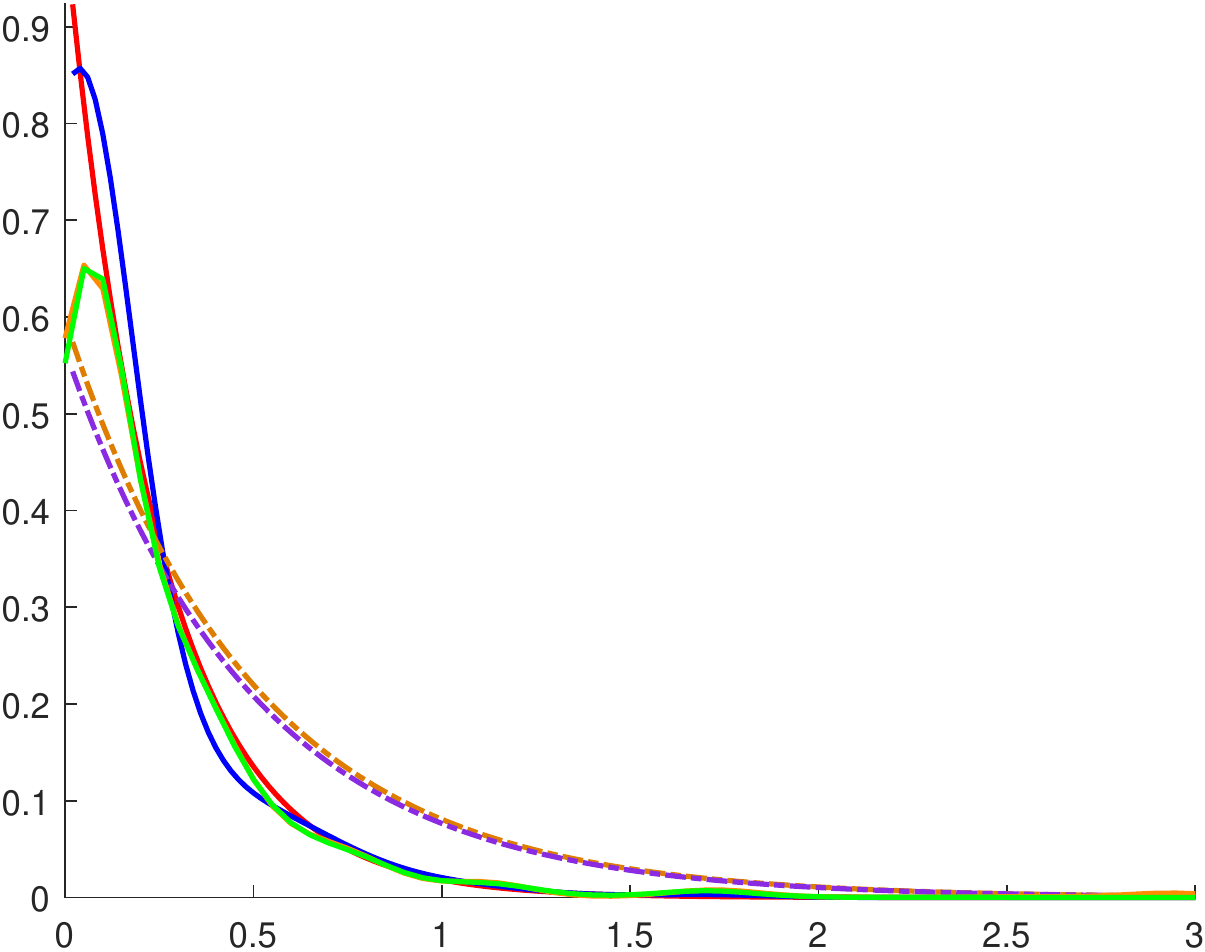} \\
		\includegraphics[width=0.16\linewidth]{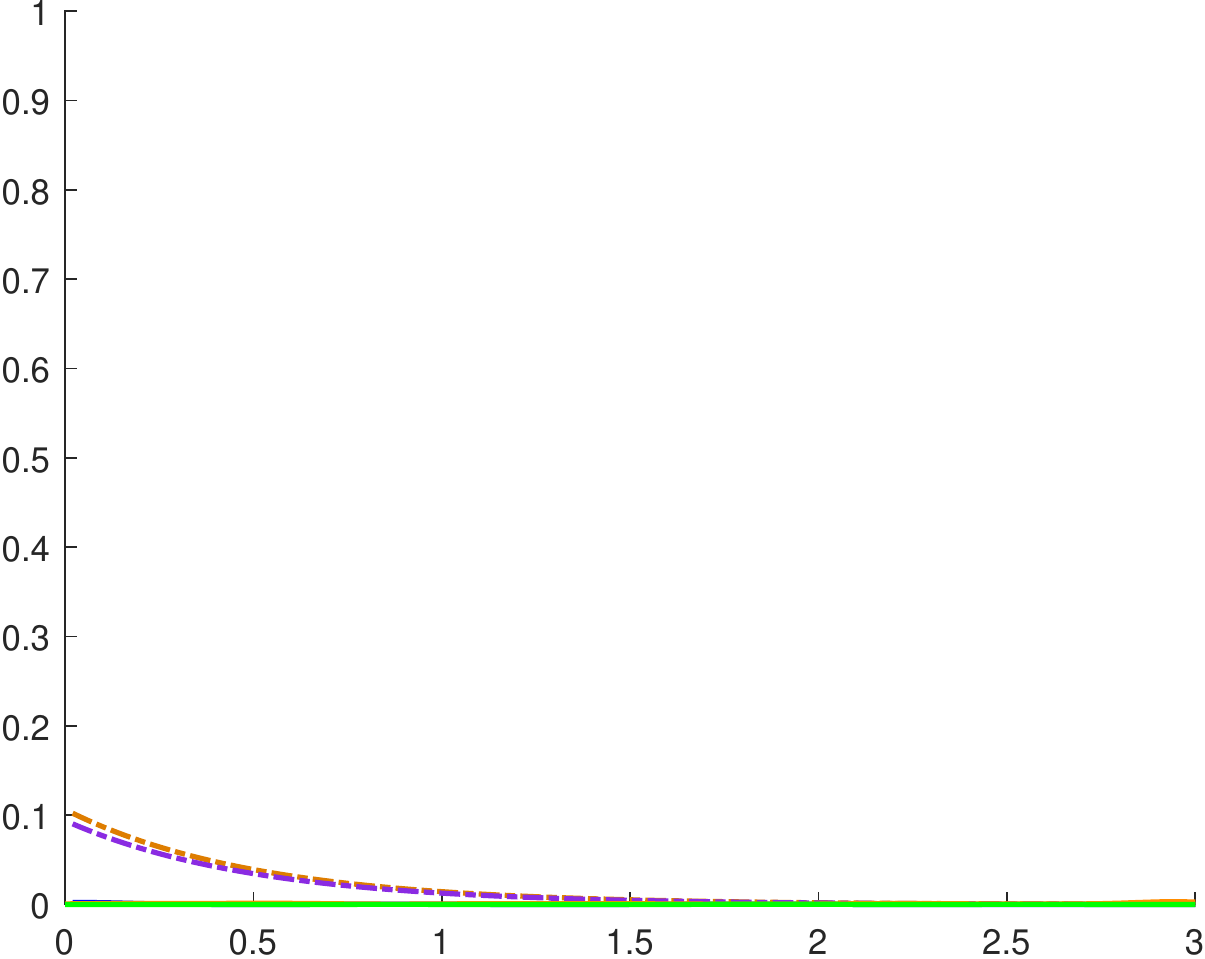} & \includegraphics[width=0.16\linewidth]{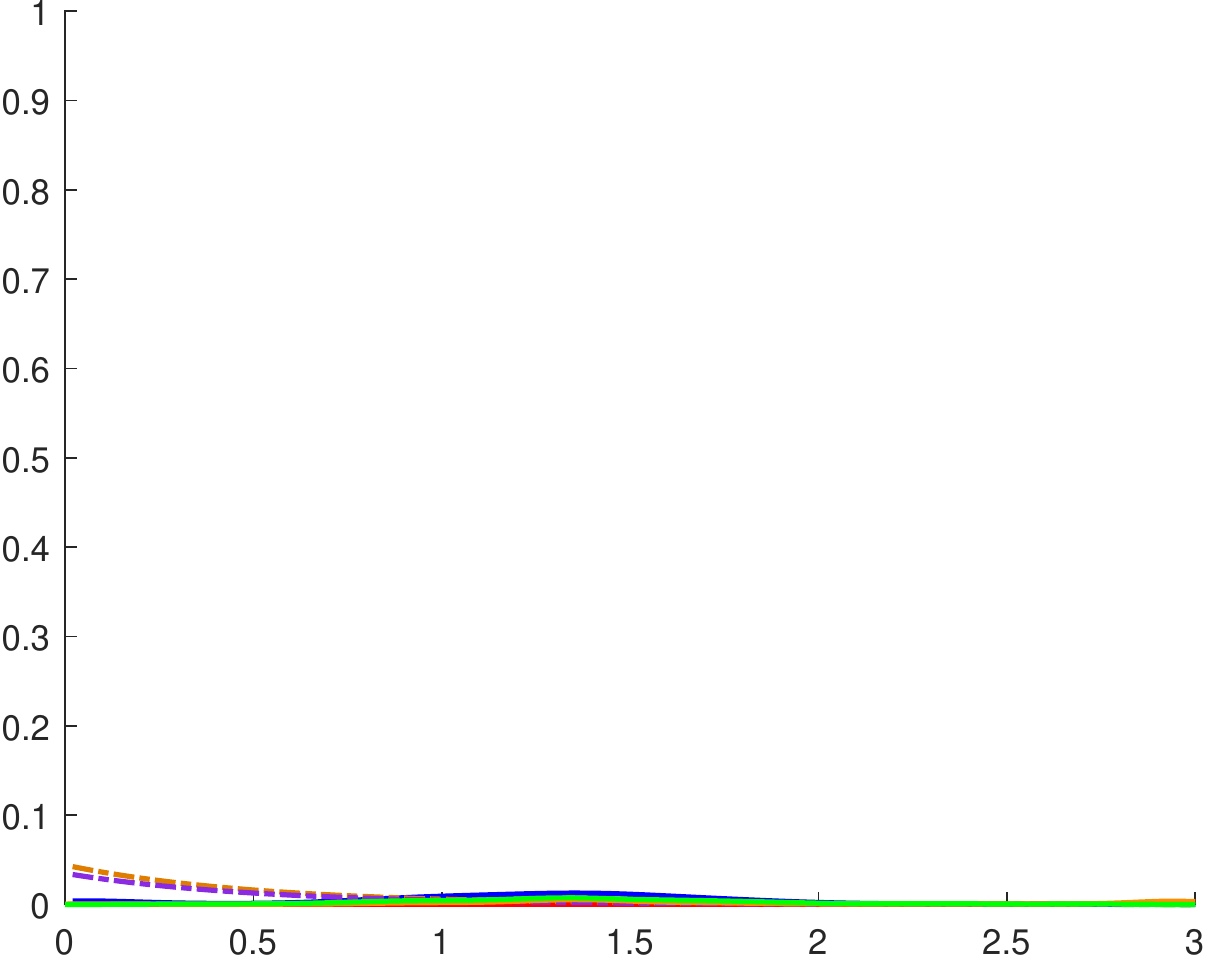} & \includegraphics[width=0.16\linewidth]{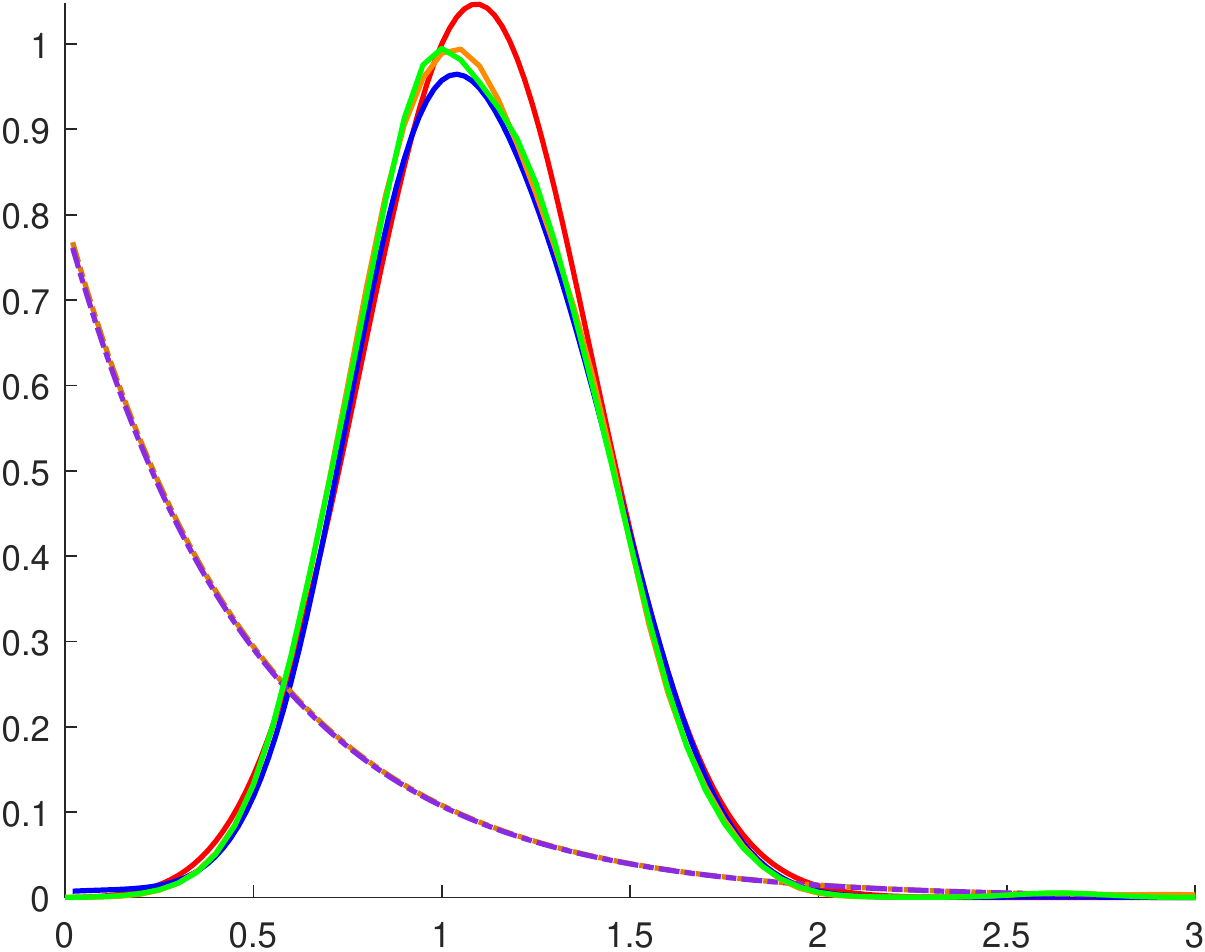} & \includegraphics[width=0.16\linewidth]{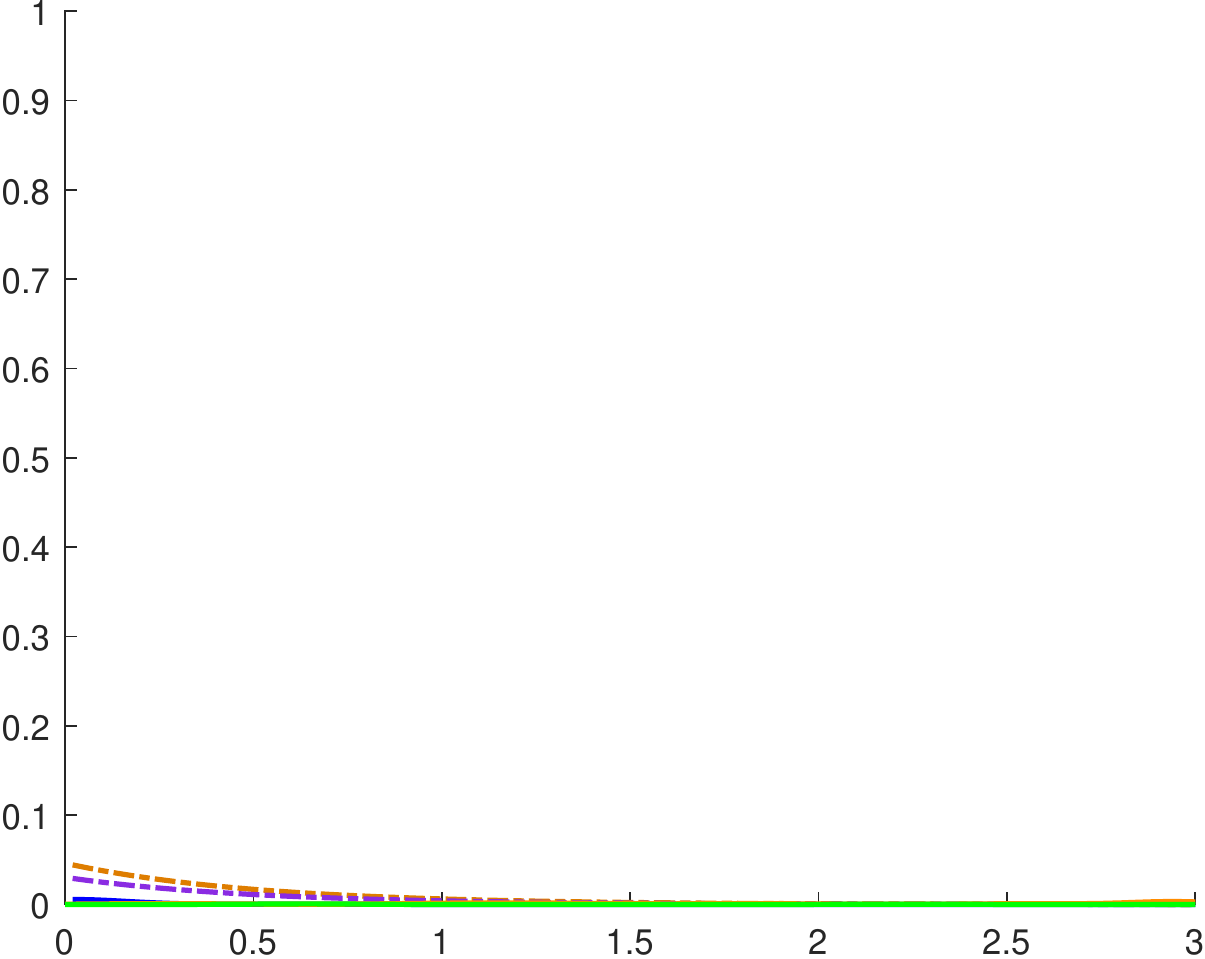} & \includegraphics[width=0.16\linewidth]{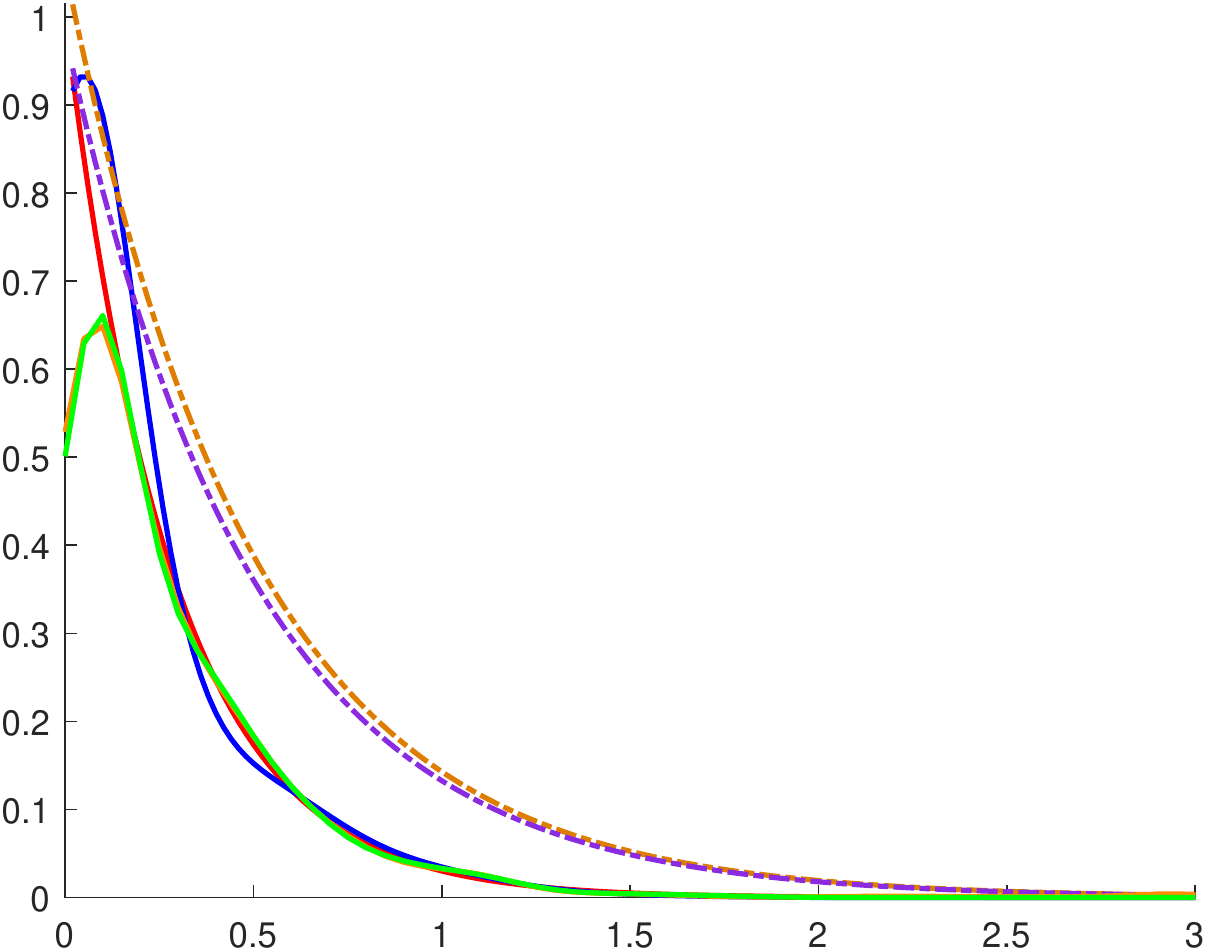}
		\\ 
		\multicolumn{5}{c}{\includegraphics[width=0.5\linewidth]{legendcrop.pdf}}
	\end{tabular}
	\caption{Performance of different algorithms for estimating $\bF=[f_{i,j}(t)]_{i,j=1,\ldots,5}$. 
	}
	\label{fig::exp2d}
\end{figure}




\begin{figure}
\centering
\begin{tabular}{cc}		
		\includegraphics[width=0.45\textwidth]{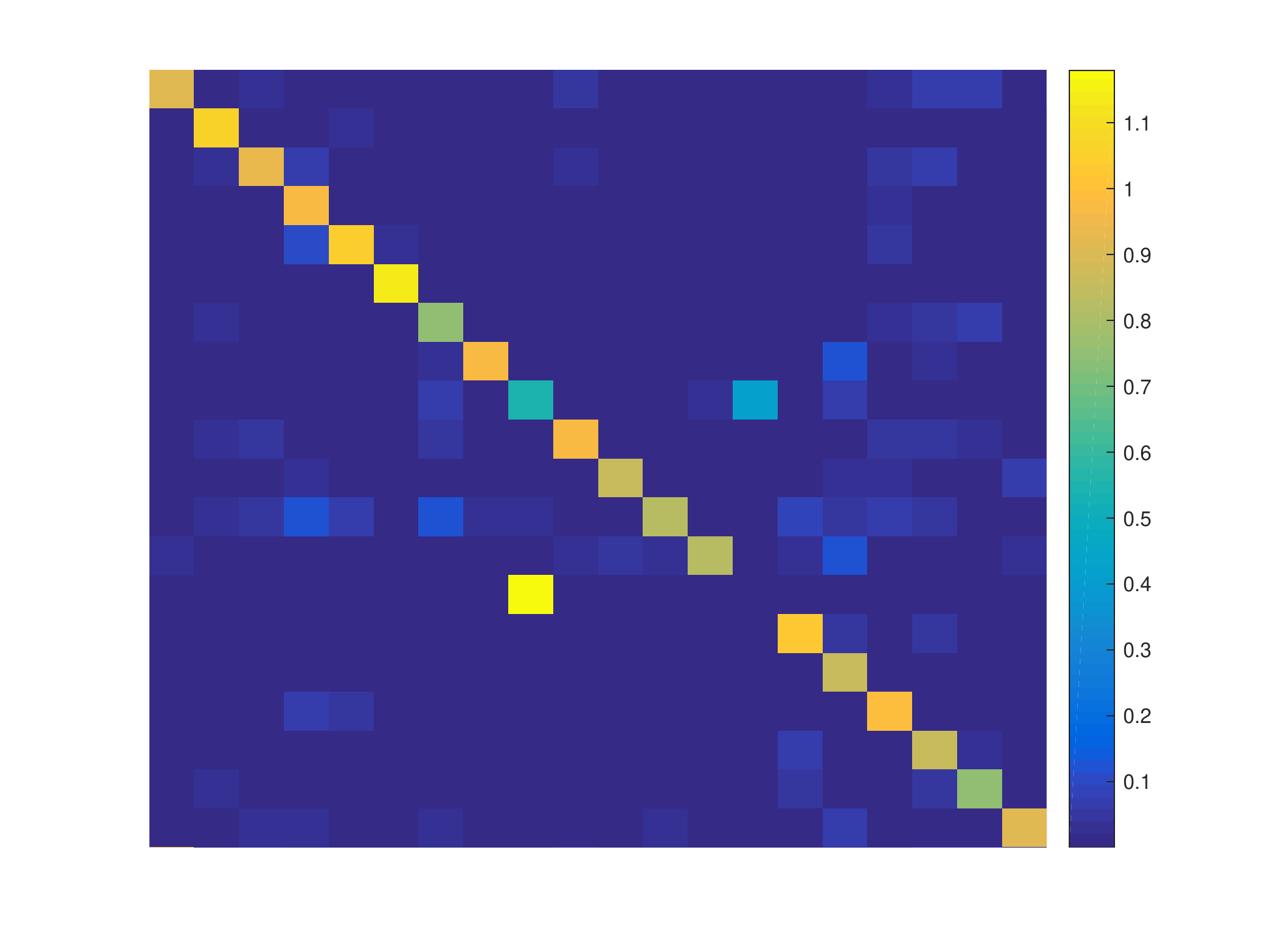}
		&\includegraphics[width=0.45\textwidth]{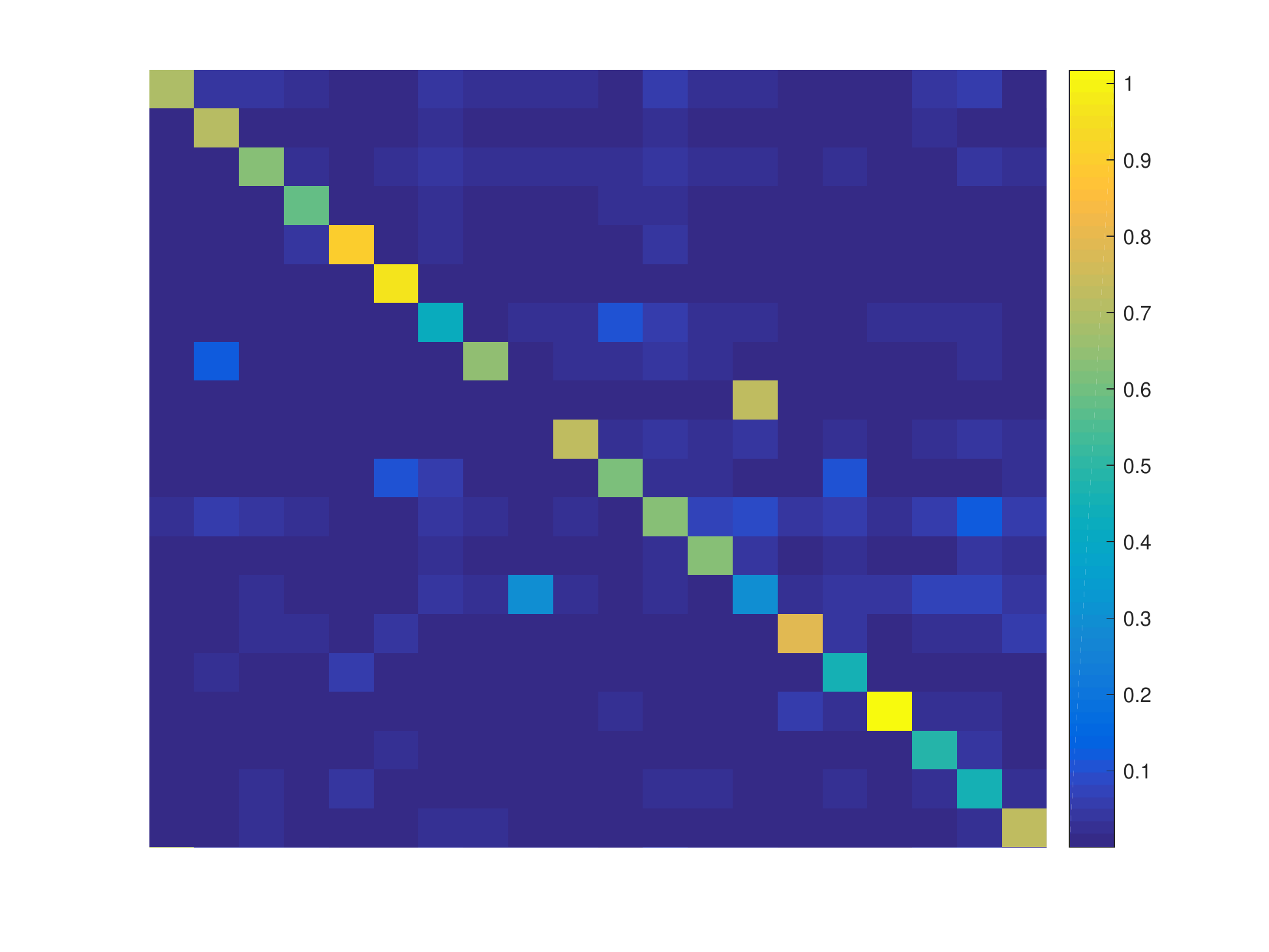}\\
		{\footnotesize (a) MLE estimate with 8 outer loops and 8 inner loops.}
		& {\footnotesize (b) NPOLE-MHP esitmate with $\eta_k=1/(k\zeta+400)$.}\\
		\includegraphics[width=0.45\textwidth]{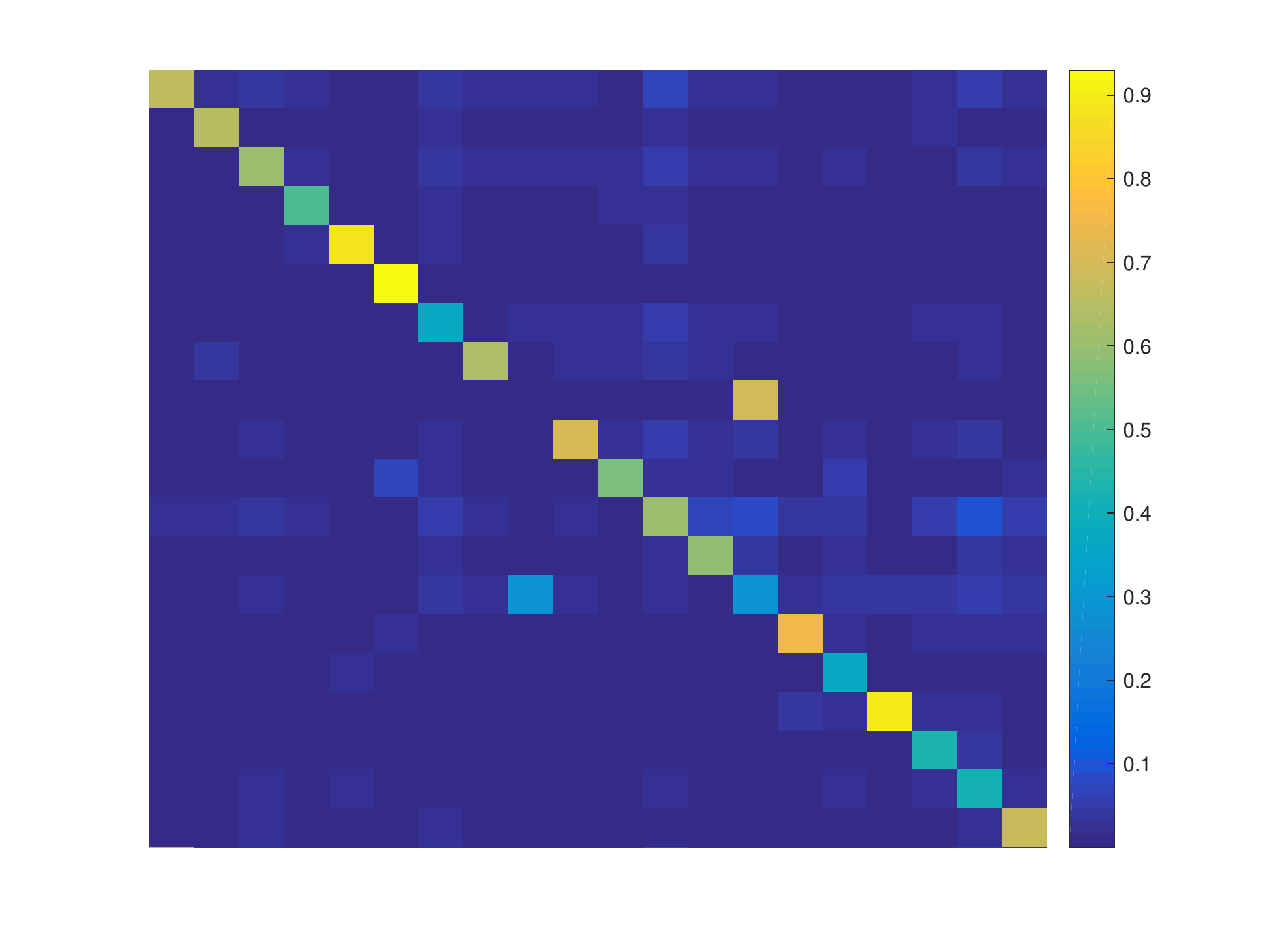}
		&\includegraphics[width=0.45\textwidth]{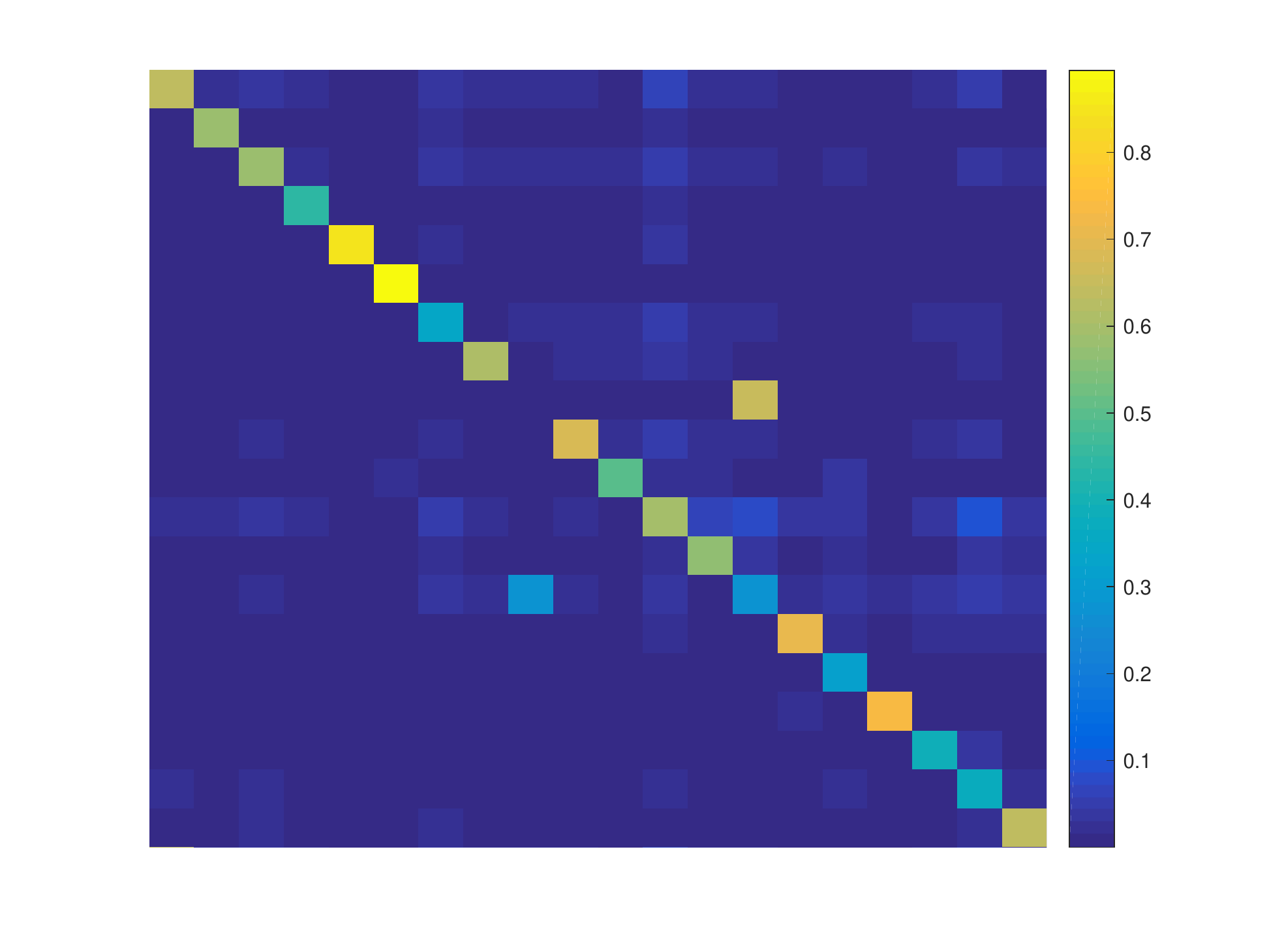}\\
		{\footnotesize (c) NPOLE-MHP esitmate with $\eta_k=1/(k\zeta+600)$.}
		& {\footnotesize (d) NPOLE-MHP esitmate with $\eta_k=1/(k\zeta+800)$.}\\
\end{tabular}
	\caption{NPOLE-MHP and MLE: a color map comparison.}
	\label{fig::heatmap}
\end{figure}

In this section, we show the complete set of estimates for the estimates on synthetic data, in Figure \ref{fig::exp2d}, and on real data, in Figure \ref{fig::heatmap}.

For the real data, we compare the values of $\|\hat{f}_{i,j}\|_{L_1[0,z]}$ by converting them into color maps (Figure \ref{fig::heatmap}). Top left corner, $\|\hat{f}_{i,j}\|_{L_1[0,z]}$ is computed using the output of MLE of \citet{xu2016learning} with 8 outer loops and 8 inner loops, respectively, using 18 days of the meme-tracking dataset. For the rest of the three plots, we calculate  $\|\hat{f}_{i,j}\|_{L_1[0,z]}$ using the output of NPOLE-MHP with different step sizes. It can be seen that NPOLE-MHP generates similar sparsity patterns to that of MLE where the diagonal dominates.

\end{document}